\documentclass[final,onefignum,onetabnum]{siamart190516}

\usepackage{braket,amsfonts}

\usepackage{array}

\usepackage[caption=false]{subfig}

\usepackage{pgfplots}

\newsiamthm{claim}{Claim}
\newsiamremark{remark}{Remark}
\newsiamremark{hypothesis}{Hypothesis}
\crefname{hypothesis}{Hypothesis}{Hypotheses}

\usepackage{algorithmic}

\usepackage{graphicx,epstopdf}
\usepackage{todonotes}

\Crefname{ALC@unique}{Line}{Lines}

\usepackage{amsopn}

\usepackage{xspace}
\usepackage{bold-extra}
\usepackage[most]{tcolorbox}

\colorlet{texcscolor}{blue!50!black}
\colorlet{texemcolor}{red!70!black}
\colorlet{texpreamble}{red!70!black}
\colorlet{codebackground}{black!25!white!25}

\usepackage{tikz}
\definecolor{sky}{RGB}{18,102,153}

\newsiamthm{expl}{Example}

\DeclareMathOperator{\VAR}{ Var }

\newcommand{\MEANNN}[2]{ \mathbb{E}_{ #1 } \left [  #2 \right ] }

\usepackage{cases}

\lstdefinestyle{siamlatex}{%
  style=tcblatex,
  texcsstyle=*\color{texcscolor},
  texcsstyle=[2]\color{texemcolor},
  keywordstyle=[2]\color{texemcolor},
  moretexcs={cref,Cref,maketitle,mathcal,text,headers,email,url},
}

\tcbset{%
  colframe=black!75!white!75,
  coltitle=white,
  colback=codebackground, 
  colbacklower=white, 
  fonttitle=\bfseries,
  arc=0pt,outer arc=0pt,
  top=1pt,bottom=1pt,left=1mm,right=1mm,middle=1mm,boxsep=1mm,
  leftrule=0.3mm,rightrule=0.3mm,toprule=0.3mm,bottomrule=0.3mm,
  listing options={style=siamlatex}
}

\newtcblisting[use counter=example]{example}[2][]{%
  title={Example~\thetcbcounter: #2},#1}

\newtcbinputlisting[use counter=example]{\examplefile}[3][]{%
  title={Example~\thetcbcounter: #2},listing file={#3},#1}

\DeclareTotalTCBox{\code}{ v O{} }
{ 
  fontupper=\ttfamily\color{black},
  nobeforeafter,
  tcbox raise base,
  colback=codebackground,colframe=white,
  top=0pt,bottom=0pt,left=0mm,right=0mm,
  leftrule=0pt,rightrule=0pt,toprule=0mm,bottomrule=0mm,
  boxsep=0.5mm,
  #2}{#1}

\patchcmd\newpage{\vfil}{}{}{}
\begin{tcbverbatimwrite}{tmp_\jobname_header.tex}
\title{
Model  Uncertainty and Correctability for Directed Graphical Models}

\author{Panagiota Birmpa\thanks{Department of Mathematics and Statistics, University of Massachusetts, Amherst, U.S.A (\email{birmpa@math.umass.edu}).}
\and Jinchao Feng\thanks{Department of Applied Mathematics and Statistics,  Johns Hopkins University, Baltimore, U.S.A (\email{jfeng34@jhu.edu}).}
\and Markos A. Katsoulakis\thanks{Department of Mathematics and Statistics, University of Massachusetts, Amherst, U.S.A (\email{markos@math.umass.edu}).}
\and Luc Rey-Bellet\thanks{Department of Mathematics and Statistics, University of Massachusetts, Amherst, U.S.A (\email{luc@math.umass.edu}).}}

\headers{
Model  
Uncertainty \& Correctability for 
Graphical Models}{P. Birmpa, J. Feng, M. A. Katsoulakis, L. Rey-Bellet}
\end{tcbverbatimwrite}
\input{tmp_\jobname_header.tex}

\ifpdf
\hypersetup{ pdftitle={Quantification of Model  Uncertainty and Correctability for Directed Graphical Models} }
\fi

\begin{document}
\maketitle

\begin{tcbverbatimwrite}{tmp_\jobname_abstract.tex}
\begin{abstract}
Probabilistic graphical models are a fundamental tool in probabilistic modeling, machine learning and artificial intelligence. They allow us to integrate in a natural way expert knowledge, physical modeling,
heterogeneous and correlated data and quantities of interest. For exactly this reason, multiple sources of model uncertainty are inherent within the  modular structure of the graphical model. In this paper we develop information-theoretic, robust  uncertainty quantification methods and non-parametric stress tests for directed graphical models to assess the  effect and the  propagation through the graph of multi-sourced model uncertainties to  quantities of interest.
These methods allow us to rank the   different  sources  of  uncertainty  and correct the graphical  model by targeting its most impactful components  with respect to  the quantities of interest. Thus, from a machine learning perspective, we provide a mathematically rigorous approach to correctability that guarantees a systematic selection for improvement of
components of a graphical model  while controlling potential new errors created in the process in other parts of the model.
We demonstrate our methods in two physico-chemical examples, namely  quantum scale-informed chemical kinetics and materials screening to improve the efficiency of fuel cells.

\end{abstract}

\end{tcbverbatimwrite}
\input{tmp_\jobname_abstract.tex}

\section{Introduction}

Data-informed, 
structured  probability models  are typically constructed by combining expert-based mathematical models with available data,  the latter often being heterogeneous, i.e. from multiple sources and  scales,  and  possibly  sparse or imperfect. 
Typically   such  structured  models are formulated as  probabilistic graphical models (PGM), which in turn   are generally classified into Markov Random Fields (MRFs)  over undirected graphs and Bayesian Networks (Bayesian network)  over Directed Acyclic Graph (DAG) \cite{pearl1988probabilistic, koller2009probabilistic},
as well as mixtures of those two classes, \cite{GZ}.   In this paper we focus on Bayesian Networks. DAGs are graphs with directed edges and without  cycles, where   individual vertices correspond to different model components or data inputs, while the directed edges
encode
conditional dependencies between vertices. 
Formally, a Bayesian Network  over a DAG is defined as a pair $\{G,P\}$ consisting of the following ingredients:   $G=\{V,E\}$  is a DAG with $n$ vertices denoted by $V=\{1,\dots,n\}, n\in\mathbb{N}$, along with directed, connecting edges $E \in V\times V$. In addition,  we define   a set of random variables $X_V=\{X_1,\dots,X_n\}$ over $V$ with probability distribution  $P$  with density 
\begin{equation}
\label{eq:PGM:def}
 p(x) = \prod_{i=1}^n p(x_i|x_{\pi_i})
\end{equation}
 where $x_{\pi_i} = \{ x_{i_1},\dots, x_{i_m}\} \subset \{ x_1,\dots,x_n\}$ denotes  the values of parents $\pi_i$ of  each vertex $i$,  
see Figure~\ref{fig: ModelsAmbiguitySet}, 
 and   
$ p(x_i|x_{\pi_i})$
 is the Conditional Probability Density (CPD) for the conditional distribution $P_{i|\pi_i}$ with  parents $\pi_i$. 
 In such models we are typically interested in quantities of interest (QoI) $f(X_A)$ that involve one or more vertices $A\subset V$ and the corresponding   random variables $X_A\subset X_V$.
 
The general mathematical formulation of PGMs was developed in foundational  works in  \cite{pearl1988probabilistic,pearl2014probabilistic}, and 
 are widely used in many real-world applications of Artificial Intelligence, like medical diagnostics, natural language processing, computer vision, robotics, computational biology,  cognitive science to name a few, e.g., \cite{friedman2000using, heckerman2016toward,aronsky1998diagnosing, levitt1990model, lazkano2007use, fung1995applying}.  
 Recently PGMs were built   as computationally tractable surrogates for multi-scale/multi-physics models
 (e.g. from quantum to molecular to engineering scales), such as in porous media and energy storage,
\cite{Katsoulakis_BayesianNetwork, Hall2021_GINN}.   Such models often have hidden correlations in data used in their construction \cite{natchem} or include physical constraints in parameters \cite{Katsoulakis_BayesianNetwork},  necessitating conditional relations  between model components and thus giving rise to CPDs such as the ones in \eqref{eq:PGM:def}.
Finally,  PGMs can be used as  the mathematical foundation for building digital twins used for control and optimization tasks of real systems \cite{Sen:IEEE}.
Some examples include   Bayesian networks  for fuel cells \cite{Feng2020Science} and  Hidden Markov models (a time-dependent special case of Bayesian networks)
for unmanned aerial vehicles \cite{Willcox2021}. 
The  structured probabilistic nature of such models   allows them to be continuously improved, e.g. based on available real-time data \cite{Willcox2021} or through  targeted  data acquisition \cite{Feng2020Science}.

\smallskip

\noindent
\textbf{A. Model Uncertainty in Bayesian networks.}
  Bayesian networks
 will typically have multiple  sources of uncertainty  due to modeling choices or  learning from  imperfect data in the process of building the graph $G$ and each one of the CPDs in  \eqref{eq:PGM:def}. These uncertainties will propagate (and occasionally not propagate -- see Section~\ref{sec: ORR}) through the directed graph structure to the targeted QoIs.
  Uncertainties in probabilistic models  are broadly classified in two categories: {\em aleatoric},  due to the inherent stochasticity  of probabilistic models such as \eqref{eq:PGM:def}
 and  {\it model uncertainties} (also known as   ``epistemic"), \cite{dupuis2016path, Smith}.  In this paper we primarily focus on model uncertainties which   stem from  the inability to accurately model one or more  of the components   of a Bayesian Network $\{G,P\}$:  omitting or simplifying  model components  as is often the case in multi-scale systems, learning from sparse data, or using    approximate inference methods to build CPDs in \eqref{eq:PGM:def}. 
Next, we discuss more concretely these challenges in the context of two physico-chemical examples that  we analyze further using model uncertainty methods developed here.

First, we consider a Langmuir bimolecular adsorption model  (see Section~\ref{sec:Langmuir})  that describes the chemical kinetics with competitive dissociative adsorption of hydrogen and oxygen on a catalyst surface \cite{catalysisGadi,doi:10.1063/1.5021351}. It is a multi-scale  system of random differential equations  with correlated dependencies in their parameters (kinetic coefficients),  arising from quantum-scale computational data calculated  using Density Functional Theory (DFT) (i.e quantum computations) for actual metals. The combination of chemical kinetics with parameter dependencies, correlations and DFT data gives rise naturally to a Bayesian network. 
The QoIs are the equilibrium hydrogen and oxygen coverages computed as the stationary solutions of a system of mean-field differential equations. 
 Here the Bayesian network allows us to incorporate data and correlations from a different scale  to the parameters of an  established chemical kinetics (differential equations) model.
However, the limited availability of the quantum-scale data creates significant model uncertainties  in the  distributions of kinetic coefficients, see for example Figure~\ref{fig:langmuir data}(a), and  the need to be accounted  towards obtaining reliable predictions for the QoI.

 In a second example  analyzed in Section~\ref{sec: ORR} we build suitable   Bayesian networks  for trustworthy  screening of materials  to increase the efficiency of  chemical reactions. Here we consider the  Oxygen Reduction Reaction (ORR) which is  a known performance bottleneck in fuel cells \cite{setzler2016activity}.  This electrochemistry mechanism involves   two reactions which are typically slow. Thus, we seek new materials that speed up these two slowest reactions. For this reason here we focus only on the thermodynamics of these reactions described by  the volcano curve of the Sabatier's principle  \cite{catalysisGadi}.
Based on the Sabatier's principle, the optimal oxygen binding energy is the natural descriptor for discovering new materials and hence it has to be our QoI. Starting from this QoI we build a Bayesian network that includes  expert knowledge (volcano curves), as well as  various available experimental and computational data and their  correlations or  conditional independence. 
Model uncertainties enter in the construction of the  Bayesian network due to lack of complete knowledge of physics and  sparse, expensive,  multi-sourced experimental and/or computational data, see for example Figure~\ref{fig: PGM}(c-g). 

Both these examples illustrate how PGMs (here Bayesian networks) allow to (a) organize in a natural way expert knowledge, modeling, heterogeneous and correlated data and QoIs; (b) study the propagation of all related model uncertainties to the QoI through the graph.  
Practically  these PGMs are  built around the QoI so that it contains all available sources of information  that may influence QoI predictions. 
\smallskip

\noindent
\textbf{B. Mathematical results.}
In this article we focus on quantifying  model uncertainties in Bayesian networks. 
Due to the graph structure of the models such uncertainties can be localized and their propagation to the QoI is affected by  the graph and the CPDs in \eqref{eq:PGM:def}. 
We develop model uncertainty and model sensitivity indices to quantify their impact, taking advantage of the graph structure.

    First, we refer to an already constructed Bayesian network $\{G, P\}$ as the  \textit{baseline model}. We will describe mathematically the model uncertainty of the baseline  by considering alternative models $Q$ in a suitably defined neighborhood of $P$ referred to as the \textit{ambiguity set},
 \begin{equation}
\label{eq:set:MFUQ:intro}
    \mathcal{D}^\eta := 
    \{\textrm{all Bayesian networks } Q: d(Q, P) \leq \eta \}\, .
\end{equation}
The two primary ingredients  for constructing ambiguity sets are the choice of a divergence or probabilistic metric $d$ between the baseline Bayesian network $P$ and an alternative model $Q$ and its size $\eta$ called \textit{model misspecification} which essentially describes  the level of uncertainty in the model. 
Next, given an ambiguity set $\mathcal{D}^\eta$, we define the \textit{model uncertainty indices}  for our QoI $f$ as 
\begin{equation}\label{eq:stresstest:intro}
I^{\pm}(f, P; \mathcal{D}^\eta):=\underset{Q \in\mathcal{D}^\eta}{\mathrm{sup/inf}}\  \{\MEANNN{Q}{f} -\MEANNN{P}{f} \}\, .
\end{equation}
We view these  indices  as a  \textit{non-parametric stress test} on the the baseline $P$ for the QoI $f$ within the ambiguity  set $\mathcal{D}^\eta$,
since they provide the corresponding worst case scenarios. Furthermore, the ambiguity set is non-parametric, allowing us to test the \textit{robustness} of
the baseline against a broader set of scenarios than just a fixed parametric family.

Here  we will define ambiguity sets using the  Kullback-Leibler (KL) divergence as it allows us to obtain 
easily computable  and scalable model uncertainty  indices $I^{\pm}(f, P; \mathcal{D}^\eta)$. Indeed, the KL chain rule allows us to break
down the calculation of any KL distance between different Bayesian network components, i.e. in terms of
conditional KL divergences between distinct CPDs as well as to isolate the uncertainty impact on QoIs from multiple Bayesian network components
and data sources.  A discussion on other natural choices of divergences and metrics can be found in  Section~\ref{sec:conclusions}.
On the other hand, the model misspecification $\eta$ can be selected in two ways. First, by the user adjusting the stress test on the QoI, for example when available  data are too sparse or absent.  
Otherwise, $\eta$  can be estimated  as the KL divergence between the model and the available data. 
Thus, we can consider  user-determined or data-informed  stress tests respectively.


Next, we design different stress tests by adjusting the ambiguity set \eqref{eq:set:MFUQ:intro} to account for global or local perturbations/uncertainties of the baseline model \eqref{eq:PGM:def}.

 \smallskip
 
\noindent \textbf{Model uncertainty indices (perturbing the entire model).} 
Let $f(X_A)$ be a QoI defined on any set of random variables $X_A\subset X_V$. The ambiguity set \eqref{eq:set:MFUQ:intro} in this case  contains all the possible alternative Bayesian networks $Q$ $\eta$-close to the baseline Bayesian network $P$ in the KL divergence for some model misspecification $\eta$. In Theorem~\ref{thm:MFUQ PGM}, we demonstrate that the model uncertainty indices  for  $f(X_A)$  \eqref{eq:stresstest:intro} can be re-written  only in terms of the baseline  $P$ through the one-dimensional optimization
 \begin{eqnarray}
\label{eq:MFUQ indexintro}
I^{\pm}(f(X_A), P; \mathcal{D}^\eta)
= \pm \inf_{c>0}\Big[\frac{1}{c} \log \MEANNN{P_{A}}{e^{\pm c\bar{f}(X_A)} } + \frac{\eta}{c}\Big]=\MEANNN{Q^\pm}{f} -\MEANNN{P}{f} 
\end{eqnarray}
where $P_{A}$ is the marginal distribution of $X_A$  and $\bar{f}(X_A)$ is the centered QoI with respect to $P$. Furthermore, there exist optimizers $Q^{\pm}$ (last equality in \eqref{eq:MFUQ indexintro}) that are Bayesian networks that can  be computed explicitly. 
We note that although the optimization in   \eqref{eq:stresstest:intro} is infinite-dimensional and thus essentially computationally intractable, the formula \eqref{eq:MFUQ indexintro} gives rise to a computable one-dimensional optimization involving only the baseline Bayesian network  $P$. This significant advantage will be exploited throughout the paper to provide practical quantification of  model uncertainty and model sensitivity for PGMs.

Next we quantify the robustness of the baseline against perturbations of individual components of \eqref{eq:PGM:def}. We intend to use 
these methods to explore the relative sensitivity of the baseline on different Bayesian network components, hence we will refer to the corresponding
indices 
model sensitivity indices.

\smallskip

\noindent\textbf{Model sensitivity indices (Perturbing a model component).} 
Let $f(X_k)$ be a QoI with $k\in V$. We examine two ambiguity sets depending on the manner individual model components are perturbed. 
The first ambiguity set consists of all Bayesian networks \eqref{eq:PGM:def} with the same CPDs except for the CPD at a specific vertex $l \in V$; the structure/parents of the component $l$ can be different, however the alternative CPDs are $\eta_l$-``close" to $P$ at the $l$-th component in  KL divergence for some model misspecification $\eta_l$, see Figure~\ref{fig: MFUQ optimizer}.
The second ambiguity set consists of all Bayesian networks with the parents of the vertex $l$ being fixed and only the CPD of $l$ varying. Even if the latter set is a subset of the first ambiguity set, such graph-based constraints allow us to focus on uncertainties arising from a given CPD of the network.  For these ambiguity sets, we derive explicit formulas for  the corresponding sensitivity indices
 that are  tight and practically computable  similarly to \eqref{eq:MFUQ indexintro}, see Theorem~\ref{thm:MFSI general} and \ref{thm:MFSI}.


\smallskip

\noindent
 \textbf{C. Model sensitivity for  ranking and correctability.}
Model sensitivity indices are used here to rank the impact of  different sources of uncertainty, from least to most  influential,  in the prediction of QoIs for Bayesian networks.  
From a machine learning perspective, such rankings  are a systematic form of \textit{interpretability}, i.e. the ability to identify cause and effect in a model, \cite{DoshiVelez2017TowardsAR,MILLER20191,Burges2016ABC},  and {\it explainability}, i.e. the ability to explain model outputs through the modeling and data choices made in the construction of the baseline predictive model, see  \cite{8466590} and references therein. In the ORR model discussed earlier,  
we compute model misspecification parameters $\eta_l$  from data, 
we implement the ranking procedure 
for each graph component of the ORR Bayesian network  (i.e solvation, DFT, experiment and parameter correlation) and reveal  the least and the most influential parts  of the Bayesian Network in the prediction of the optimal oxygen binding energy QoI, 
see Figure~\ref{fig:Pie} and  Section~\ref{sec: ORR}.

 Lastly we leverage  model uncertainty and model sensitivity indices to improve the baseline   with either targeted data acquisition or improved modeling of CPDs and  graph $G$ in \eqref{eq:PGM:def}. We target  for correction  under-performing  components of the baseline, i.e. those inducing the most uncertainty on the QoI in the ranking  above.  Again from a machine learning perspective such a strategy is  a step towards   the  {\em correctability}  of PGMs, 
 namely  the ability  to  correct  targeted  components of a (baseline) model  without creating new errors in other parts of the model in the process \cite{8466590}. 
  Indeed, in the ORR model, we correct the baseline Bayesian network in two distinct ways:  by including  targeted high quality data and  by increasing the model complexity, e.g. considering richer CPD classes or more complex PGMs,  as discussed in Section~\ref{sec: ORR}.
  This is an example of closing the ``data-model-predictions loop" by  iteratively  improving  the model while taking into account   trade-offs between  model complexity, data and predictive guarantees on QoIs. 

\smallskip

\noindent
\textbf{D. Related  work.} The robust perspective in \eqref{eq:stresstest:intro} for general probabilistic models  is  known in the Operations Research literature as \textit{Distributionally Robust Optimization} (DRO) and includes different choices for divergences or metrics in \eqref{eq:set:MFUQ:intro}, see for example \cite{doi:10.1287/opre.1090.0741,doi:10.1287/opre.1090.0795,doi:10.1287/opre.2014.1314,Jiang2016,2016arXiv160402199G,2016arXiv160509349L,MohajerinEsfahani2018,doi:10.1137/16M1094725,doi:10.1287/moor.2018.0936}. 
     Related work  is also encountered in macroeconomics, we refer to the book  Hansen and Sargent \cite{bookHansenSargent}.  Stress testing  via a DRO perspective has been developed in the context of insurance risk analysis in \cite{BlaLamTangYuan2019}. Finally, \cite{Owhadhi:2013} and 
\cite{gourgoulias2017biased} develop robust uncertainty quantification methods using different combinations of concentration inequalities and/or information divergences.
Regarding sensitivity analysis, we note that existing  methods, e.g.,  gradient  and ANOVA-based methods \cite{Smith} are suitable for parametric uncertainties, and thus cannot handle   model uncertainty.  Furthermore,  it is not immediately obvious how they can  take advantage of the graphical structure in Bayesian networks such as conditional independence. 
Here, our mathematical methods broadly rely on UQ information inequalities for QoIs of high-dimensional probabilistic models and stochastic processes  
\cite{dupuis2011uq,dupuis2016path,Hall:Katsoulakis:sparse,Birell-RB2020,Birrell_GOKL} (see also Appendix~\ref{subsec:MFRMFU}). 
 The mathematical novelty of our results lies on extending these earlier works on directed graphs by  developing the proper model uncertainty and model sensitivity framework for general Bayesian networks and studying 
 the propagation of multiple uncertainties through the network to the QoIs. 
Earlier work on  building a predictive chemistry-based PGM  with quantified   model uncertainty for the resulting  Gaussian Bayesian network  was carried out in \cite{Feng2020Science} 
and demonstrated in materials design for  speeding up the oxygen reduction reaction in fuel cells. 
Model uncertainty for PGMs over undirected graphs, also known as Markov Random Fields (MRF) has been recently studied in \cite{birmpa2020uncertainty}. An MRF is a unifying model for statistical mechanics (Gibbs measures) and machine learning (Boltzmann machines), while  the special case of  Gibbs measures  was studied earlier in \cite{katsoulakis2017scalable}. We note that for MRFs
the robust perspective is less flexible  as we cannot fully take advantage of the KL chain rule due to the undirected structure of the graphs.

 
\smallskip

\noindent\textbf{E. Organization.}  The main mathematical results  are presented in Section~\ref{sec:MFUQ} (model uncertainty) and Section~\ref{sec:QUSI} (model sensitivity).
In Section~\ref{sec: correctability} and Section~\ref{Sec:GBN:Correctability}, we discuss ranking and correctability for Bayesian networks. 
In Section~\ref{sec:Langmuir} we discuss  
a DFT-informed Langmuir model while 
in Section~\ref{sec: ORR} we analyze  
the ORR model arising in fuel cells. In Section~\ref{sec:conclusions} we discuss some outstanding issues and directions.
Supporting material is included in the Appendices.

\medskip

\section{Model Uncertainty Indices for Bayesian networks}
\label{sec:MFUQ}

In this section, we develop model uncertainty methods and associated  indices for Bayesian networks.
We start with the key ingredients needed to state and prove the main result (Theorem~\ref{thm:MFUQ PGM}), namely the ambiguity set,  QoIs and the definition of the model uncertainty indices.
First we define the ambiguity set
with model misspecification $\eta$ by 
 \begin{equation}
\label{eq:set:MFUQ}
    \mathcal{D}^\eta := 
    \{\textrm{all PGMs } Q: R(Q\|P) \leq \eta \}
\end{equation}
where $R(Q\|P)=\MEANNN{Q}{\log{\frac{dQ}{dP}}}$ denotes the KL divergence, 
 i.e. we perturb the baseline model $P$ to any alternative model $Q \in \mathcal{D}^\eta$, altering both the structure of the graph and the CPDs.
\begin{figure}[ht]
\centering
\includegraphics[width=1\textwidth]{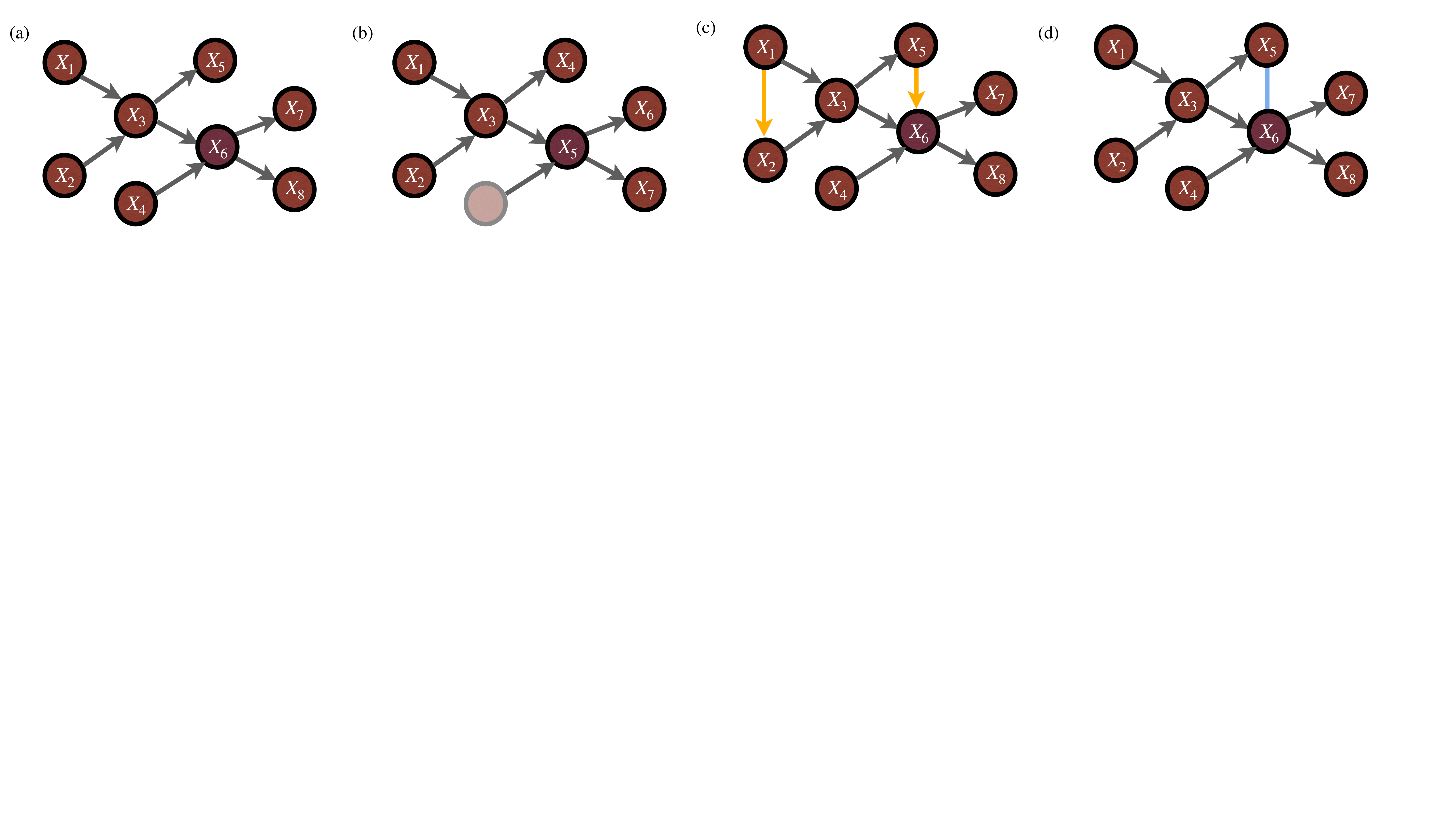}
\vspace{-5.5cm}
\caption{\small{\bf (a)} Example of the graph structure of a baseline Bayesian network $P$ and the corresponding random variables $X=\{X_1,\dots,X_8\}$. {\bf (b)} Example of the graph structure of a Bayesian network $Q\in\mathcal{D}^{\eta}$ defined on a set with one vertex less than the baseline Bayesian network $P$. {\bf (c)} An example of an alternative Bayesian network $Q\in\mathcal{D}^{\eta}$ with the same number of vertices while   $X_2$ and $X_8$ have extra new parents (in yellow).  {\bf (d)} An example of a  PGM $Q\in\mathcal{D}^{\eta}$ with a new undirected edge (in blue).  
}
\label{fig: ModelsAmbiguitySet}
\end{figure}
Examples of  models $Q$ included in $\mathcal{D}^\eta$ can be Bayesian networks defined on a smaller number of vertices than $P$, or same number of vertices with some of them having extra parents, or same number of vertices and parents but different CPDs, see Figure~\ref{fig: ModelsAmbiguitySet}(b-c). 
Furthermore, $\mathcal{D}^\eta$ can include PGMs which are not necessarily Bayesian networks, for example when  some of the  edges  between vertices are not  directed \cite{GZ}, see Figure~\ref{fig: ModelsAmbiguitySet}(d).  For a baseline Bayesian network $P$ 
we define the model uncertainty indices as
\begin{equation}
\label{eq:PU MFUQ PGM}
I^{\pm}(f(X_A), P; \mathcal{D}^\eta)=\underset{Q \in \mathcal{D}^\eta}{\mathrm{sup/inf}}\  \MEANNN{Q}{f(X_A)} -\MEANNN{P}{f(X_A)}
\end{equation}
 for a QoI $f$ which is a function of some subset of $A$ vertices in the graph, i.e., 
\begin{equation}\label{def:QoIGen}
    \textrm{for }f(X_A)=f(X_{i_1},\dots,X_{i_m}), \textrm{ with } A=\{i_1,\dots,i_m\}\subseteq V \,.
        \end{equation}
        In the next theorem, we characterize the  optimizers  $Q^{\pm}$ in \eqref{eq:PU MFUQ PGM} which turn out to be Bayesian networks of the form \eqref{eq:PGM:def} and we provide their CPDs explicitly. 

\smallskip

\noindent {\bf Notation.} Before we state our results let us fix some notation. For a  Bayesian network $\{G, P\}$  
$\pi_i^P$ denotes the set of indices of all the parents of vertex $i$ 
 and $\rho_i^P$ denotes the set of indices of all the ancestors for $i$ 
 (we may omit the superscript  ``$P$'' if only one Bayesian network is involved).
 Without loss of generality we assume that all Bayesian networks are topologically ordered as we can always relabel the DAG so that $j<i$ for all $j \in \pi_i$ by topological sorting \cite{koller2009probabilistic}.

The random vector $X=(X_1,\dots,X_n)$,  indexed by the vertices $V=\{1,...,n\}$, takes values  $X=x=(x_1,\dots,x_n)\in\mathcal{X}$.   
The joint probability distribution of $X$ is denoted by $P$ with density $p(x)$; the results are presented when the joint probability distribution $P$ is continuous but all results hold when $p$ is a discrete distribution as well. For any subset $A=\{i_1,\cdots, i_m\} \subset V$ we denote $X_A=(X_{i_1}, \cdots, X_{i_m})$ which takes values $X_A=x_A=(x_{i_1}, \cdots, x_{i_m})\in\mathcal{X}_A$
and we denote by $P_A$ its marginal distribution.

We denote  $P_{i|\pi_i^P}$ the conditional distribution of $X_i$ given parents values $X_{\pi_i^P} = x_{\pi_i}$, i.e. $P_{i|\pi_i^P}(dx_i)=P(dx_i|x_{\pi_i})$ with corresponding CPD   $p(x_i|x_{\pi_i})$.  
The marginal distribution $P_A$ of $X_A$ has the form $P_{A}(dx_A)=\int_{\mathcal{X}_{\rho_A}}\prod_{i\in A}P_{i|\pi_i^P}(dx_i)\prod_{i\in \rho_{A}^P}P_{i|\pi_i^P}(dx_i)$ where $\rho_A^P$ are the set of indices of all the ancestors of $A$, i.e. $\rho_A^P=\cup_{i\in A}\rho_{i}^P$. Furthermore the density $p_A$ of $P_A$ is   
$ p_{A}(x_A) = 
\int_{\mathcal{X}_{\rho_A^P}}\prod_{i\in A} p(x_i|x_{\pi_i})\prod_{i\in\rho_A^P}p(x_i|x_{\pi_i})d {x_{\rho_A}}$.  
Two special cases are the marginals of  $X_k$,  $P_{\{k\}}(dx_k) = \int_{x_{\rho_k}} \prod_{i \in \{k \cup \rho_k\}} P(dx_i|x_{\pi_i})$ and the marginal of $X_{\rho_A}$, 
     $P_{\rho_A}(dx_{\rho_A}) =  \prod_{i \in \rho_{A} }P(dx_i|x_{\pi_i})$.

 
Finally for  $l_1 < \cdots < l_k$ and any QoI $f$ and for $j\in\{1,\dots,k\}$ we define the notation
    \begin{equation}\label{eq:notation:conditionals:ancestors}
       \mathbb{E}_{P_{l_{j}}|\pi_{l_{j}}^P,\dots, P_{l_{k}}|\pi_{l_{k}}^P}[f]:=\MEANNN{P_{l_j|\pi_{l_j}^P}}{\MEANNN{P_{l_{j+1}|\pi_{l_{j+1}}^P}}{\cdots \MEANNN{P_{l_k|\pi_{l_k}^P}}{f}}}. 
    \end{equation}

\begin{theorem}
\label{thm:MFUQ PGM}
Let $\{G, P\}$ be a Bayesian network with density defined as \eqref{eq:PGM:def}, and $f(X_A)$ be a QoI given in \eqref{def:QoIGen},
$f(X_A)=f(X_{i_1},\dots,X_{i_m})$. Let also $\bar{f}(X_A) := f(X_A) - \MEANNN{P}{f(X_A)}$ be the centered QoI  with finite moment generating function (MGF), $\MEANNN{P}{e^{c\bar{f}(X_{A})}}$, in a neighborhood of the origin.
\smallskip

\noindent $(a)$ {\bf{Tightness.}} For the model uncertainty indices defined in \eqref{eq:PU MFUQ PGM}, there exist $0< \eta_{\pm} \le \infty$, such that for any $\eta \le \eta_{\pm}$,
\begin{eqnarray}
\label{eq:MFUQ PGM}
I^{\pm}(f(X_A), P; \mathcal{D}^\eta)&=&\pm \inf_{c>0}\Big[\frac{1}{c} \log \MEANNN{P_{A}}{\pm e^{ c\bar{f}(X_A)} } + \frac{\eta}{c}\Big]\nonumber\\
&=& \MEANNN{Q^\pm}{f(X_A)} -\MEANNN{P}{f(X_A)},
\end{eqnarray}
where $P_{A}$ is the marginal distribution of $X_A$ with respect to $P$,  and $Q^\pm(\cdot) \equiv Q^\pm(\cdot\;;\pm c_\pm) \in \mathcal{D}^\eta$  are  Bayesian networks  \eqref{eq:PGM:def} that depend on $\eta$  and are given by
\begin{equation}\label{eq:den1}
    \frac{dQ^\pm}{dP} = \frac{e^{\pm c_\pm f(x_A)}}{\MEANNN{P}{e^{\pm c_\pm f(X_A)}}}
\end{equation}
where $c_\pm\equiv c_{\pm}(\eta)$ are the unique solutions of the equation
\begin{equation}
\label{eq:optimizer:c}
R(Q^{\pm}\|P)=\eta.
\end{equation}
\smallskip

\noindent$(b)$ {\bf{Graph Structure of $Q^{\pm}$.}} Let $L$ be   all vertices that include $A$ and all its ancestors, i.e. $L=\underset{j\in A}\cup{\rho_{j}^{p}\cup A}=\{l_1,\dots,l_{k+1}\}$, where $l_1 < \cdots < l_{k+1}$ and $l_{k+1} = i_m$, then the CPDs of $Q^\pm$ are given by
\begin{equation}\label{eq:MFUQ optimizer:1}
q^\pm(x_i|x_{\pi_i^{Q^\pm}})=
\begin{cases} 
        p( x_i|x_{\pi_i^P})\,  & \textrm{$i \notin L$} \\
       \frac{\mathbb{E}_{P_{l_{j+1}}|\pi_{l_{j+1}}^P,\dots, P_{l_{k}}|\pi_{l_{k}}^P}[e^{\pm c_\pm f(X_{A})}]}{\mathbb{E}_{P_{l_{j}}|\pi_{l_{j}}^P,\dots, P_{l_{k}}|\pi_{l_{k}}^P}[e^{\pm c_\pm f(X_{A})}]}  p( x_{l_j}|x_{\pi_{l_j}^P})\,  & \textrm{$i=l_j$}, \,j\in A
   \end{cases}
   \end{equation}
\noindent where $l_j\in L$ and $\mathbb{E}_{P_{l_{j}}|\pi_{l_{j}}^P,\dots, P_{l_{k}}|\pi_{l_{k}}^P}$ is given by \eqref{eq:notation:conditionals:ancestors}. The parents/structure is given by
$\pi_i^{Q^\pm} \equiv \pi_i^P, i \notin L$ and $\pi_{l_j}^P\! \subset\!\pi_{l_j}^{Q^\pm}\!\subset\!\pi_{l_j}^P\!\cup \! (\pi_{l_{j+1}}^{Q^\pm}\setminus l_j)$.

\begin{remark}  Theorem~\ref{thm:MFUQ PGM} readily implies that we can severely restrict the ambiguity set \eqref{eq:set:MFUQ} to  a subclass of Bayesian networks, yielding the exact same index, 
\begin{eqnarray}
I^{\pm}(f(X_A), P; \mathcal{D}^\eta) = I^{\pm}(f(X_A), P; \mathcal{D}^\eta_{\rho_A})
\end{eqnarray} 
where  a set of Bayesian networks $\mathcal{D}^\eta_{\rho_A}$ is defined as
\begin{equation}
     \!\mathcal{D}^\eta_{\rho_A} :=\left\{\!\!\!\!
    \begin{array}{cc}
    \textrm{all Bayesian networks } Q: R(Q_{A \cup \rho_A}\|P_{A\cup \rho_A}) \leq \eta \\ \textrm{ and } 
    q(x_i|x_{\pi_i^{q}}) \equiv p( x_i|x_{\pi_i^p}) \  \textrm{with $\pi_i^{Q} \equiv \pi_i^P$ for all $i \notin A\cup \rho_A$}
    \end{array}
    \!\!\!\!\right\}.
\end{equation}
This follows from  the formula
\begin{eqnarray}
 \MEANNN{P}{f(X_A)}=\int_{\mathcal{X}} f(x_A) \prod_{i=1}^n p(x_i|x_{\pi_i})dx\nonumber
 =\int_{\mathcal{X}_{A\cup \rho_A}} f(x_A) \prod_{x_i \in A\cup\rho_A} p(x_i|x_{\pi_i})dx_Adx_{\rho_A}
\end{eqnarray}
which implies that only the perturbation of $P_{A\cup\rho_{A}}$  affects the prediction of the QoI. A similar  calculation  for  the MGF of $\bar{f}$ 
implies that the optimal $Q^\pm$  has the  same CPDs as $P$ for all $X_i$, $i \notin \{A\} \cup \rho_A$, as shown in Theorem \ref{thm:MFUQ PGM}. 
\end{remark}

\begin{remark}
 We illustrate Theorem~\ref{thm:MFUQ PGM} in the special
 case $A=\{k\}$, i.e QoIs defined on one vertex through the example in Figure~\ref{fig: MFUQ optimizer}, see also Appendix~\ref{sec:GBNsup} for more details.
\begin{figure}[ht]
\centering
\includegraphics[width=1.2\textwidth]{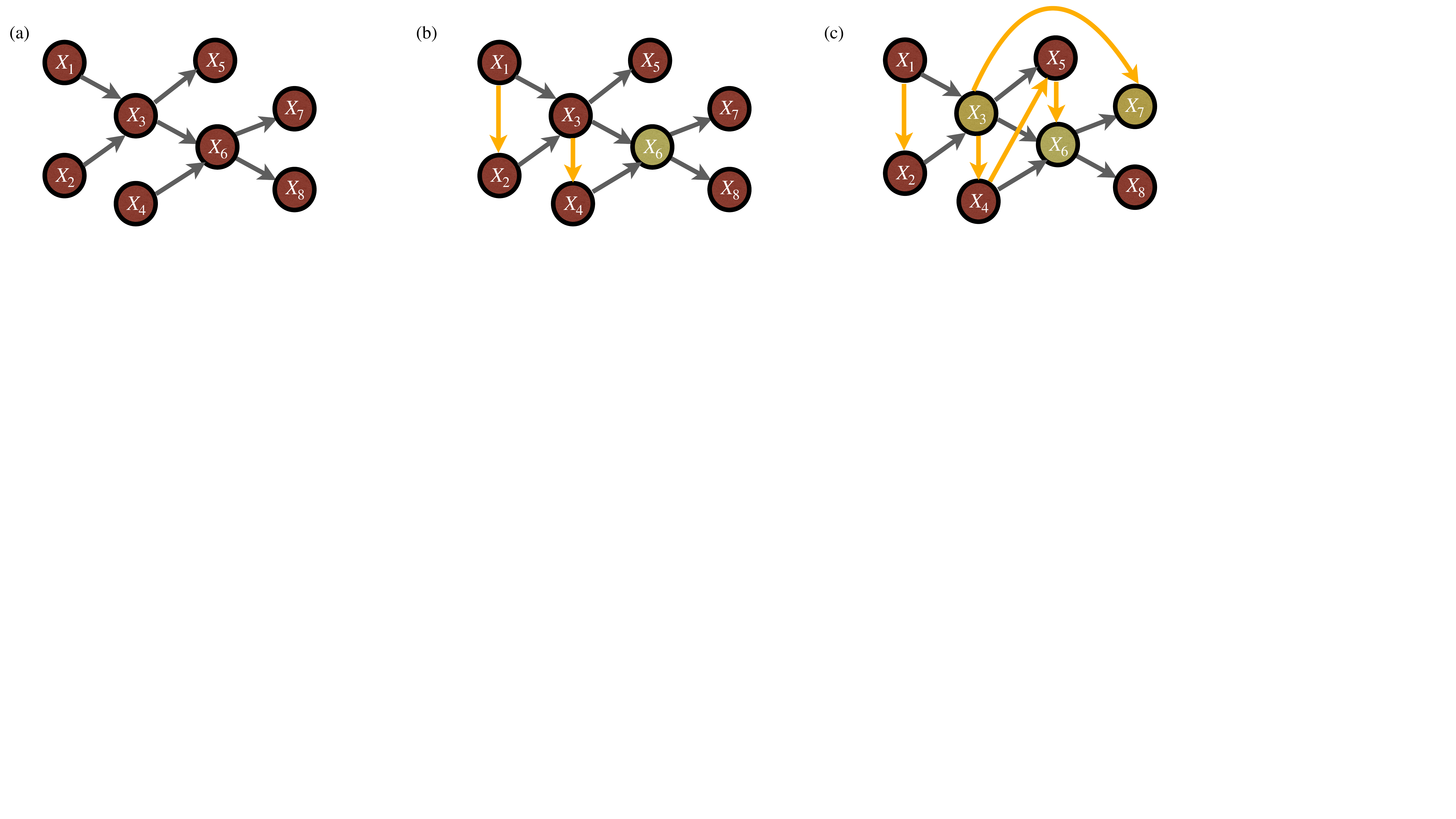}
\vspace{-6.8cm}
\caption{\small{\bf (a)} Example of  graph structure of a baseline Bayesian network $P$. {\bf (b)} The structure of the optimizers $Q^\pm$ in Theorem \ref{thm:MFUQ PGM} (b) with QoI $f(X_6)$ 
is highlighted (in green). In contrast  to the CPDs of the vertices involved in the QoI and their ancestors, the CPD of any other vertex does not change. The new parents of $X_4, X_2$ are connected (in yellow), i.e.  $X_3$ is a new parent  for $X_4$ and $X_1$ is a new parent for $X_2$. 
{\bf (c)} Structure of the optimizers $Q^\pm$ in  Theorem \ref{thm:MFUQ PGM} (b) for a QoI of the type   $f(X_3,X_6, X_7)$, see \eqref{exnode8}-\eqref{exgen*}. }
\label{fig: MFUQ optimizer}
\end{figure}
\end{remark}

\end{theorem}

\noindent{\it Proof of Theorem~\ref{thm:MFUQ PGM}.} $(a)$ The existence of $Q^{\pm}$ and \eqref{eq:den1} are direct consequences of \eqref{eq:UQ variational form} with $f(X)=f(X_A)$. For  $p(x)=\prod_{i=1}^n p(x_i|x_{\pi_i})$, we further compute
\begin{eqnarray}
\underset{Q \in \mathcal{D}^\eta}{\mathrm{sup/inf}}\  \MEANNN{Q}{f(X_A)} -\MEANNN{P}{f(X_A)} 
&=&\pm \inf_{c>0}\Big[\frac{1}{c} \log \MEANNN{P}{e^{\pm c\bar{f}(X_A)}}+ \frac{\eta}{c}\Big] \notag\\
&=& \pm \inf_{c>0}\Big[\frac{1}{c} \log \int_{\mathcal{X}} e^{\pm c\bar{f}(X_A)} \prod_{i=1}^n P(dx_i|x_{\pi_i}^P) + \frac{\eta}{c}\Big] \notag\\
&=& \pm \inf_{c>0}\Big[\frac{1}{c} \log \MEANNN{P_{A}}{e^{\pm c\bar{f}(X_A)} } + \frac{\eta}{c}\Big]
\end{eqnarray}
where $p_A$ is given in the notation before the theorem

\medskip

\noindent $(b)$ We use \eqref{eq:den1} and we factorize $q^{\pm}$ as follows:
\begin{eqnarray}\label{eq:factorization1}
    q^{\pm}(x)  &=& \frac{e^{\pm c_\pm f(x_A)}}{\MEANNN{P}{e^{\pm c_\pm f(X_A)}}} \prod_{i=1}^n p(x_i|x_{\pi_i^P}) \notag\\
    &=& \frac{1}{\MEANNN{p_A}{e^{\pm c_\pm f(X_A)}}}\prod_{i \notin\{l_1,\dots,l_{k+1}\}} p(x_i|x_{\pi_i^p}) \cdot e^{\pm c_\pm f(x_A)} p(x_{i_m}|x_{\pi_{i_m}^p}) \nonumber\\
    &&\quad\quad \times \prod_{i \in \{l_1,\dots,l_{k}\}} p(x_i|x_{\pi_i^p})
\end{eqnarray}
where $\pm c_\pm$ are the unique solutions of $R(P^{\pm c_\pm}\|P)=\eta$. Formula \eqref{eq:factorization1} is not factorized yet into CPDs as in \eqref{eq:PGM:def} due to the normalization factor at the denominator. The following analysis provides the steps for expressing \eqref{eq:factorization1} in a product of certain CPDs: Assuming that $i_1<\dots<i_m$, we start with the CPD of $X_{i_m}$ as its index is the largest among the elements of $A$. Based on \eqref{eq:factorization1}, 
\begin{equation}\label{eq:factorization2}
q^\pm(x_{i_m}|x_{\pi_{i_m}^{q^\pm}})\propto e^{\pm c_\pm f(x_A)} p(x_{i_m}|x_{\pi_{i_m}^P})
\end{equation}
We normalize the LHS of \eqref{eq:factorization1} by dividing by 
\begin{equation}\label{exp1}
\MEANNN{P_{i_m|\pi_{i_m}^P}}{e^{\pm c_\pm f(X_A)}}
\end{equation}
and by conditioning to $x_{\pi_{i_m}^P}$ and $x_{A\setminus i_m}$. Therefore, the CPD of $X_{i_m}$ and its parents $\pi_{i_m}^{Q^\pm} $ are given by
\[
q^\pm(x_{i_m}|x_{\pi_{i_m}^{Q^\pm}})=\frac{e^{\pm c_\pm f(x_A)}}{\MEANNN{P_{i_m|\pi_{i_m}^P}}{e^{\pm c_\pm f(X_A)}}} \cdot p(x_{i_m}|x_{\pi_{i_m}^P})  \textrm{ and $ \pi_{i_m}^P\subset\pi_{i_m}^{Q^\pm} = \pi_{i_m}^P\cup (A\setminus i_m) $} 
\]
Such a consideration provides the new edges in the graph of $Q^{\pm}$. In particular, $X_{i_m}$ has the same parents as in $P$ model and possibly new parents specified by $x_{A\setminus i_m}$ e.g if  $A\setminus i_m\neq \pi_{i_m}^P$. Next, we compute the  CPD of $X_{l_k}$ since $l_k<l_{k+1}=i_{m}$: As we divided by  \eqref{exp1} to normalize the the LHS of \eqref{eq:factorization1}, we keep \eqref{eq:factorization1} same if we also multiple $q^{\pm}(x)$ by  \eqref{exp1}. Hence, we pair \eqref{eq:factorization1} and $p(x_{l_k}|x_{\pi_{l_k}^P})$ so as 
 \begin{equation}\label{eq:factorization3}
q^\pm(x_{l_k}|x_{\pi_{l_k}^{Q^\pm}})\propto \MEANNN{P_{i_m|\pi_{i_m}^P}}{e^{\pm c_\pm f(X_A)}}p(x_{l_k}|x_{\pi_{l_k}^P})
\end{equation}
As before, we normalize the LHS of \eqref{eq:factorization3} by dividing by 
\begin{equation}\label{exp2}
\MEANNN{P_{l_k|\pi_{l_k}^P}}{\MEANNN{P_{i_m|\pi_{i_m}^P}}{e^{\pm c_\pm f(X_A)}}}
\end{equation}
and by conditioning to $x_{\pi_{l_k}^P}$ and $x_{\pi_{i_m}^{Q^\pm}\setminus l_k}$, we obtain 
\[
q^\pm(x_{l_k}|x_{\pi_{l_k}^{Q^\pm}})=\frac{\MEANNN{P_{i_m|\pi_{i_m}^P}}{e^{\pm c_\pm f(X_A)}}}{\MEANNN{P_{l_k|\pi_{l_k}^p}}{\MEANNN{P_{i_m|\pi_{i_m}^p}}{e^{\pm c_\pm f(X_A)}}}} \cdot p(x_{i_m}|x_{\pi_{i_m}^P})  
\]
and $ \pi_{l_k}^P\subset \pi_{l_k}^{Q^\pm} = \pi_{l_k}^P\cup (\pi_{i_m}^{Q^\pm}\setminus l_k) $. The latter shows the new edges that the associated graph to $Q^{\pm}$ may have. In this way, we obtain the remaining CPDs given by the second part of \eqref{eq:MFUQ optimizer:1}. It is straightforward that the random variables indexed differently than $\{l_1,\dots,l_{k+1}\}$ inherent the corresponding CPDs of $P$, and thus \eqref{eq:MFUQ optimizer:1} is obtained.

\subsection{Gaussian Bayesian Networks}\label{exGBayesian network1}
In this subsection, we focus on Gaussian Bayesian Networks. It is a special class of Bayesian networks commonly used in natural and social sciences with the CPDs as in \eqref{eq:PGM:def}
being linear and Gaussian \cite{koller2009probabilistic,10.2307/2632102, gomez2013effect, gomez2011evaluating, gomez2007sensitivity}. More specifically, for a Gaussian Bayesian network consisting of variables $X$, each vertex $X_i$ is a linear Gaussian of its parents, i.e.,
\begin{eqnarray}
\label{eq:GBayesian network dist}
    p(x_i|x_{\pi_i}) = \mathcal{N}(\beta_{i0}+\beta_i^T x_{\pi_i}, \sigma_i^2), \quad\textrm{ equivalently}\\
     X_i = \beta_{i0}+\beta_i^T X_{\pi_i} + \epsilon_i,\quad\textrm{ with } \epsilon_i \sim \mathcal{N}(0, \sigma_i^2)\nonumber
\end{eqnarray}
 for  some $\beta_0$, $\sigma_i$ and  $\beta_i=[\beta_{ii_1},\dots,\beta_{ii_{|\pi_i|}}]$.  By the conjugacy properties  of Gaussians, the joint distribution $P$ becomes 
$ p(x)=\mathcal{N}(\mu,\mathcal{C})$, i.e. it 
is also a Gaussian with parameters $\mu$, $\mathcal{C}$, which can be calculated from $\beta_{i0}$, $\beta_i$, and $\sigma_i$ \cite{bishop}. 

\medskip


\begin{theorem}
\label{cor:GBayesian network MFUQ}
Let $P$ be a Gaussian Bayesian network that satisfies \eqref{eq:GBayesian network dist}, and $f(X_k) = a X_k+b$ be a QoI only depends on $X_k$ linearly. 
\medskip

\noindent $(a)$ Then for the model uncertainty indices defined in \eqref{eq:PU MFUQ PGM}, we have
\begin{equation}
\label{eq:GBayesian network MFUQ}
 I^{\pm}(f(X_k), P; \mathcal{D}^\eta) 
    = \pm \sqrt{2a^2\mathcal{C}_{kk}\eta}
\end{equation}
where $\mathcal{C}_{kk}$ is the variance for the marginal distribution of $X_k$. 
\smallskip

\noindent $(b)$ Furthermore, the optimizers $Q^\pm = Q^\pm(\eta) \in \mathcal{D}^\eta$  are given by  \eqref{eq:MFUQ optimizer:1} in Theorem \ref{thm:MFUQ PGM} and are also Gaussian Bayesian networks with same graph structure as $P$. 
\end{theorem}
\noindent {\it Proof.} (a) The distribution of $X_k$ denoted by $P_{\{k\}}$ is Gaussian with variance 
\[
\mathcal{C}_{kk}=\sigma_{k}^2+\beta_k^{T}\mathcal{C}_{\rho_k}\beta_k
\]
where $\mathcal{C}_{\rho_k}$ is the variance of the joint distribution of the random variables $\{X_i:i\in\rho_{k}\}$, \cite[Theorem 7.3]{koller2009probabilistic}. By a straightforward computation, the moment generating function of $\bar{f}(X_k)$ is given by 
\[
\MEANNN{P_{\{k\}}}{ e^{ \pm c\bar{f}(X_k)} }=\exp( a^2c^2 \beta_k^{T}\mathcal{C}_{\rho_k}\beta_k)
\] 
\begin{eqnarray}
I^{\pm}(f(X_k), P; \mathcal{D}^\eta)= \pm \inf_{c>0}\Big[\frac{1}{c} \log \MEANNN{P_{\{k\}}}{ e^{\pm c\bar{f}(X_k)} } + \frac{\eta}{c}\Big]=\pm \inf_{c>0}\Big[a^2c\,\mathcal{C}_{kk}+ \frac{\eta}{c}\Big]\nonumber\\
\end{eqnarray}
Then, the optimal $c$ is given by $c=\sqrt{\frac{\eta}{a^2\mathcal{C}_{kk}}}$ which in turn proves \eqref{eq:GBayesian network MFUQ}.
\medskip

\noindent (b) Next, we show that the graph structure of the $Q^{\pm}$ is the same as $P$. For any $j>k$, by Theorem~\ref{thm:MFUQ PGM}, $q(x_j|x_{\pi_j^{Q^{\pm}}})=p(x_j|x_{\pi_j^{P}})$. For $j=k$,  we compute
\begin{eqnarray}\label{comp:GBN1}
q^\pm(x_{k}|x_{\pi_{k}^{Q^\pm}})&=&\frac{e^{\pm c_\pm f(x_k)}}{\MEANNN{P_{k|\pi_{k}^P}}{e^{\pm c_\pm f(X_k)}}} \cdot p(x_{k}|x_{\pi_{k}^P}) \\
&=&\frac{\exp\left\{-\frac{(x_k-\beta_{k0}-\beta_k^T x_{\pi_k^P} \mp c_{\pm}a\sigma_k^2 )^2}{2\sigma_k^2}\pm c_\pm a(\beta_k^T x_{\pi_k^P})\right\}}{\int_{\mathcal{X}_k} \exp\left\{-\frac{(x_k-\beta_{k0}-\beta_k^T x_{\pi_k^P} \mp c_{\pm}a\sigma_k^2 )^2}{2\sigma_k^2}\pm c_\pm a (\beta_k^T x_{\pi_k^P})\right\}dx_k}\notag\\
&=&\mathcal{N}\left(\beta_{k0}+\beta_k^T x_{\pi_k^P} \pm c_{\pm}a\sigma_k^2, \sigma_k^2 \right)\notag
\end{eqnarray}
Thus $\pi_k^{Q^\pm}=\pi_k^{P}$ since $ c_\pm a(\beta_k^T x_{\pi_k})$ of the numerator and denominator are canceled out. Let $k_m$ be the maximum element of $\pi_k^P=\{k_1,\dots,k_m:k_1<\dots<k_{m-1}<k_{m}\}$ and $\beta_k^{m-1}:=[\beta_{kk_1},\dots, \beta_{kk_{m-1}}]$. Then 
\begin{eqnarray}\label{comp:GBN2}
&&\;\;\;\;\;\;\;\;\;\;\;q^\pm(x_{k_m}|x_{\pi_{k_m}^{Q^\pm}})=\frac{\MEANNN{P_{k|\pi_{k}^P}}{e^{\pm c_\pm f(X_k)}}}{\MEANNN{P_{k_m|\pi_{k_m}^P}}{\MEANNN{P_{k|\pi_{k}^P}}{e^{\pm c_\pm f(X_k)}}}} \cdot p(x_{k_m}|x_{\pi_{k_m}^P}) \\
&&\qquad\;\;=\frac{\exp\left\{\pm c_{\pm}a\beta_k^T x_{\pi_k^P} -\frac{(x_{k_m}-\beta_{k_m0}-\beta_{k_m}^T x_{\pi_{k_m}^P})^2}{2\sigma_{m}^2}\right\}}{\int_{\mathcal{X}_{k_m}} \exp\left\{\pm c_{\pm}a\beta_k^T x_{\pi_k^P} -\frac{(x_{k_m}-\beta_{{k_m}0}-\beta_{k_m}^T x_{\pi_{k_m}^P})^2}{2\sigma_{k_m}^2}\right\}dx_{k_m}}\notag\\
&&\qquad\;\;=\frac{\exp\left\{\pm c_{\pm}a(\beta_{k}^{m-1})^T x_{\pi_k^P\setminus k_m} -\frac{(x_{k_m}-\beta_{{k_m}0}-\beta_{k_m}^T x_{\pi_{k_m}^P}\mp c_{\pm}a\beta_{kk_m}\sigma_{k_m}^2)^2}{2\sigma_{k_m}^2}\right\}}{\int_{\mathcal{X}_{k_m}} \exp\left\{\pm c_{\pm}a(\beta_{k}^{m-1})^T x_{\pi_k^P\setminus k_m} -\frac{(x_{k_m}-\beta_{{k_m}0}-\beta_{k_m}^T x_{\pi_{k_m}^P}\mp c_{\pm}a\beta_{kk_m}\sigma_{k_m}^2)^2}{2\sigma_{k_m}^2}\right\}dx_{k_m}}\notag\\
&&\qquad\;\;=\mathcal{N}\left(\beta_{{k_m}0}+\beta_{k_m}^T x_{\pi_{k_m}^P}\pm c_{\pm}a\beta_{kk_m}\sigma_{k_m}^2, \sigma_{k_m}^2 \right)\notag
\end{eqnarray}
Again, $\pi_{k_m}^{Q^\pm}=\pi_{k_m}^{P}$ as the factor $\exp\left\{\pm c_{\pm}a(\beta_{k}^{m-1})^T x_{\pi_k^P\setminus k_m}\right\}$ in the numerator and denominator are canceled out. The CPD of  the remaining vertices in $\pi_k^P$ are computed in the same way which further implies that their parents do not change. Therefore, the factors in CPDs of $Q^{\pm}$ that  could create new directed edges  appear  in both numerator and denominator and are finally canceled out. We demonstrate \eqref{comp:GBN1} and \eqref{comp:GBN2} as it applies in  Example~\ref{ex:cont} in Appendix~\ref{sec:GBNsup}.

\section{Model Sensitivity Indices for Bayesian networks}\label{sec:QUSI}
In this section, we develop  a  non-parametric sensitivity analysis for Bayesian networks  by refining the concepts of model uncertainty indices introduced in Section~\ref{sec:MFUQ}. 
This is accomplished through designing  localized ambiguity sets suitable for  model uncertainty/perturbations in specific components of the graphical model such as a single CPD.  
\begin{figure}[ht]
\centering
\includegraphics[width=1.3\textwidth]{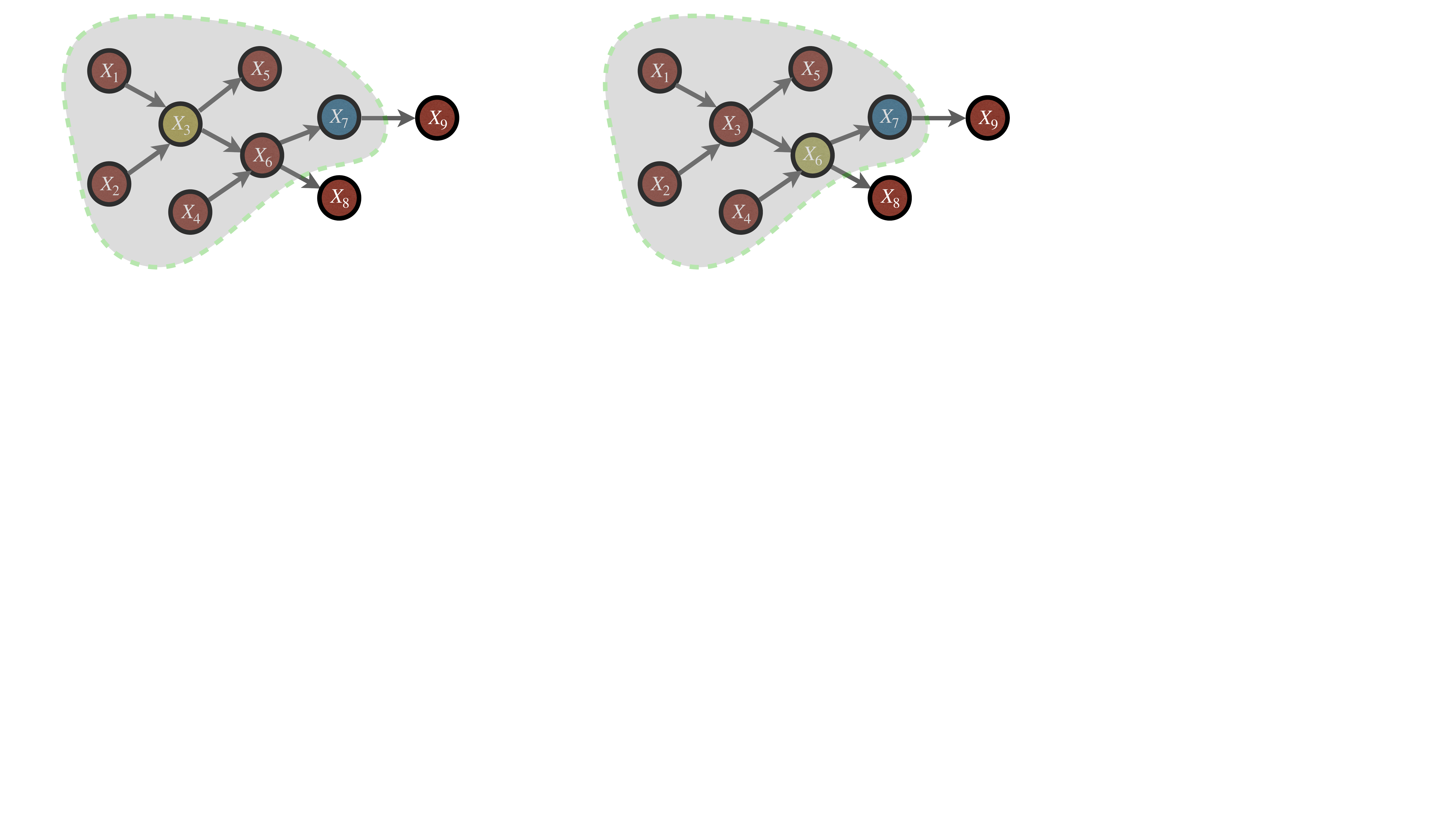}
\vspace{-6.5cm}
\caption{\small{Example of the structure of a Bayesian network baseline model. The QoI is given by $f(X_7)=X_7$ (blue color). We fix $X_7$ and we perturb one vertex at time, e.g. $X_3$ (left) and $X_6$ (right) in green color. The vertices involved in the graph can be classified into  $l\in \bar{\rho}_7^P
=\rho_7^P \cup \{7\}
=\{1,2,3,4,5,6,7\}$ (vertices in the dashed area) and $\{8,9\}$ which are not in $\bar{\rho}_7^P$ (vertices outside of the dashed area), see left and right figures. Based on these figures  and Lemma~\ref{lem:UpperF}, the model sensitivity indices \eqref{eq:PU general MFSI} over $\mathcal{D}^{\eta_l}_l$ and $\mathcal{D}^{\eta_l}_{l,P}$ is 0 for  $l=8,9$,  meaning that perturbations on vertices which are not ancestors of 7 do not affect the QoI,  while perturbations on those vertices in $\bar{\rho}_7^P$ affect the QoI. 
}}
\label{fig:pict4}
\end{figure}
\medskip

\noindent{\bf Notation.} For the notation of this section, we refer to Section~\ref{sec:MFUQ}. Moreover, we denote $\bar{\rho}_k^P:=\rho_k^P \cup \{k\}$.
\medskip

Let $f=f(X_k)$ be a QoI depending only on vertex $k \in V$  and let $l\in V$ be another vertex. The first ambiguity set 
$\mathcal{D}^{\eta_l}_l$
consists of all Bayesian networks (BN) $Q$ that differ from the baseline $P$ only in the CPD at the vertex $l$ while also  allowing for the parents $\pi_l^P$ at $l$ to  change. Namely, 
\begin{equation}\label{eq:set:MFSI general}
\mathcal{D}^{\eta_l}_l =\left \{\begin{array}{cc}\textrm{all BN } Q: R(Q_{l|\pi_l^Q}\|P_{l|\pi_l^P})  \leq \eta_l \textrm{ for all $x_{\pi_i^P} \cup x_{\pi_i^Q}$},\\
 Q_{j|\pi_j} \equiv P_{j|\pi_j} \textrm{ for all } j \neq l
 \end{array}\right\}
\end{equation}
where the parents $\pi_l^Q$ in model $Q$ may  differ from  the parents $\pi_l^P$ in model $P$.

The second ambiguity set $\mathcal{D}^{\eta_l}_{l, P}$ consists of all Bayesian networks (BN) $Q$ that differ from the baseline $P$ only in the CPD at the vertex $l$, however here we require that   $\pi_l^Q=\pi_l^P=\pi_l$, i.e. parents  are not allowed to change:
\begin{equation}\label{eq:set:MFSI}
\mathcal{D}^{\eta_l}_{l,P} =\left \{\textrm{all BN } Q: R(Q_{l|\pi_l}\|P_{l|\pi_l})  \leq \eta_l \textrm{ for all } x_{\pi_l},
Q_{j|\pi_j} \equiv P_{j|\pi_j} \textrm{ for all } j \neq l\right\}
\end{equation}
Note that 
\begin{equation}\label{eq:subset}
\mathcal{D}^{\eta_l}_{l,P}\subset\mathcal{D}^{\eta_l}_{l}.
\end{equation}
We accordingly define the  model sensitivity indices of the QoI $f(X_k)$  as 
\begin{equation}\label{eq:PU general MFSI}
    I^{\pm}(f(X_k), P; \mathcal{Q}_{\eta_l}) = \underset{Q \in\mathcal{Q}_{\eta_l}}{\mathrm{sup/inf}} \   \MEANNN{Q}{f(X_k)} -\MEANNN{P}{f(X_k)} 
\end{equation}
where $\mathcal{Q}_{\eta_l}=\mathcal{D}^{\eta_l}_l$ or $\mathcal{D}^{\eta_l}_{l,P}$ given by \eqref{eq:set:MFSI general} and \eqref{eq:set:MFSI} respectively.

The evaluation  of these  model sensitivity indices will necessarily depend on the relative graph position of vertices $k, l \in V$ and in particular if $l$ is an ancestor of $k$. In particular we have the following:
%
\begin{lemma}\label{lem:UpperF}
Let $Q\in \mathcal{Q}_{\eta_l}$ where $\mathcal{Q}_{\eta_l}=\mathcal{D}^{\eta_l}_l$ or $\mathcal{D}^{\eta_l}_{l,P}$. Then
\begin{eqnarray}
\label{eq:MFSI dif}
 \;\;\;\;\;\MEANNN{Q}{f(X_k)} - \MEANNN{P}{f(X_k)}&=& 
\begin{cases} 
    \MEANNN{P_{\rho_l^P}}{ \MEANNN{Q_{l|\pi_l^Q}}{F} - {\MEANNN{P_{l|\pi_l^P}}{F}}}&, l \in \bar{\rho}_k^P \\
       0 &, l \notin \bar{\rho}_k^P 
\end{cases}
\end{eqnarray}
where
\begin{eqnarray}\label{eq:F}
    F:=F(x_l,x_{\rho_l^P})&=&  \int_{\mathcal{X}_{\bar{\rho}_k^P \setminus \rho_l^P\cup\{l\}}} f(x_k) \prod_{i \in \bar{\rho}_k^P \setminus \rho_l^P\cup\{l\}} P(d x_i|x_{\pi_i^P})\\
    & = &\MEANNN{P_{\{k\}| \bar{\rho}_l^P}}{f(X_k)}\, ,\notag
\end{eqnarray}
and  the last expectation is  with respect to the conditional distribution of $X_k$ given $X_{\bar{\rho}_l^P}=x_{\bar{\rho}_l^P}$.
\end{lemma} 

\noindent  The proof of Lemma~\ref{lem:UpperF} is a direct calculation of the difference between the expectations of $f(X_k)$ and is based on a rearrangement between the CPDs of $X_{\rho_k^P\cup\{k\}}$, $X_{\rho_l^P}$ and $X_l$ with respect to $P$ and $Q$, see  Appendix~\ref{app:lem:UpperF}, while a concrete computation of $F$ is given in Appendix~\ref{ex:contF} for the Bayesian network of Example~\ref{ex:cont}.

 Next, following the structure of  Theorem~\ref{thm:MFUQ PGM} and using  Lemma~\ref{lem:UpperF}, we present our results on tightness and optimal distributions over $\mathcal{D}^{\eta_l}_{l}$ and  $\mathcal{D}^{\eta_l}_{l,P}$   stated in Theorem~\ref{thm:MFSI general} and ~\ref{thm:MFSI} respectively. Theorem~\ref{thm:MFSI} could be thought of as a subcase of Theorem~\ref{thm:MFSI general} due to \eqref{eq:subset}, however tightness on $\mathcal{D}^{\eta_l}_{l,P}$ cannot be accomplished unless the additional condition \eqref{eq:tightness cond} is assumed. All these results  are summarized in a schematic in 
  Figure~\ref{fig:readgraph}. 

\begin{theorem}[Model Sensitivity Indices--vary graph structure and CPD]
\label{thm:MFSI general}
Let $P$ be a Bayesian network with density defined as \eqref{eq:PGM:def}, and $f(X_k)$ be a QoI  that only depends on $X_k$. Let also $\bar{f}(X_k)$ be the centered QoI  with finite moment generating function (MGF), $\MEANNN{P}{e^{c\bar{f}(X_{k})}}$, in a neighborhood of the origin.
\smallskip

\noindent $(a)$ {\bf{Tightness.}} For the model sensitivity indices defined in \eqref{eq:PU general MFSI}, there exist $0< \eta_{\pm} \le \infty$, such that for any $\eta \le \eta_{\pm}$,
\begin{eqnarray}
\label{eq:MFSI general}
\;\;\;\;\;\;\;I^{\pm}(f(X_k), P; \mathcal{D}^{\eta_l}_l) &=&\underset{Q \in \mathcal{D}^{\eta_l}_l}{\mathrm{sup/inf}} \   \MEANNN{Q}{f(X_k)} -\MEANNN{P}{f(X_k)} \\
&=& 
\begin{cases} 
      \pm \MEANNN{P_{\rho_l^P}}{
\inf_{c>0}\Big[ \frac{1}{c} \log \MEANNN{P_{l|\pi_l^P}}{e^{\pm c\bar{F}}} +\frac{\eta_l}{c} \Big]} &, l \in \bar{\rho}_k^P\nonumber \\
       0 &, l \notin \bar{\rho}_k^P 
\end{cases}\notag\\
&=& \MEANNN{Q^\pm}{f(X_k)} -\MEANNN{P}{f(X_k)}\nonumber
\end{eqnarray}
where $\bar{F}$ is the centered function of $F$ defined in \eqref{eq:F}, $\eta_l\equiv \eta$
and $Q^\pm(\cdot) \equiv Q^\pm(\cdot\;;\pm c_\pm) \in \mathcal{D}_{l}^{\eta_l}$  are  Bayesian networks of the form \eqref{eq:PGM:def} that depend on $\eta_l$   with $f(X_k)$
and $c_{\pm}\equiv c_{\pm}(x_{\rho_l^P};\eta_l)$ being  functions of $x_{\rho_l^P}$, depend on $\eta_l$ and are determined by the equations
\begin{equation}
\label{eq:opt:MFSI:c}
    R(Q^{\pm}_{l|\pi_l^{Q^{\pm}}}\|P_{l|\pi_l^P})=\eta_l.
\end{equation}
\noindent $(b)$ {\bf{Graph Structure of $Q^{\pm}$.}} The optimal distributions  $Q^\pm$ are the probability measures with densities given by
\begin{equation}
\label{eq:opt:MFSI:1}q^\pm(x_i|x_{\pi_i^{Q^\pm}})=
\begin{cases}
        p( x_i|x_{\pi_i^P}) &, i \neq l \\
      \frac{e^{\pm c_\pm F(x_l,x_{\rho_l^P})}}{\MEANNN{P_{l|\pi_l^P}}{e^{\pm c_\pm F(X_l,x_{\rho_l^P})}}} p(x_l|x_{\pi_l^P}) &, i=l
   \end{cases}
   \end{equation} 
The structure of the first and second part of \eqref{eq:opt:MFSI:1} satisfy $\pi_i^{Q^\pm} \equiv \pi_i^P$ and $\pi_l^P \subset \pi_l^{Q^\pm} \subset \rho_l^P=\rho_l^{Q^\pm}$ respectively.
\end{theorem}


\begin{proof} The proof of $(a)$ and $(b)$ are worked together and is split into two main steps.
\medskip

\noindent \textbf{Step 1: Model sensitivity indices:} For $l \in \bar{\rho}_k^P$, we denote  $\pi_l := \pi_l^Q \cup \pi_l^P$, and $\rho_i := \rho_i^Q \cup \rho_i^P$ for all $i$.  We define 
\begin{equation}
    Q(dx_l|x_{\pi_l}):= Q(dx_l|x_{\pi_l^Q}) \textrm{ for all } x_{\pi_l},\qquad    P(dx_l|x_{\pi_l}):= P(dx_l|x_{\pi_l^P}) \textrm{ for all } x_{\pi_l}
\end{equation}
We now use Lemma~\ref{lem:UpperF} and we further bound the right hand side of the first part of \eqref{eq:MFSI general} as follows:
  \begin{eqnarray}   
  \label{eq:thm2:1}
&&\;\;\;\;\;\underset{Q \in \mathcal{D}^{\eta_l}_l}{\mathrm{sup}} \   \MEANNN{P_{\rho_l}}{\MEANNN{Q_{l|\pi_l}}{F}- {\MEANNN{P_{l|\pi_l}}{F}}}
   \leq \MEANNN{P_{\rho_l}}{\underset{Q \in \mathcal{D}^{\eta_l}_l}{\mathrm{sup}}\ \MEANNN{Q_{l|\pi_l}}{F} - {\MEANNN{P_{l|\pi_l}}{F}}}\\
   &&\qquad= \MEANNN{P_{\rho_l}}{\underset{Q_l \in \mathcal{E}^{\eta_l}_l}{\mathrm{sup}}\ \MEANNN{Q_{l|\pi_l}}{F} - {\MEANNN{P_{l|\pi_l}}{F}}}\notag
   \end{eqnarray}
where $\mathcal{E}^{\eta_l}_l$ is the ambiguity set for CPDs at $l$ defined as
\begin{equation}
    \mathcal{E}_l^{\eta_l}:= \{\textrm{all CPD } Q_{l|\pi_l}: R(Q_{l|\pi_l}\|P_{l|\pi_l}) \leq \eta_l \textrm{ for all $x_{\pi_l}=x_{\pi_l}^P \cup x_{\pi_l}^Q$}\}
\end{equation}
By using Lemma \ref{lemma:bounds}, for any given $X_{\rho_l}=x_{\rho_l}$, we have
\begin{equation}
    \underset{Q_l \in \mathcal{E}^{\eta_l}_l}{\mathrm{sup}}\ \MEANNN{Q_{l|\pi_l}}{F} - {\MEANNN{P_{l|\pi_l}}{F}} \leq \inf_{c>0}\Big[ \frac{1}{c} \log \MEANNN{P_{l|\pi_l}}{e^{c\bar{F}(X_l,X_{\rho_l})}}+\frac{\eta_l}{c} \Big].
\end{equation}
Hence \eqref{eq:thm2:1} implies that
\begin{equation}
    \underset{Q \in \mathcal{D}^{\eta_l}_l}{\mathrm{sup}} \   \MEANNN{Q}{f(X_k)} -\MEANNN{P}{f(X_k)}  \leq \MEANNN{P_{\rho_l}}{
\inf_{c>0}\Big[ \frac{1}{c} \log \MEANNN{P_{l|\pi_l}}{e^{c\bar{F}(X_l,X_{\rho_l})}}+\frac{\eta_l}{c} \Big]} 
\end{equation}

\noindent
\textbf{Step 2: Tightness of the bounds:}
As in Theorem \ref{thm:MFUQ PGM}, for any given $x_{\rho_l^P}$, we can consider the conditional measure $P^{c_+}_{l|\rho_l^P}$ defined by
\begin{equation}
    \frac{dP^{c_+}_{l|\rho_l^P}}{dP_{l|\pi_l^P}} = \frac{e^{c_+(x_{\rho_l^P}) F(x_l,x_{\rho_l^P})}}{\MEANNN{P_{l|\pi_l^P}}{e^{c_+(x_{\rho_l^P}) F(X_l,x_{\rho_l^P})}}}
\end{equation}
where $c_{+}(x_{\rho_l^P})$ is a function of $x_{\rho_l^P}$ determined by $
R(P^{ c_{+}}_{l|\pi_l^P}\|P_{l|\pi_l^P})
=\eta_l$.
By using Lemma \ref{lemma:tightness}, we define
\begin{equation}
    q^+_l(x_l|x_{\pi_l^{Q^+}}):=P^{c_+}_{l|\rho_l^P} \propto e^{c_+(x_{\rho_l^P}) F(x_l,x_{\rho_l^P})} p(x_l|x_{\pi_l^P}) \quad \textrm{for all $x_{\pi_l^{Q^+}}$}\, .
\end{equation}
Note that $\pi_l^{Q^+}$ depends on $\pi_l^P$ and $F(x_l,x_{\rho_l^P})$, hence $\pi_l^P \subset \pi_l^{Q^+} \subset \rho_l^P$, and $\rho_l^{Q^+} = \rho_l^P$. Therefore, using the same notation as in Step 1, for $\pi_l =\pi_l^{Q^+}$, $\rho_l =\rho_l^{Q^+}$, we have
\begin{equation}
     \MEANNN{Q^+_{l|\pi_l}}{F} - {\MEANNN{P_{l|\pi_l}}{F}} = \inf_{c>0}\Big[ \frac{1}{c} \log \MEANNN{P_{l|\pi_l}}{e^{c\bar{F}}}+\frac{\eta_l}{c} \Big]\ .
\end{equation}
Furthermore,    $R(Q^+_{l|\pi_l}\|P_{l|\pi_l}) \leq \eta_l$ for all $x_{\pi_l}$ and hence  $ Q^+_l \in \mathcal{E}^{\eta_l}_l$. Let $q^+(x) = q^+_l(x_l|x_{\pi_l})\prod_{i \neq l}p(x_i|x_{\pi_i})$, then $Q^+ \in \mathcal{D}^{\eta_l}_l$, and $$\MEANNN{Q^+}{f(X_k)} - \MEANNN{P}{f(X_k)} =\MEANNN{P_{\rho_l}}{
\inf_{c>0}\Big[ \frac{1}{c} \log \MEANNN{P_{l|\pi_l}}{e^{c\bar{F}}}+\frac{\eta_l}{c} \Big]}, 
$$
and thus  \eqref{eq:MFSI general} is proved. The calculations  for $\underset{Q \in \mathcal{D}^{\eta_l}_l}{\mathrm{inf}} \   \MEANNN{Q}{f(X_k)} -\MEANNN{P}{f(X_k)}$ are similar.
 \end{proof}
 
 We turn next to the ambiguity set $\mathcal{D}^{\eta_l}_{l,P}$ defined in \eqref{eq:set:MFSI} and its corresponding index.
Due to Theorem~\ref{thm:MFSI general} and \eqref{eq:subset}, the following uncertainty bound holds for  $\mathcal{D}^{\eta_l}_{l,P}$:
\begin{eqnarray}
\label{ineq:MFSI variational}
 I^{+}(f(X_k), P; \mathcal{D}^{\eta_l}_{l,P}) &= &\underset{Q \in \mathcal{D}^{\eta_l}_{l,P}}{\mathrm{sup}}\  \MEANNN{Q}{f(X_k)} -\MEANNN{P}{f(X_k)}\nonumber\\
 & \leq &  \MEANNN{P_{\rho_l}}{
\inf_{c>0}\Big[ \frac{1}{c} \log \MEANNN{P_{l|\pi_l}}{e^{ c\bar{F}(X_l,X_{\rho_l})} }+\frac{\eta_l}{c} \Big]}
\end{eqnarray}
for any $l \in \bar{\rho}_{k}^P$, see also Figure~\ref{fig:readgraph}. A similar bound holds for $ I^{-}(f(X_k), P; \mathcal{D}^{\eta_l}_{l,P})$. However, the next theorem provides a  condition  on the Bayesian network $P$ that   implies    equality in  \eqref{ineq:MFSI variational}, see \eqref{eq:tightness cond} and  Fig~\ref{fig:lastmainthm}.
\begin{theorem}[Model Sensitivity Indices--only vary CPD]
\label{thm:MFSI}
Let $P$ be a Bayesian network with density defined as \eqref{eq:PGM:def}, and $f(X_k)$ be a QoI that only depends on $X_k$ with  its centered QoI $\bar{f}(X_k)$ having finite moment generating function (MGF), $\MEANNN{P}{e^{c\bar{f}(X_k)}}$, in a neighborhood of the origin. 
\smallskip

\noindent$(a)$ For $l \notin \bar{\rho}_k^P$, $ \MEANNN{Q}{f(X_k)} - \MEANNN{P}{f(X_k)}=0$, for any $Q\in \mathcal{D}^{\eta_l}_{l,P}$.
\medskip

\noindent$(b)$ For $l \in \bar{\rho}_k^P$  satisfying the condition
\begin{equation}
\label{eq:tightness cond}
    X_k \perp X_{\rho_l \setminus \pi_l} | X_{\pi_l},
\end{equation}
i.e.,  $X_k$ is independent of all the ancestors of $X_l$ given the parents of $X_l$,  there exist probability measures $Q^\pm = Q^\pm(\eta) \in \mathcal{D}^{\eta_l}_{l,P}$ given by \eqref{eq:opt:MFSI:c} - \eqref{eq:opt:MFSI:1} such that
\begin{eqnarray}
    \MEANNN{Q^\pm}{f(X_k)} -\MEANNN{P}{f(X_k)} &=& \underset{Q \in \mathcal{D}^{\eta_l}_{l,P}}{\mathrm{sup/inf}}\  \MEANNN{Q}{f(X_k)} -\MEANNN{P}{f(X_k)}
\end{eqnarray}
\item[$(c)$] For $l \in \bar{\rho}_k^P$  such that  \eqref{eq:tightness cond} is not satisfied, \eqref{ineq:MFSI variational} holds.
\begin{figure}[ht]
\centering
\includegraphics[width=1.2\textwidth]{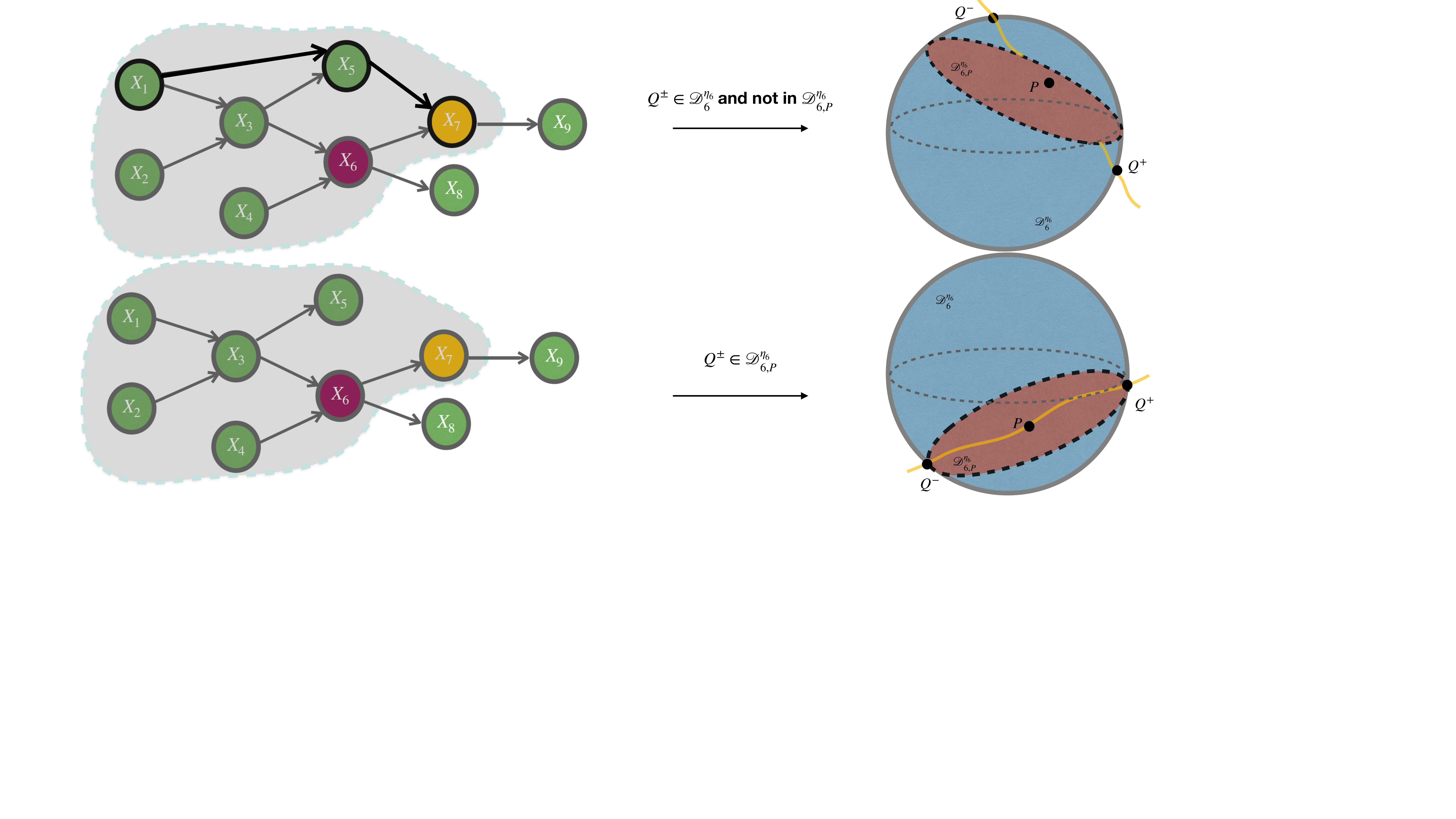}
\vspace{-4cm}
\caption{\small{{\bf(Left)} Two examples of the structure of a baseline Bayesian network. The QoI is $f(X_7)$ in yellow (thus k=7) and $l=6$ in purple. {\bf(Right)} Schematic of  relationships between the  ambiguity sets $\mathcal{D}^{\eta_6}_{6},\mathcal{D}^{\eta_6}_{6,P}$. They share the same boundary and thus we represent $\mathcal{D}^{\eta_6}_{6}$ as a sphere in blue, while $\mathcal{D}^{\eta_6}_{6,P}$ as an embedded disc in brown. The yellow curve  in the both figures demonstrates the parametric family of Bayesian Networks $P^c$ with $d P^c_{l|\pi_l}=dP_{l|\pi_l}$ for $l\neq6$ and $dP^c_{6|\pi_6}\propto \exp\{c F(x_6,x_{\rho_6^P})\}dP_{6|\pi_6}$. The {\bf top graph}  does not satisfy condition \eqref{eq:tightness cond} since  $X_1$ is not conditionally independent of $X_7$ given $X_{\pi_6}$. This is illustrated through the path $X_1 \to X_5 \to X_7$ in black. The function $F$ given by \eqref{eq:F} depends on $x_1$ and $x_6$ which makes the parents of $X_6$ in the optimizers $Q^{\pm}$ be different than its parents in $P$ and thus $Q^{\pm}\notin\mathcal{D}^{\eta_6}_{6,P}$ (in general) as illustrated in the top left picture. The {\bf bottom graph}  could achieve the equality in \eqref{eq:MFSI general}  since it satisfies  condition \eqref{eq:tightness cond} ($X_{\rho_6\setminus\pi_6} = \{X_1,X_2\}$ are connected with $X_7$ only through $X_3 \in X_{\pi_6}$). The function $F$ depends on $x_3$ and $x_6$ and  $\pi_6^Q = \{3, 4 \}=\pi_6$ which makes  $Q^{\pm}\in\mathcal{D}^{\eta_6}_{6,P}$, see bottom right picture. }}\label{fig:lastmainthm}
\end{figure}


\end{theorem}

\begin{proof}
Part $(a)$ and $(c)$ are straightforward consequences of Lemma~\ref{lem:UpperF} and \eqref{ineq:MFSI variational} respectively. The proof of Part $(b)$ is as follows: For $l\in  \bar{\rho}_k^P$ with $X_k \perp X_{\rho_l \setminus \pi_l} | X_{\pi_l}$, we have $F(x_l,x_{\rho_l^P})=F(x_l,x_{\pi_l})$, 
then the proof is the same as the proof in Theorem \ref{thm:MFSI general}. Indeed, let

\begin{equation}\label{inequality*}    q^{+}_l(x_l|x_{\pi_l^{Q^+}})=\frac{e^{\pm c_{+} F(x_l,x_{\rho_l^P})}}{\MEANNN{P_{l|\pi_l^P}}{e^{ c_{\pm} F(X_l,x_{\pi_l^P})}}} p(x_l|x_{\pi_l^P}) \quad \textrm{for all $x_{\pi_l^{Q^{+}}}$ and $\pi_l^{Q^+} =\pi_l^P$},
\end{equation} 
where $c_{+}\equiv c_{+}(x_{\pi_l^P};\eta_l)$ are functions of  $x_{\pi_l}$ (since $F$ only depends on $x_l$ and $x_{\pi_l}$), depend on $\eta_l$ and are determined by the equations $R(Q^{+}_{l|\pi_l^{Q^{+}}}\|P_{l|\pi_l^P})=\eta_l$. Therefore, the density of $Q^{+}$ are given by $q^{+}(x) = q^{+}_l(x_l|x_{\pi_l^{Q^+}})\prod_{i \neq l}p(x_i|x_{\pi_i})$.  Thus, $Q^{+}_l \in \mathcal{D}^{\eta_l}_{l,P}$ make \eqref{inequality*} equality. Therefore we can conclude that
\begin{equation}
    \underset{Q \in \mathcal{D}^{\eta_l}_{l,P}}{\mathrm{sup}} \   \MEANNN{Q}{f(X_k)} -\MEANNN{P}{f(X_k)}  = \MEANNN{P_{\rho_l}}{
\inf_{c>0}\Big[ \frac{1}{c} \log \MEANNN{P_{l|\pi_l}}{e^{c\bar{F}(X_l,X_{\rho_l})}}+\frac{\eta_l}{c} \Big]}.
\end{equation}
The case of $\underset{Q \in \mathcal{D}^{\eta_l}_l}{\mathrm{inf}} \   \MEANNN{Q}{f(X_k)} -\MEANNN{P}{f(X_k)}$ is treated similarly.  By Lemma~\ref{lem:UpperF}, for $l \notin \bar{\rho}_k^P$ and $Q\in \mathcal{D}^{\eta_l}_{l,P}$, $ \MEANNN{Q}{f(X_k)} - \MEANNN{P}{f(X_k)}=0$.
\end{proof}
\begin{remark}The condition $X_k \perp X_{\rho_l \setminus \pi_l} | X_{\pi_l}$ can be satisfied when $\rho_l \cap \rho_i \subset \pi_l$ for all $i \in \bar{\rho}_{k} \setminus\bar{\rho}_{l}$, i.e. any path from $X_{\rho_l \setminus \pi_l}$ to $X_k$ must go through $X_{\pi_l}$, for instance, all Markov chains, tree/polytree structure model, etc. Two simple examples where the assumption is satisfied or violated are shown in Figure~\ref{fig:lastmainthm}. This condition is also satisfied by the baseline Bayesian network discussed in Section~\ref{sec: ORR}.
\end{remark}

\begin{remark} Note that for the model sensitivity indices shown in \eqref{eq:MFSI general} in Theorem \ref{thm:MFSI general} or the uncertainty bounds shown in \eqref{ineq:MFSI variational} in Theorem \ref{thm:MFSI}, sometimes it might be practically difficult to find the infimum for every conditioning $\rho_l$. However, we can use an alternative looser bound by Jensen's inequality, i.e.
\begin{eqnarray}
 I^{+}(f(X_k), P; \mathcal{D}^{\eta_l}_{l,P}) &\leq&   \MEANNN{P_{\rho_l}}{
\inf_{c>0}\Big[ \frac{1}{c} \log \MEANNN{P_{l|\pi_l}}{e^{ c\bar{F}(X_l,X_{\rho_l})} }+\frac{\eta_l}{c} \Big]} \notag\\
&\leq& 
\inf_{c>0}\Big[ \MEANNN{P_{\rho_l}}{\frac{1}{c} \log \MEANNN{P_{l|\pi_l}}{e^{ c\bar{F}(X_l,X_{\rho_l})}} }+\frac{\eta_l}{c} \Big]
\end{eqnarray}
the model sensitivity index $ I^{-}(f(X_k), P; \mathcal{D}^{\eta_l}_{l,P})$ can be treated analogously. Moreover, the corresponding bounds for $I^{\pm}(f(X_k), P; \mathcal{D}^{\eta_l}_{l})$ are similar. Moreover, if $\rho_l^P=\emptyset$, then expectation $\MEANNN{P_{\rho_l^P}}{\cdot}$ does not enter in  the overall calculations, and hence 
\[
I^{\pm}(f(X_k), P; \mathcal{D}^{\eta_l}_l)=
\inf_{c>0}\Big[ \frac{1}{c} \log \MEANNN{P_{l|\pi_l^P}}{e^{\pm c\bar{F}}} +\frac{\eta_l}{c} \Big], \qquad l \in \bar{\rho}_k^P
\]
e.g for $l\in\{1,2,4\}$ and $k=7$ as illustrated in Figure~\ref{fig:pict4}. This is a special case, however it is used in the computation of the model sensitivity indices for the materials design problem in Section~\ref{sec: ORR}.
\end{remark}
\subsection{Gaussian Bayesian networks}
 Next, we develop model sensitivity indices $I^{\pm}(f(X_k), P; \mathcal{D}_l^{\eta_l})$, when $P$ is a Gaussian Bayesian Network satisfying \eqref{eq:GBayesian network dist}, and $f(X_k)$ depends on $X_k$ linearly. We first use Lemma~\ref{eq:F}, along with the fact that each model component is linear Gaussian of its parents, and compute $F$ and $\bar F$ explicitly. We show that $\bar{F}$ depends only on the $l$-th component and its parents $\pi_l^P$. Then, to implement Theorem~\ref{thm:MFSI general}, we  calculate the MGF of $\bar F$ with respect to $P_{l|\pi_l^P}$. We prove that  it no longer depends on $\pi_l^P$, due to cancellations between the terms involving  $\pi_l^P$. Thus, the expectation $\mathbb{E}_{P_{\rho_l^P}}$ does not enter  the overall computation of \eqref{eq:MFSI general}. Finally, we  prove that $Q^\pm \in \mathcal{D}_{l,P}^{\eta_l}$, i.e. $Q^\pm$ are Gaussian Bayesian Networks with the same structure as $P$, without requiring condition \eqref{eq:tightness cond} be satisfied, as explained in the proof of the theorem.

\begin{theorem}[Model Sensitivity indices for Gaussian Bayesian Networks]
\label{cor:Gaussian Bayesian network MFSI}
Let $P$ be a Gaussian Bayesian network satisfies \eqref{eq:GBayesian network dist}, and $f(X_k) = a X_k+b$ be a QoI only depends on $X_k$ linearly. Then,
\begin{itemize} 
\item[$(a)$]For the model sensitivity indices defined in \eqref{eq:PU general MFSI}, we have
\begin{equation}\label{eqforIplusminus}
    I^{\pm}(f(X_k), P; \mathcal{D}_l^{\eta_l}) \equiv  I^{\pm}(f(X_k), P; \mathcal{D}_{l,P}^{\eta_l}) 
\end{equation}
and the optimizer $Q^\pm = Q^\pm(\eta) \in \mathcal{D}_{l,P}^{\eta_l} \subset \mathcal{D}_l^{\eta_l}$ given by \eqref{eq:opt:MFSI:1} - \eqref{eq:opt:MFSI:c} are also Gaussian Bayesian networks with same graph structure as $P$. Furthermore, for $l \in \pi_k^P$ and $l \notin \rho_{\pi_j}^P$ for all $j \in \pi_k$, $j\neq l$, we have
\begin{equation}\label{GBayesian networkopt1}
I^{\pm}(f(X_k), P; \mathcal{D}_l^{\eta_l}) =
      \pm |\beta_{kl}|\sqrt{2a^2\sigma_l^2\eta_l} 
\end{equation}
\item[$(b)$] Moreover, for any $l \in \rho_k^P$, we also have
\begin{equation}\label{GBayesian networkopt2}
I^{\pm}(f(X_k), P; \mathcal{D}_l^{\eta_l}) =
      \pm |\tilde{\beta}_{kl}|\sqrt{2a^2\sigma_l^2\eta_l} 
\end{equation}
for a computable constant $\tilde{\beta}_{kl}$. 
\end{itemize}
\end{theorem}
\begin{proof}
Let $f(X_k)=aX_k+b$ and $l\in\bar{\rho}_k$, then by a straightforward calculation of $F$ given by \eqref{eq:F} can be expressed as  
\begin{eqnarray}\label{FforGBNs}
F(X_l,X_{\rho_l})&=&a\tilde{\beta}_{k_0}+a\sum_{j\in\rho_l}\tilde{\beta}_{kj}X_j+a\tilde{\beta}_{kl}X_l +b
\end{eqnarray}
for some computable $\tilde{\beta}_{k0}, \tilde{\beta}_{kj}$ with $j\in\rho_l$ (see  Example~\ref{ex:contBeta}  where we compute $\beta_{kl}$ and $\tilde{\beta}_{kl}$). Furthermore, by using \eqref{eq:GBayesian network dist}, the centered $F$ denoted by $\bar F$
\begin{equation}\label{FforGBNs2}
\bar F(X_l,X_{\pi_l})=\tilde{\beta}_{kl}a(X_l-\beta_{l0}-\beta_l^T X_{\pi_l})
\end{equation}
and thus the MGF of $\bar F$ with respect to $P_{l|\pi_l}$ in the second equality of \eqref{eq:MFSI general} is 
\begin{eqnarray}\label{FforGBNs1}
    \MEANNN{P_{l|\pi_l}}{e^{\pm c\bar F(X_l,X_{\pi_l})}}&=&\int_{\mathcal{X}_l}e^{\pm c_{\pm}\tilde{\beta}_{kl}ax_l}e^{\mp c_{\pm}\tilde{\beta}_{kl}a(\beta_{l0}+\beta_l^T x_{\pi_l})}dx_l\\
    &=&e^{\pm c\tilde{\beta}_{kl}a(\beta_{l0}+\beta_l^T x_{\pi_l})+c^2\tilde{\beta}_{kl}^2a^2\frac{\sigma_{l}^2}{2}}e^{\mp c\tilde{\beta}_{kl}a(\beta_{l0}+\beta_l^T x_{\pi_l})}\\
    &=&e^{c^2\tilde{\beta}_{kl}^2a^2\frac{\sigma_{l}^2}{2}}\nonumber
\end{eqnarray}
We compute the minimization problem of \eqref{eq:MFSI general} by following the steps given in the proof of Theorem~\ref{cor:GBayesian network MFUQ}.

Regarding the structure of $Q^\pm$,  $Q^\pm \in \mathcal{D}^{\eta_l}_{l,P}$, i.e. the graph of $Q^\pm$ is same as $P$, as proved in Theorem~\ref{cor:GBayesian network MFUQ} (see also Example~\ref{ex:cont}) where we showed that due to cancellations  that may occur in the derivation of CPDs $q^{\pm}$ the graph remains the same.
\end{proof}

\section{Stress tests, Ranking and Correctability}
\label{sec: correctability}

Based on the model   sensitivity  indices discussed in  Section~\ref{sec:QUSI} we build an iterative approach that ranks the Bayesian network components of the baseline  $P$ according to their model sensitivity indices and subsequently improve its predictive ability for specific QoIs. 
The model misspecification $\eta_l$  of the ambiguity sets  can be either set up by the user e.g. when the data for  component $l$ are very sparse or absent, or can be estimated from  data,  building a data-informed ambiguity set. Once $\eta_l$'s are specified, we rank the sensitivity indices  $I^\pm(f(X_k),P;\mathcal{D}^{\eta_l}_l)$  for all   vertices $l$ based on their  relative size. 
Here the largest indices correspond to the  most ``sensitive" CPDs in the sense that they have the largest effect on the uncertainty of the QoI.    From a Machine Learning perspective, such a ranking procedure is a form of {\it interprability}, i.e. the ability to identify cause and effect in a model \cite{DoshiVelez2017TowardsAR,MILLER20191,Burges2016ABC} and {\it explainability}, i.e. the ability to explain model outputs based on modeling and data choices made during the learning of the baseline \cite{8466590}.  

Once the ranking is completed,  we turn to    correcting the most influential components  of a baseline Bayesian network, a task also referred to as {\it correctability} in Machine Learning; namely  the ability to correct predictive
errors without introducing or (tightly) controlling any newly created errors (see Theorem~\ref{THM:GBNsCor} for Gaussian Bayesian Networks) \cite{8466590, Gunning_Aha_2019,Burges2016ABC}. To this end we need to assess the impact of limited data, seek additional data targeting specific model components,
 or update some of the CPDs or the graph  of the baseline Bayesian $\{G, P\}$. All these elements can be organized in  a 4-step  strategy discussed next, while they are implemented in an example in materials design for fuel cells in Section~\ref{sec: ORR}. 
 \smallskip
 
\noindent  {\bf Notation.} We remind that $P_{l|\pi_l}$ is  the conditional distribution of $X_l$ with the given parents values $X_{\pi_l}=x_{\pi_l}$. However, we write $P_{l|X_{\pi_l^P}}$   when $X_{\pi_l}$ is still random variable and  $P_{l|X_{\pi_l}=x_{\pi_l}}$ when we simply emphasize the dependence on  given parents, see  Step 1 below and the KL chain rule in Appendix~\ref{subsec:eta in PGM}. Finally,
for each vertex $l\in V$ we use the notation $\pi_l:=\pi_l^Q \cup \pi_l^P$ when we consider simultaneously  the parents for both models.

\smallskip
\noindent
{\bf Step 1: Stress tests and model sensitivity.} In this step, we determine the level of model misspecification  $\eta_l$ for each component $l \in V$ of the baseline  using  data-informed or  user-determined stress tests.  In particular:

 \noindent{\it A. Data-informed stress tests.}  
 For Bayesian networks  (or parts thereof) for which there is a reasonable amount of data here
 we construct data-informed ambiguity sets  \eqref{eq:set:MFUQ}, \eqref{eq:set:MFSI general} and \eqref{eq:set:MFSI} respectively. %
 The corresponding levels of model misspecification $\eta, \eta_l$ are computed  as  distances between the baseline $P$ and the  data distribution $Q$; the latter can be selected   as a histogram or a Kernel Density Estimation (KDE).  
 In that sense,  we provide  surrogate values for the model misspecifications $\eta$ or  $\eta_l$ taking into account   the  ``real" model which is  accessible  only through the available data.
 In these calculations we are 
 taking full advantage of the graph structure of the models. 
 First, we discuss  the model uncertainty ambiguity set $\mathcal{D}^{\eta}$ in 
 \eqref{eq:set:MFUQ}. Using  
 the chain rule of KL divergence for Bayesian networks (Appendix~\ref{subsec:eta in PGM}) we define  a data-informed misspecification $\eta$ as 
\begin{equation}\label{eq:eta-all}
\eta:=R(Q\|P)=\sum_{l=1}^n \MEANNN{Q}{\eta_l^{\pi_l}},  \quad Q\in\mathcal{D}^{\eta}
\end{equation}
where $\eta_l^{\pi_l}$ is a function  of $X_{\pi_l}$ given by
\begin{equation}\label{eq:eta-all1}
\eta_l^{\pi_l} 
=  \MEANNN{Q}{R(Q_{l|X_{\pi_l^Q}}\|P_{l|X_{\pi_l^P}})}=
\int_{\mathcal{X}_l} \log \frac{Q_{l|X_{\pi_l^Q}}}{P_{l|X_{\pi_l^P}}}Q_{l|X_{\pi_l^Q}}dx_l \, .
\end{equation}
Second,  for the case of model sensitivity,  definition \eqref{eq:eta-all} reduces to 
\begin{equation}\label{eq:eta_vertex}
  \eta_l=R(Q\|P)=\MEANNN{Q}{\eta_l^{\pi_l}},\quad Q\in\mathcal{Q}_{\eta_l} 
\end{equation}
where $\mathcal{Q}_{\eta_l}$ is given by  \eqref{eq:set:MFSI general} or \eqref{eq:set:MFSI}; to obtain this simplification of \eqref{eq:eta-all} we  used  the structure of the ambiguity sets $\mathcal{Q}_{\eta_l}$ where all CPDs are identical except for the one on the $l$-th vertex. 

We now turn to the estimation of  \eqref{eq:eta-all},  \eqref{eq:eta_vertex}. 
We note that due to the graphical structure of Bayesian networks their  estimation reduces to focusing on individual model components. Related recent ideas using   subadditivity
for divergences or probability metrics of PGMs, instead of a full chain rule, were explored for statistical learning in \cite{Daskalakis:PGM:2021GANsWC}; such an approach could be also used here in an uncertainty quantification context.
Lastly, we can  simplify the estimation of  \eqref{eq:eta-all} or \eqref{eq:eta_vertex} by using an upper bound,
$
\eta_l \le \sup_{x_{\pi_l}}R(Q_{l|X_{\pi_l}=x_{\pi_l}}\|P_{l|X_{\pi_l}=x_{\pi_l}})\, .
$
Under certain conditions we can also show that using KDE gives rise to  consistent statistical estimator, see \eqref{GBayesian networkstress1}-\eqref{eq:eta_i_Pa} for a Gaussian Bayesian network baseline.
%
Finally, we note that significant literature  on statistical estimators
 for divergences includes  non-parametric 
 estimators  \cite{8007086},  statistical estimators  based on variational representations of divergences  \cite{Nguyen_Full_2010,pmlr-v80-belghazi18a},   density-estimator based methods for estimating divergences in low-dimensions  \cite{10.5555/2969239.2969284},   estimators of divergence based on nearest-neighbor distances \cite{4839047,4035959, NIPS2008_ccb09896} and statistical estimators for R\'enyi Divergences  \cite{birrell2021variational}.

 \noindent{\it B. User-determined stress tests.} 
%
Here we  use $\eta_l \ge 0 $ as a parameter to be tuned by hand to  explore how different  levels of uncertainty will affect the QoI; 
for instance when   we have  very sparse or missing data and  $\eta_l$'s are  set  by a user.  
This is a form of non-parametric sensitivity analysis and is reminiscent in spirit of the stress tests used in finance and actuarial science, e.g.  \cite{BlaLamTangYuan2019} to protect against sudden changes and extreme uncertainty under
various scenarios. 
%
In our Bayesian network context,  individual model misspecification $\eta_l$, $l\in V$ for the model sensitivity indices $I^\pm(f(X_k),P;\mathcal{Q}_{\eta_l})$  can take arbitrary fixed values that correspond to model perturbations associated with local sensitivity analysis (small  $\eta_l$) or global sensitivity analysis (larger $\eta_l$). Both local and global sensitivity analyses are conducted in the same mathematical
framework, therefore we have the flexibility to explore combinations of small/large model perturbations at different vertices of the Bayesian network. 
From a practical point of view, these sensitivity computations can be done using only one fixed constructed Bayesian network (the baseline), yielding guarantees for entire neighborhoods of models. 

 \smallskip
\noindent
{\bf Step 2: Ranking of model sensitivities.}  Once  $\eta_l$'s are  specified in Step 1 for  each vertex $l$,
we calculate the model   sensitivity indices $I^\pm(f(X_k),P;\mathcal{Q}_{\eta_l})$ using Theorem~\ref{thm:MFSI general} and \ref{thm:MFSI}, where $\mathcal{Q}_{\eta_l}=\mathcal{D}^{\eta_l}_l$ or $\mathcal{D}^{\eta_l}_{l,P}$ are defined in \eqref{eq:set:MFSI general} and \eqref{eq:set:MFSI}. Subsequently  we  rank them according to their relative contributions  
\begin{equation}
\label{eq:ranking}
 \frac{I^+(f(X_k), P; \mathcal{Q}_{\eta_l})}{\sum_j I^+(f(X_k), P; \mathcal{Q}_{\eta_j})}.
\end{equation}
See also the example in   Figure~\ref{fig:Pie}.

\smallskip
\noindent
{\bf Step 3: Assessing the baseline.} After we have ranked the model sensitivities  in Step 2, we focus on the most impactful model components and assess their impact on the QoI $\MEANNN{P}{f(X_k)}$. Specifically, if the relative model uncertainty is less that  an application-dependent  tolerance $TOL$, 
\begin{equation}\label{eq:TOL:1}
    I^+(f(X_k),P;\mathcal{Q}_{\eta_{l}}) \leq TOL \, ,  
\end{equation}
then we decide to ``trust" the model component $l$. If there are model components that do not satisfy \eqref{eq:TOL:1}, we proceed to the next step in order to correct the baseline model $P$. This is a form of interpretability, since we can systematically identify under-performing parts of the model. A related quantity that can  also be used  in  \eqref{eq:TOL:1} is the relative model sensitivity 
\begin{equation}\label{eq:TOL:2}
  \frac{I^+(f(X_k),P;\mathcal{Q}_{\eta_{l}})}{\MEANNN{P}{f(X_k)}}   \, ,  
\end{equation}
see for example
Figure~\ref{fig:bimetallics}.


\smallskip
\noindent
{\bf Step 4: Model correctability.} Once  Steps 2 and 3 are completed,  we turn to  correcting the most influential components  of the  baseline Bayesian network $P$, a task also referred to as correctability in Machine Learning.  We formulate mathematically this procedure in Section~\ref{Sec:GBN:Correctability}, however practically
we aim at reducing the index $I^+(f(X_k),P;\mathcal{D}^{\eta_{l}}_{l})$ for each vertex $l \in V$ that violates \eqref{eq:TOL:1}. This can be  accomplished, for instance, by either acquiring additional data or updating the CPD of these specific vertices.
However, as we correct these targeted model components of the baseline, we also need to guarantee that we do not introduce new, bigger  errors in the remaining components of the Bayesian network that would violate \eqref{eq:TOL:1}. Section~\ref{Sec:GBN:Correctability} provides both theory and related practical implementation strategies to this end.

\section{Mathematical analysis of correctability in Bayesian networks}\label{Sec:GBN:Correctability}
In this section, we focus on the mathematical formulation of correctability in  Bayesian Networks  outlined in Step 4 of Section~\ref{sec: correctability}. Our methods are motivated by  ``correcting" a baseline model by either  acquiring targeted high quality data, or updating the CPDs  of the most under-performing components (see Step 3 of Section~\ref{sec: correctability}), or  correcting the graph $G$ itself. We demonstrate these scenarios,  their combinations and our mathematical methods on a materials screening problem for fuel cells in Section~\ref{sec: ORR}.

The intuition behind our correctability  analysis lies in the model sensitivity results for the Gaussian case. By Theorem~\ref{cor:Gaussian Bayesian network MFSI}, the model sensitivity indices of a baseline $P$ for a targeted  $l^*$-th CPD component are given by
\begin{equation}\label{eq:correct:GBN1}
    I^{\pm}(f(X_k), P; \mathcal{D}_{l^*,P}^{\eta_{l^*}}) =
      \pm |\tilde{\beta}_{kl^*}|\sqrt{2\sigma_{l^*}^2\eta_{l^*}} \, .
\end{equation}
Therefore,  additional/better data  or an improved CPD for the $l^*$-th vertex  could   allow us to build a new $l^*$-th CPD with  a corresponding new Gaussian Bayesian model $\tilde{P}$ that is otherwise identical to  $P$. Indeed, if we could guarantee a combination of 
\[
 \tilde{\sigma}_{l^*}^2< \sigma_{l^*}^2\;\;\textrm{and/or}\;\;\tilde{\eta}_{l^*}<\eta_{l^*} 
 \]
 for the new model $\tilde P$ then we can quantify the improvement of the baseline   $P$  using \eqref{eq:correct:GBN1} and show that the indices of $\tilde P$ at $\l^*$  would decrease.
 
In general, we seek to correct the targeted $l^*$-th vertex of the baseline $P$ to obtain a new Bayesian network $\tilde P$ 
such that 
\begin{equation}\label{eq:correct:noworse}
    I^{\pm}(f(X_k), \tilde{P}; \mathcal{D}^{\eta_l}_l)\le
    I^{\pm}(f(X_k), P; \mathcal{D}^{\eta_l}_l) ,\qquad \textrm{all }\; l\neq l^*
\end{equation}
and 
\begin{equation}\label{eq:correct:better}
    I^{\pm}(f(X_k), \tilde{P}; \mathcal{D}_{l^*,P}^{\eta_{l^*}}) <
      I^{\pm}(f(X_k), P; \mathcal{D}_{l^*,P}^{\eta_{l^*}})\, .
\end{equation}
In particular, \eqref{eq:correct:noworse} and \eqref{eq:correct:better} would imply that we can improve  the  CPD of the  $l^*$ vertex and  at the same time we do not decrease the performance of the rest of the Bayesian network.
 The next theorem demonstrates that we can achieve \eqref{eq:correct:noworse} when $P$ is a Gaussian Bayesian network satisfying \eqref{eq:GBayesian network dist}. Moreover, when $P$ is a general  Bayesian network, we prove that new errors that may violate \eqref{eq:correct:noworse} can only be created in the descendant components of $l^*$, see also Remark~\ref{rem:correct:pract}.

  \begin{theorem} 
\label{THM:GBNsCor}
(a) {\bf (Gaussian Bayesian Network)} Consider  $f(X_k) = a X_k+b$ to be a QoI that only depends on $X_k$ linearly. Let also $P$ be a Gaussian Bayesian network satisfies \eqref{eq:GBayesian network dist}. Suppose now that we construct a new Bayesian Network $\tilde{P}$ by only updating the CPD $p(x_{l^*}|x_{\pi_{l^*}})$ for some $l^*\in\rho_k$ as follows: we change the distribution of $\epsilon_{l^*}$ in \eqref{eq:GBayesian network dist} from Gaussian to another mean zero distribution denoted by $\tilde{p}(x_{l^*}|x_{\pi_{l^*}})$. Note that the graph structure of $\tilde{P}$ is the same as $P$.  Then, 
\begin{equation}
    I^{\pm}(f(X_k), \tilde{P}; \mathcal{D}^{\eta_l}_l)=
    I^{\pm}(f(X_k), P; \mathcal{D}^{\eta_l}_l) ,\qquad \textrm{for all }\; l\neq l^*
\end{equation}
where  $I^{\pm}(f(X_k), P; \mathcal{D}^{\eta_l}_l) $ is given by \eqref{GBayesian networkopt1}-\eqref{GBayesian networkopt2}. Moreover, for the relative model sensitivity \eqref{eq:TOL:2} the following holds:
\begin{equation}\label{eq:relative}
    \frac{I^{\pm}(f(X_k), \tilde{P}; \mathcal{D}^{\eta_l}_l)}{\MEANNN{\tilde{P}}{f}}=
    \frac{I^{\pm}(f(X_k), P; \mathcal{D}^{\eta_l}_l)}{{\MEANNN{P}{f}}} ,\qquad \textrm{for all }\; l\neq l^*
\end{equation}

\smallskip

\noindent
(b) {\bf (Non Gaussian Bayesian Network)} Let  $f(X_k)$ be a QoI that only depends on $X_k$. Let also $P$ be a non Gaussian Bayesian network. Let us suppose that we construct a new Bayesian Network $\tilde{P}$ with the same structure as $P$ by only updating the CPD $p(x_{l^*}|x_{\pi_{l^*}})$ for some $l^*\in\rho_k^P$.  Then,
\begin{equation}
    I^{\pm}(f(X_k), \tilde{P}; \mathcal{D}^{\eta_l}_{l,P})=
    I^{\pm}(f(X_k), P; \mathcal{D}^{\eta_l}_{l,P}),\qquad \textrm{all }\; l\in\rho_k^P \textrm{ with } l<l^*
\end{equation}
while the model sensitivity indices for any $l\in\rho_{k}^P$ with $l\geq l^*$ (descendant components) change and are given by Theorem~\ref{thm:MFSI}.
  \end{theorem}
 \begin{proof}
(a) First, updating  $p(x_{l^*}|x_{\pi_{l^*}})$ with $l^*\in\rho_{k}$ does not affect the computation of $F$ defined in \eqref{eq:F}. This is straightforward by \eqref{eq:F}. In the case of a Gaussian Bayesian network, $F$ is given by \eqref{FforGBNs}. Second, if $l\neq l^*$, by \eqref{FforGBNs}, the MGF  of $\bar{F}$ with respect to $\tilde{P}_{l|\pi_{l}}$ is always the same with the MGF computed with respect to  $P_{l|\pi_{l}}$ (since $\tilde{P}_{l|\pi_{l}}=P_{l|\pi_{l}}$), see \eqref{FforGBNs1}. However, it only changes when  $l=l^*$. Moreover, the relative model sensitivity with respect to model $\tilde{P}$ satisfies \eqref{eq:relative}, since $\tilde{P}_{l^*|\pi_{l^*}}: X_{l^*}=\beta_{l^*0}+\beta_{l^*}^TX_{\pi_{l^*}}+\tilde{\epsilon}_{l^*}$ with $\tilde{\epsilon}_{l^*}$ another mean zero distribution. Thus the  expected values of $f(X_k)$ with respect to $P$ and $\tilde{P}$ are equal.
\smallskip

\noindent (b) It is enough to observe that for any $l\in\rho_{k}^P$ with $l< l^*$, the MGF in  Theorem~\ref{thm:MFSI general} computed with the respect to $\tilde{P}_{l|\pi_l}^P=P_{l|\pi_l}^P$ and $P_{l|\pi_l}^P$ are equal   and both  depend on the ancestors of $\rho_{l}^P$, where $l^*\notin\rho_{l}^P$. Hence, \eqref{eq:MFSI general} for both models is the same. Similarly, we prove the case $l\in\rho_{k}^P$ with $l\geq l^*$ (descendant components). Note that this time  $l^*\in\rho_{l}^P$ and thus \eqref{eq:MFSI general} is different for the two Bayesian Networks.
 \end{proof}
 \smallskip
 
 Both developed approaches  are implemented in Section~\ref{subsec:ORR:correctability}. For example, Theorem~\ref{THM:GBNsCor} (a) is applied when we update a CPD of the baseline Gaussian Bayesian Network by using a kernel-based (KDE) method,  
 see Figure~\ref{fig:bimetallics}. We refer to Section~\ref{subsec:ORR:correctability} for full details.

  \begin{remark}\label{rem:correct:pract}
   Even if the conditions of Theorem~\ref{THM:GBNsCor} are not applicable, the ranking procedure of Step 2 \&  Step 3 in Section~\ref{sec: correctability} can always identify the best candidates   among the components of the graphical model for improvement relative to a QoI. Once we correct the component $l^*$ selected through ranking we need to  recompute the relative model uncertainties in \eqref{eq:TOL:1} for all vertices $l \in V$ and then determine the suitability of the corrected model.  In fact, due to Theorem~\ref{THM:GBNsCor} (b) we only need to compute \eqref{eq:TOL:1} for just  the vertices $l$ in the descendants of $l^*$ since  all the remaining ones are not affected by the model correction.
  \end{remark}

\section{DFT-Informed Langmuir Model}\label{sec:Langmuir}






In this section, we consider  the Langmuir bimolecular adsorption model
that describes the chemical kinetics with competitive dissociative adsorption of hydrogen and oxygen on a catalyst surface \cite{catalysisGadi,doi:10.1063/1.5021351}. It is a multi-scale  system of random differential equations  with correlated dependencies in their parameters (kinetic coefficients),  arising from quantum-scale computational data calculated  using Density Functional Theory (DFT) (i.e quantum computations) for actual metals. The combination of chemical kinetics with parameter dependencies, correlations and DFT data gives rise naturally to a Bayesian network. However, the limited availability of the quantum-scale data creates significant model uncertainties both in the  distributions of kinetic coefficients  and their correlations, see for example Figure~\ref{fig:langmuir data} (a).
%
%
Thus, we will quantify the ensuing model uncertainties  by implementing our analysis in Section~\ref{sec:MFUQ} and ~\ref{sec:QUSI}. Here, the equilibrium hydrogen and oxygen coverages  are our QoIs and can be calculated by the dynamics of the chemical reaction network described by  the following system of random ODEs with random (correlated) coefficients, 
\begin{align}
\label{eq:Langmuir ODEs}
\frac{dC_{H^*}}{dt}&=k^{ads}_{H_2}P_{H_2}(1-C_{H^*}-C_{O^*})^2-k^{des}_{H_2} C^2_{H^*},\qquad C^0_{H^*}=C_{H^*}(0),\\
\frac{dC_{O^*}}{dt}&=k^{ads}_{O_2}P_{O_2}(1- C_{H^*}- C_{O^*})^2-k^{des}_{O_2} C^2_{O^*}, \qquad C^0_{O^*}= C_{O^*}(0)\label{eq:Langmuir ODEs2},
\end{align}
where $C_{H^*}$ and $C_{O^*}$ represent the  hydrogen and oxygen coverages. $P_{H_2}$ and $P_{O_2}$ are the partial pressures of the gas phase species and are fixed.
\begin{figure}[ht]
\centering
\includegraphics[width=6in]{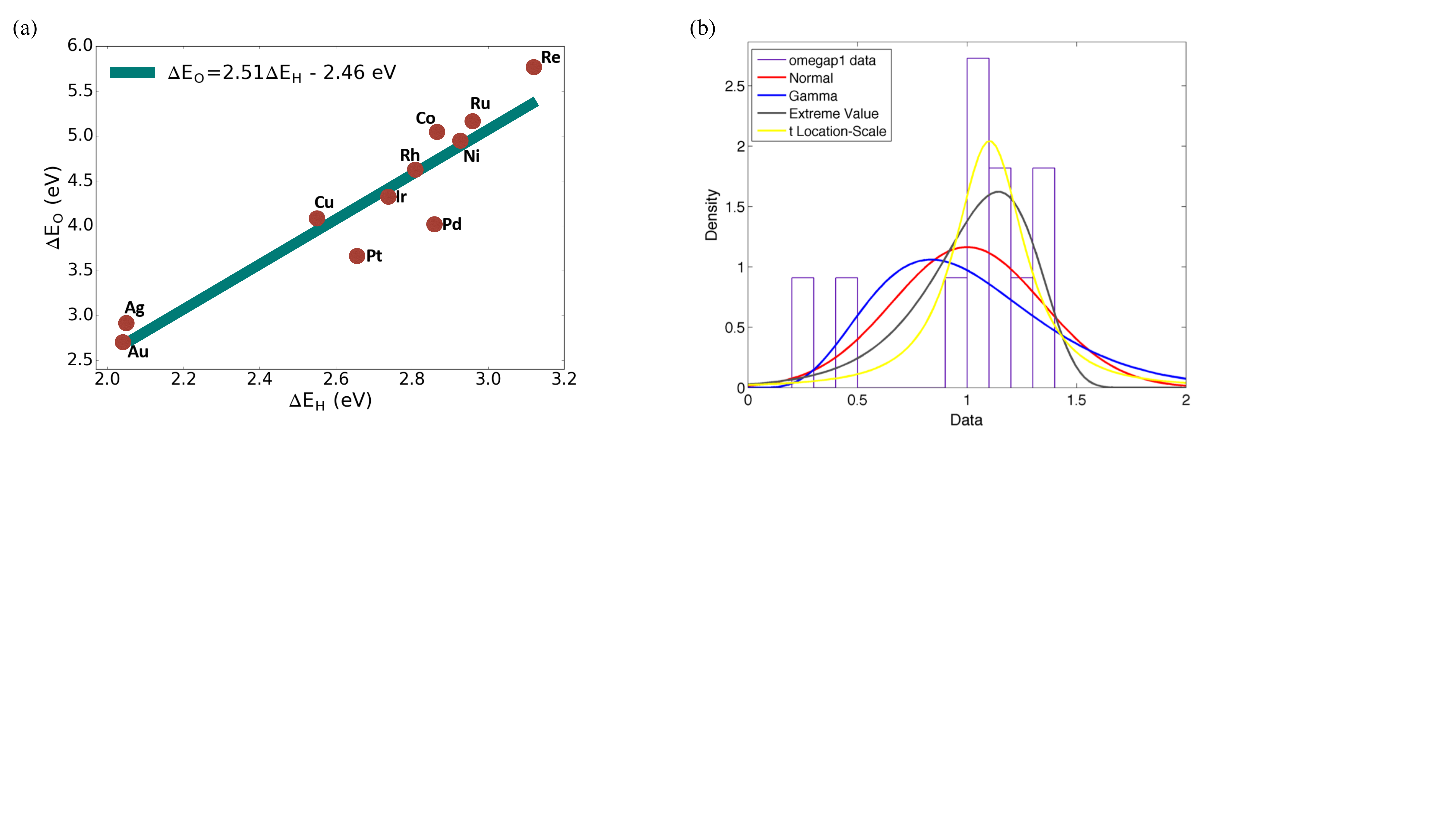}
\vspace{-4.5cm}
\caption{\small (a) Correlation between oxygen and hydrogen adsorption energies on  metal surfaces as defined in  (\ref{linearEq}), (b) Fit of $\omega$ in \eqref{linearEq}
with various parametric distributions.}
\label{fig:langmuir data}
\end{figure}

\noindent Then, the steady state solution of \eqref{eq:Langmuir ODEs}-\eqref{eq:Langmuir ODEs2}, which constitute our  QoIs,  is given by 
\begin{eqnarray}
\label{eq:Langmuir equilibrium}
\;\;\;\;\;\;\;\;\;\displaystyle{\hat C_{H^*}=\frac{(K_{H_2}P_{H_2})^{\frac{1}{2}}}{1+(K_{H_2}P_{H_2})^{\frac{1}{2}}+(K_{O_2}P_{O_2})^{\frac{1}{2}}},\quad \hat C_{O^*}=\frac{(K_{O_2}P_{O_2})^{\frac{1}{2}}}{1+(K_{H_2}P_{H_2})^{\frac{1}{2}}+(K_{O_2}P_{O_2})^{\frac{1}{2}}}}.
\end{eqnarray}
Here    $K_i=\frac{k_i^{ads}}{k_i^{des}}$ for $i=H_2, O_2$ and  for each species they are related to electronic structure (DFT) calculations through   an  Arrhenius law \cite{feng2018non}:
\begin{align}
K_{H_2}&=e^{-\frac{G_{H_2}}{k_BT}}(P_{H_2}+P_{O_2})^{-1},\;\;G_{H_2}\propto -2\Delta E_{H}
\label{eq:Arrhenius1}
\\
K_{O_2}&=e^{-\frac{G_{O_2}}{k_BT}}(P_{H_2}+P_{O_2})^{-1},\;\;G_{O_2}\propto -2\Delta E_{O}
\label{eq:Arrhenius2}
\end{align}
The constants  $k_B$ and $T$ are  the Boltzmann constant and the temperature respectively. In the above formulas, $G_{H_2}$ and $G_{O_2}$ are  the hydrogen and oxygen Gibbs free energies of adsorption.  Therefore, the coverages $\hat C_{H^*}$ and $\hat C_{O^*}$ are  non-linear functions of $\Delta E_{H}$ and $\Delta E_{O}$. We refer to \cite{doi:10.1063/1.5021351} for  the  chemistry background and analysis of the model. In \cite{doi:10.1063/1.5021351}, the authors have estimated  the two binding energies for various  metal catalyst surfaces via DFT calculations as illustrated in Figure~\ref{fig:langmuir data}(a). Furthermore, correlations between 
$\Delta E_{O}$ and $\Delta E_{H}$ are captured by a statistical linear model 
\begin{equation}\label{linearEq}
\Delta E_{O}=a\Delta E_{H}+b+\omega
\end{equation}
where $\omega$ is a random variable. The distribution of $\omega$ can be determined by fitting the residual data from linear regression using Maximum Likelihood Estimation (MLE),  Figure~\ref{fig:langmuir data}(b). 
\begin{figure}[ht]
\centering
\includegraphics[width=6in]{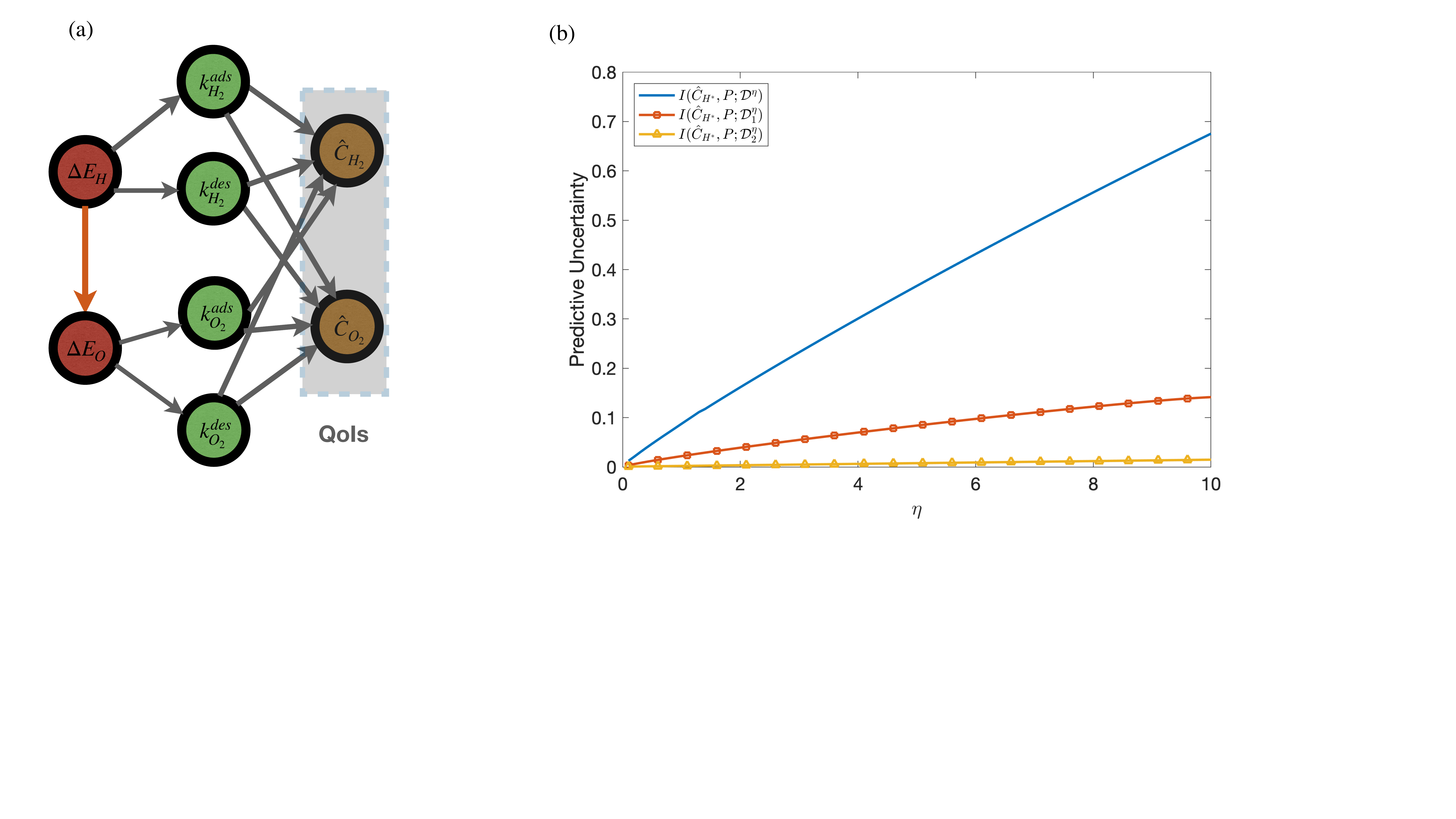}
\vspace{-3.5cm}
\caption{(a) Graph structure of the baseline Bayesian network $P$  in \eqref{eq:Langmuir prob model} built by  data \eqref{eq:langmuir PGM_p2} and  \eqref{eq:langmuir PGM_p1}, 
physics knowledge \eqref{eq:Arrhenius1}, \eqref{eq:Arrhenius2},  \eqref{eq:langmuir PGM_p1} and the steady state of the ODEs given by \eqref{eq:Langmuir equilibrium}. 
(b) We consider the QoIs \eqref{eq:Langmuir equilibrium} of  \eqref{eq:Langmuir prob model}: the blue line represents the model uncertainty index $I^{\pm}(f, P; \mathcal{D}^\eta)$ as a function of $\eta$ (Theorem~\ref{thm:MFUQ PGM});
the red and yellow lines are respectively the model sensitivity indices $I^{\pm}(f, P; \mathcal{D}^\eta_1)$ and $I^{\pm}(f, P; \mathcal{D}^\eta_2)$ (Theorem~\ref{thm:MFSI general}) for $p(\Delta E_H)$ and $p(\Delta E_O|\Delta E_H)$ where $\mathcal{D}^\eta_1$ indicates the perturbation on $p(\Delta E_H)$ and $\mathcal{D}^\eta_2$ for $p(\Delta E_O|\Delta E_H)$. 
}
\label{fig:langmuir PGM}
\end{figure}
\noindent
In \eqref{linearEq} we select a Gaussian distribution for $\omega$ (red line in Figure \ref{fig:langmuir data} (b)) as the baseline  CPD for the correlation in Figure~\ref{fig:langmuir data}:
\begin{equation}
\label{eq:langmuir PGM_p2}
    p(\Delta E_O| \Delta E_H) = \mathcal{N}(a\,\Delta E_H+b, \sigma_{\omega}^2).
\end{equation}
Next we model the distribution of the prior $p(\Delta E_H)$.   Based on physical constraints (e.g. positivity of the random variable without physical upper bound),  in \cite{doi:10.1063/1.5021351} the distribution  of $\Delta E_H$ was selected to be a  gamma distribution with mean $x_H$ with  standard deviation given by the difference between experiment and DFT, $(x_H-y_H)$, 
\begin{equation}
\label{eq:langmuir PGM_p1}
 p(\Delta E_H)=\frac{1}{b_H^{a_H}\Gamma(a_H)} \Delta E_H ^{\ a_H-1} \exp \left(-\frac{\Delta E_H}{b_H}\right) \quad \mbox{for}\quad \Delta E_H > 0,
\end{equation}
where $a_H=x_{H}^2/(x_{H}-y_{H})^2$ and $b_H=(x_{H}-y_{H})^2/x_{H}$. This is a case with  very little data $(x_{H}, y_{H})$ and only some reasonable physical constraints without any further knowledge on the model, therefore model uncertainty in \eqref{eq:langmuir PGM_p1} is evident.


We now  build the baseline Bayesian network $P$  by combining the following ingredients: 
data through \eqref{eq:langmuir PGM_p2} and  \eqref{eq:langmuir PGM_p1}, 
physics and expert knowledge in \eqref{eq:Arrhenius1}, \eqref{eq:Arrhenius2},  \eqref{eq:langmuir PGM_p1} and the steady state of the ODEs (QoI) given by \eqref{eq:Langmuir equilibrium}, see also Figure~\ref{fig:langmuir PGM}. We obtain the following Bayesian network and the corresponding CPDs:
\begin{equation}
\label{eq:Langmuir prob model}
    p(x) = 
    \underbrace{p(\hat{C}_{H^\ast},\hat{C}_{O^\ast}|K_{H_2}, K_{O_2})}_{\eqref{eq:Langmuir equilibrium}}
    \prod_{i=H_2,O_2}\!\!
    \underbrace{p(K_{i}|\Delta E_i)}_{\eqref{eq:Arrhenius1}, \eqref{eq:Arrhenius2}}
    \underbrace{p(\Delta E_O| \Delta E_H)}_{
    \eqref{eq:langmuir PGM_p2}
    }\underbrace{p(\Delta E_H)}_{
    \eqref{eq:langmuir PGM_p1}
    }
\end{equation}
 In the above formula, $p(\hat{C}_{H^\ast},\hat{C}_{O^\ast}|K_{H_2}, K_{O_2})$ and $p(K_{i}|\Delta E_i)$ are deterministic, while the only random parts in $P$ are   $p(\Delta E_O|\Delta E_H)$ and  $p(\Delta E_H)$.

%
In the process  of building the baseline model $P$  above, the sparse   data  in Figure~\ref{fig:langmuir data} for \eqref{eq:langmuir PGM_p2} and the lack  of both  knowledge and  (almost any) data in \eqref{eq:langmuir PGM_p1}
create model uncertainties for the prediction of the QoIs in \eqref{eq:Langmuir equilibrium}. We quantify these uncertainties by implementing the model uncertainty index of Theorem~\ref{thm:MFUQ PGM} and the model sensitivity indices of Theorem~\ref{thm:MFSI general}; see Figure~\ref{fig:langmuir PGM} (b) where we readily see how the indices change  for different values $\eta$; the implementation of the indices was carried out through Monte Carlo simulation of the moment generating functions. 
Moreover,  we observe that for the QoIs \eqref{eq:Langmuir equilibrium} the impact of  uncertainties in the prior  $p(\Delta E_H)$ are significantly higher than  in the correlation $p(\Delta E_O|\Delta E_H)$  when we  perturb with same model misspecification $\eta$.
Finally, we note that due to the lack of data in \eqref{eq:langmuir PGM_p1}, we elected to perform the user-determined stress tests of Step 1.B of Section~\ref{sec: correctability} where the user selects various levels of model misspecification $\eta$.  

\section{Model Uncertainty for Sabatier's Principle}
\label{sec: ORR}

We study     Bayesian networks  built for trustworthy  prediction of materials screening  to increase the efficiency of  chemical reactions in catalysis. Our starting point  is Sabatier's principle which describes the efficiency of a catalyst \cite{catalysisGadi} through the so-called ``volcano curve", e.g. the  black curve  in Figure~\ref{fig: PGM}(c). The volcano curve  suggests that high catalytic activity is exhibited when the binding interaction between reactants and catalysts is neither  too strong nor too weak, i.e. at the peak of the volcano marked by a  star in Figure~\ref{fig: PGM}(c). For this reason Sabatier's principle is widely viewed as an important criterion for screening materials for increased efficiency in catalysis. Our ultimate goal here is to understand how various uncertainties can affect the shape and position of the volcano curve and its peak.


Here we consider the  Oxygen Reduction Reaction (ORR) which is  a known performance bottleneck in fuel cells \cite{setzler2016activity}.
The ORR reaction depends on the formation of surface hydroperoxyl ($OOH^\ast$) from molecular oxygen ($O_2$), and water ($H_2O$) from surface hydroxide ($OH^\ast$) \cite{suen2017}. The complete mechanism \cite{callevallejo2015,antoine2001,holewinski2012} involves four electron exchange steps with reactions (R1) and (R4) being slow, see Figure~\ref{fig: PGM}(a).  Therefore, the discovery of new materials will have to rely  on speeding up the two slowest reactions in order to accelerate the entire ORR mechanism. 
Furthermore,  such a physicochemical system has hidden correlations between variables which have emerged after statistical analysis of data \cite{Feng2020Science}. In particular, the corresponding Gibbs energies of reactions (1) and (4) denoted by $-\Delta G_4\equiv y_1$ and $-\Delta G_1\equiv y_2$ are computed as linear combinations of free energies of species and are regressed versus the oxygen binding energy $\Delta G_{O}\equiv x$ calculated by DFT calculations. The oxygen binding energy $x$ is chosen as a descriptor in  \cite{Feng2020Science} since it is the natural coordinate arising from Sabatier's principle.  
The principle is graphically represented by the volcano curve , i.e. the  solids black lines in Figure~\ref{fig: PGM}(c) which is a function of the descriptor. 
Therefore, the QoI considered here is the optimal oxygen binding energy $\Delta G_{O}$ denoted by $x_{O^*}^P$ and identified as the maximum of the volcano curve:
\begin{equation}
\label{QoI1}
x_{O^*}^P:=\mathrm{argmax}_{x_0} \left[\min\{\MEANNN{P}{y_1|x_0},\MEANNN{P}{y_2|x_0}\}\right]\, .
\end{equation}
Starting from this QoI we build a Bayesian network in Figure~\ref{fig: PGM}(b)that includes  expert knowledge (volcano curves), as well as  various available experimental and computational data and their  correlations or  conditional independence.
\begin{figure}[ht]
\centering
\includegraphics[width=1\linewidth]{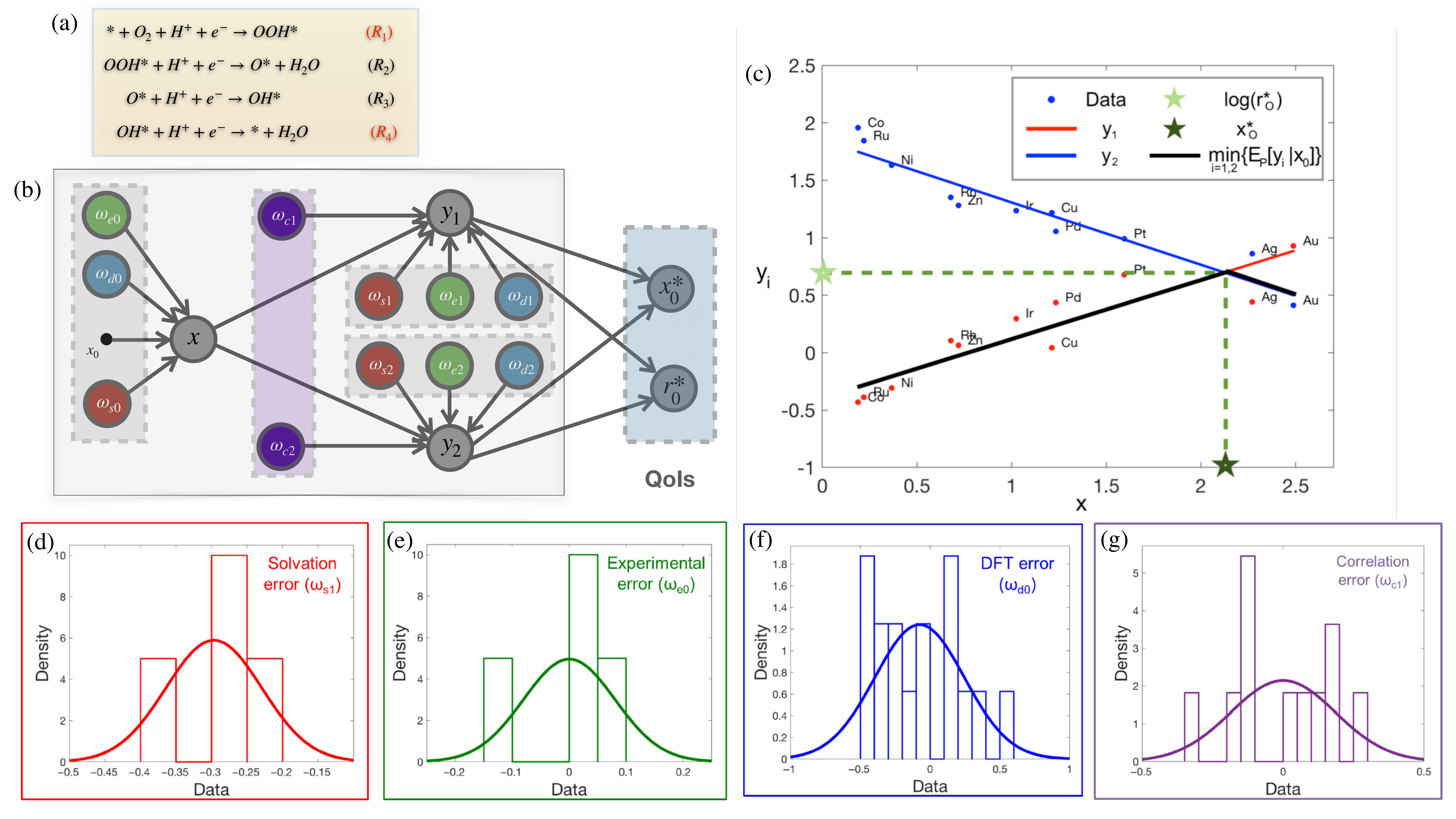}
\vspace{-0.6cm}
\caption{\small (a) ORR reaction steps (R1 to R4) in hydrogen
fuel cells, (b) Bayesian network for ORR. The construction of the Bayesian network (Section~\ref{subsec:construction:ORR}) is based on  expert knowledge, physicochemical  modeling and statistical analysis of data. We include these random variables into the Bayesian network and build the directional relationships (connection/arrows)
between corresponding random variable $x$ or $y_i$. We build a Gaussian Bayesian network, i.e., all CPDs are Gaussians which are  fitted to available data using MLE (see histogram approximations in (d-g)). Note the conditional independence between the $y$-variables, assumed based on expert knowledge.  (c) The QoI of the ORR model is the optimal oxygen binding energy $x_{O^\ast}^P$ and is identified when the two reaction energies are equal by physical modeling (marked with a star).  (d-f) Here we model different kinds of errors in $x$ and $y_i$, given expert knowledge.} 
\label{fig: PGM}
\end{figure}


\smallskip
 
\subsection{Construction of the ORR Bayesian network for the QoI \eqref{QoI1}} \label{subsec:construction:ORR}
First, we  relate the QoI  with the $y_i$'s  and then we include errors from different sources in $x$ and $y_i$'s.
More precisely,
\smallskip

\noindent(1) [Graph] We first build the directed graph for the Bayesian network. The first selected vertices in the graph are the QoIs  $x_{O^\ast}^P$,  $r_{O^\ast}^P$, as well as $y_i$'s and  $x$, see gray vertices in Fig. \ref{fig: PGM} (b). Subsequently,
\begin{itemize}
\item[(1a)] Through the statistical independence test \cite{wasserman2013all}, we  learn that  $y_1$ and $y_2$ depend on $x$ and  are conditionally independent given $x$ as illustrated in Figure~\ref{fig: PGM} (b).
\item[(1b)] The construction of $x$ comes from the DFT data (using quantum calculations) for the  oxygen binding energy given the real unknown value $x_0$. As mentioned in the beginning of the section, $x$ is also selected to be the descriptor by expert knowledge (see also the supplementary material of  \cite{Feng2020Science}) and justifies the conditional  relationships between $x$ and $y_i$'s.
 \item[(1c)] The evaluations of the QoIs depend on the values of $y_i$'s for each $x_0$ due to the volcano curve of the Sabatier's principle.
     
\end{itemize}
Overall, in (1) we built part of the network structure for $x$, $y_1$, $y_2$ and the QoI  using a constraint-based method \cite{spirtes2000causation}, which selects a desired structure based on  constraints of dependency among variables.
\smallskip

\noindent (2)[CPD] Next, we build the individual CPDs on the graph constructed above. 
\begin{itemize}
    \item[(2a)] We include statistical correlations between  DFT (quantum calculation) data for $x$ and $y_i$, see data in \cite{Feng2020Science}. We model the residual using a  linear model with  a random correlation error  denoted by  $\omega_{ci}$, see \eqref{eqyi}.
    \item[(2b)] We model as random variables and incorporate in the Bayesian network different kinds of errors in $x$ and $y_i$'s 
from the following  sources:   $\omega_{ei}$ is the error in experimental data, $\omega_{di}$ is error between quantum  and experimental values and $\omega_{si}$ is error due to solvation effects; all  are calculated by DFT, see the corresponding data in \cite{Feng2020Science}. See \eqref{eqyi}.
\end{itemize}
\smallskip
\noindent
More specifically, after conducting independence tests on the corresponding data, and also based on expert knowledge or intuition \cite{Feng2020Science} we assume that the random variables $\omega$ are independent. Based on the graph construction above we obtain the Bayesian network
\begin{eqnarray}
\label{eq:ORR model}
 p(\mathbf{x}|x_0)= \prod_{i=1,2}p(y_i|x,\omega_{ei},\omega_{di},\omega_{si},\omega_{ci}) \cdot p(x|\omega_{e0},\omega_{d0},\omega_{s0},x_0) \cdot  \!\!\!\!\!\!\!\!\!\prod_{\substack{j=e_k,d_k,s_k,c_1,c_2\\k=0,1,2}} \!\!\!\!\!\!\!\!\!\!p(\omega_j)
\end{eqnarray}
where $\mathbf{x}=(x,y_1,y_2,\omega_{e0},\omega_{d0},\omega_{s0},\omega_{e1},\omega_{d1},\omega_{s1},\omega_{c1},\omega_{e2},\omega_{d2},\omega_{s2},\omega_{c2})$.
The baseline CPDs in \eqref{eq:ORR model} are constructed as  linear Gaussian models, namely for $i=1,2$:
\begin{equation}
\label{eqyi}
    y_i = \beta_{y_i,0} + \beta_{y_i,x}x + \omega_{ei} + \omega_{di} + \omega_{si} + \omega_{ci} \quad\textrm{and}\quad  x = x_0 + \omega_{e0} + \omega_{d0} + \omega_{s0}\, .
\end{equation}
The CPDs  for each vertex  are selected as 
\begin{align}\label{eq: omega dist}
    p(y_i|x,\omega_{ei},\omega_{di},\omega_{si},\omega_{ci}) &= \mathcal{N}(\beta_{y_i,0} + \beta_{y_i,x}x + \omega_{ei} + \omega_{di} + \omega_{si} + \omega_{ci}, 0)\\
        p(x|\omega_{e0},\omega_{d0},\omega_{s0},x_0) &= \mathcal{N}(x_0 + \omega_{e0} + \omega_{d0} + \omega_{s0}, 0)\\ \label{eqyi'}
     p(\omega_j) &= \mathcal{N}(\beta_{j,0},\sigma_{j}^2)
\end{align}
where $i=1,2$, and $j=e0,d0,s0,e1,d1,s1,c1,e2,d2,s2,c2$. Then the resulting baseline model
\eqref{eq:ORR model} is a Gaussian Bayesian network.
Subsequently we use the global likelihood decomposition method \cite{koller2009probabilistic} 
to learn the parameters $\beta_{y_i,0},\beta_{y_i,x}$ and $\sigma_{j}$. The outcomes are given in Table~\ref{tab:MLE beta}. This approach is essentially a Maximum Likelihood Estimation (MLE) on PGMs (see \cite[Chapter 17.2]{koller2009probabilistic}), that  exploits a fundamental scalability property that allows us to ``divide and conquer" the parameter inference problem on the graph.  We can also employ a Bayesian approach instead of MLE, see for instance \cite{koller2009probabilistic}  for the case of PGMs. 
\medskip

\subsection{Model sensitivity, stress tests and ranking} 
Here,  we implement the four-step strategy of Section~\ref{sec: correctability} to the ORR model by using data-informed stress tests or user-determined stress tests (Step 1.A and Step 1.B of Section~\ref{sec: correctability}). The primary goal  is to quantify and rank the impact of model uncertainties
from each component of the Bayesian network through the model sensitivity
indices in Section~\ref{sec:QUSI}. Next, we compute these model sensitivity indices for the QoI $x_{O^*}^P$ in \eqref{QoI1}, namely
\begin{equation}
\underset{Q \in  \mathcal{D}^{\eta_l}_{l,P}}{\mathrm{sup/inf}}\left \{ x_{O^{*}}^Q - x_{O^{*}}^P  \right\}
\end{equation}
for $l\in\{ei,di,si,ci,e0,d0,s0\}$ with $i=1,2$. To this end, we  first 
use Theorem~\ref{cor:Gaussian Bayesian network MFSI} for $i=1,2$ to obtain
 \begin{equation}
\label{eq:MFSI:ORR}
      I^\pm(y_i,P; \mathcal{D}^{\eta_l}_{l,P}) = \pm |\tilde{\beta}_{y_i,\omega_l}|\sqrt{2\sigma_l^2\eta_l}\ ,
\end{equation}  
with  $\sigma_l$ and  $\tilde{\beta}_{y_i,\omega_l}$  given in \eqref{eq: omega dist} and  Table \ref{tab:tildebeta} respectively. Subsequently we 
solve the optimization problem for $x_O=x_{O^\ast}^P$ and obtain  the bounds for $x_{O^{*}}^Q - x_{O^{*}}^P$  as shown in  Figure~\ref{fig:bounds of QoI} and given by
\begin{equation}
\label{eq:MFSI opt1}
\frac{-\sqrt{2\sigma_l^2\eta_l}}{\beta_{y_1,x}-\beta_{y_2,x}}\leq x_{O^{*}}^Q - x_{O^{*}}^P\leq \frac{\sqrt{2\sigma_l^2\eta_l}}{\beta_{y_1,x}-\beta_{y_2,x}}
\end{equation}
for $l=ei,di,si,ci$ and  $i=1,2$; note that the model uncertainty of $\omega_l$ only  affects $y_i$ according to the ORR Bayesian network. Furthermore,
\begin{equation}
\label{eq:MFSI opt2}
\frac{-( |\beta_{y_1,x}|+|\beta_{y_2,x}|)\sqrt{2\sigma_l^2\eta_l}}{\beta_{y_1,x}-\beta_{y_2,x}}\leq x_{O^{*}}^Q - x_{O^{*}}^P\leq\frac{( |\beta_{y_1,x}|+|\beta_{y_2,x}|)\sqrt{2\sigma_l^2\eta_l}}{\beta_{y_1,x}-\beta_{y_2,x}}
\end{equation}
for $l=e0,d0,s0$ as the model uncertainty of $\omega_{l}$ affects both $y_1$ and $y_2$.  Here $\beta_{y_i,x}$ are the coefficients given by the first CPD in \eqref{eq: omega dist}. The complete algebraic calculation of \eqref{eq:MFSI opt1} and \eqref{eq:MFSI opt2} is given in the Appendix~\ref{sebsec:important calculation for qoi}. 
\begin{figure}[ht]
\centering
\includegraphics[width=1\textwidth]{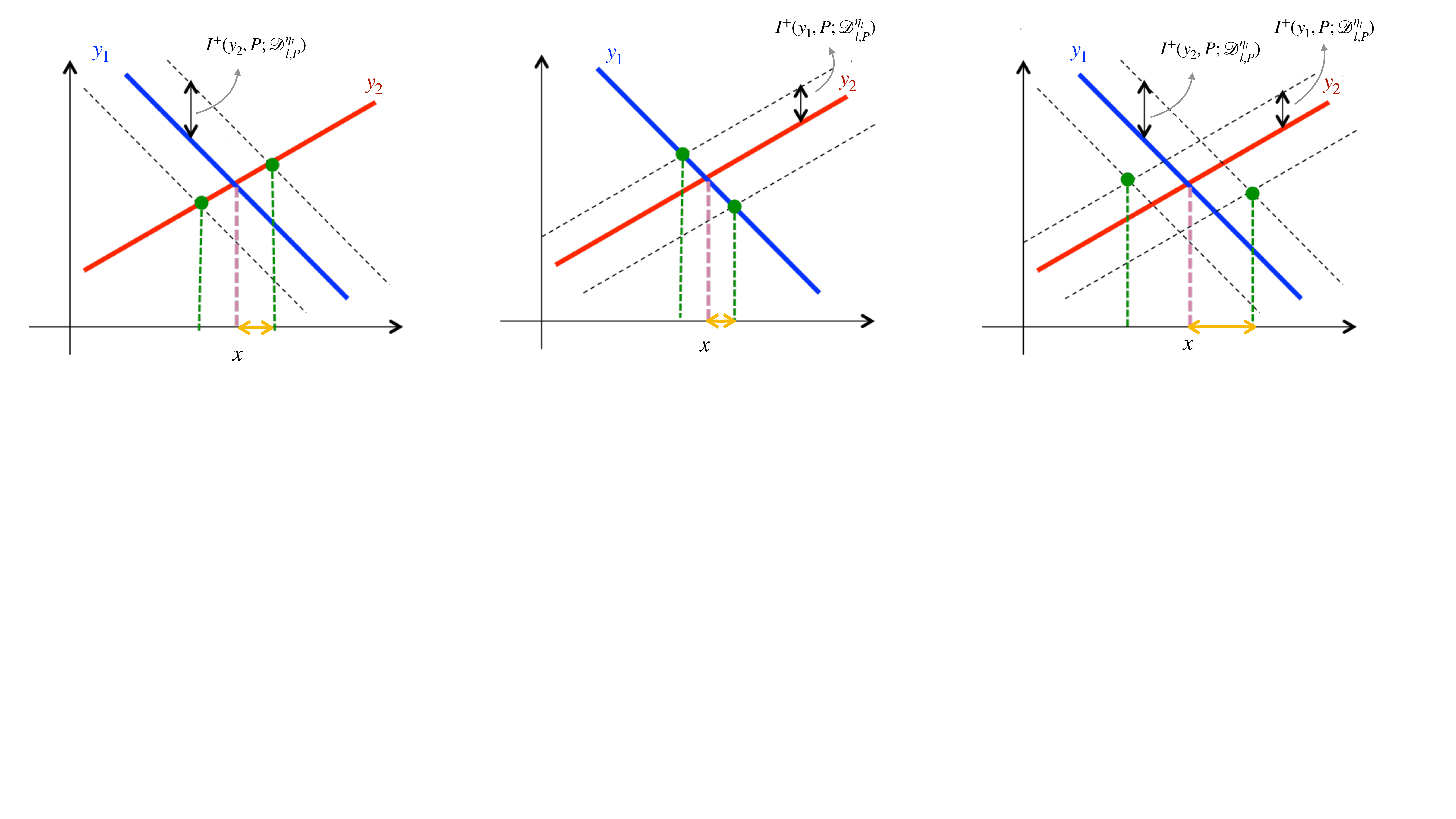}
\vspace{-4.5cm}
\caption{\small Typical model  uncertainty bounds $I^\pm(y_i,P; \mathcal{D}^{\eta_l}_{l,P})\, , i=1, 2$  computed by \eqref{eq:MFSI:ORR}. The model uncertainty  for the QoI $x_{O^\ast}^P$ (see Figure~\ref{fig: PGM}(c)) is  computed by  \eqref{eq:MFSI opt1}-\eqref{eq:MFSI opt2} and demonstrated in yellow for model misspecification $\eta_l$ in $P(\omega_l)$: (a)  for $l=e1,d1,s1,c1$,  $I^\pm(y_1,P; \mathcal{D}^{\eta_l}_{l,P}) =  \pm\sqrt{2\sigma_l^2\eta_l}$; (b) for  $l=e2,d2,s2,c2$,  $I^\pm(y_2,P; \mathcal{D}^{\eta_l}_{l,P}) = \pm\sqrt{2\sigma_l^2\eta_l}$; (c)  for $l=e0,d0,s0$,  $I^\pm(y_i,P; \mathcal{D}^{\eta_l}_{l,P}) =\pm |\beta_{y_i,x}|\sqrt{2\sigma_l^2\eta_l}$, $i=1,2$.
}
\label{fig:bounds of QoI}
\end{figure}
\noindent Then by implementing Step 2 of Section~\ref{sec: correctability}, we rank the model components as demonstrated in  Figure~\ref{fig:Pie}. There we plot \eqref{eq:ranking} as a pie chart, where  the most impactful components are depicted.
\medskip

\begin{figure}[ht]
\centering
\includegraphics[width=0.49\textwidth]{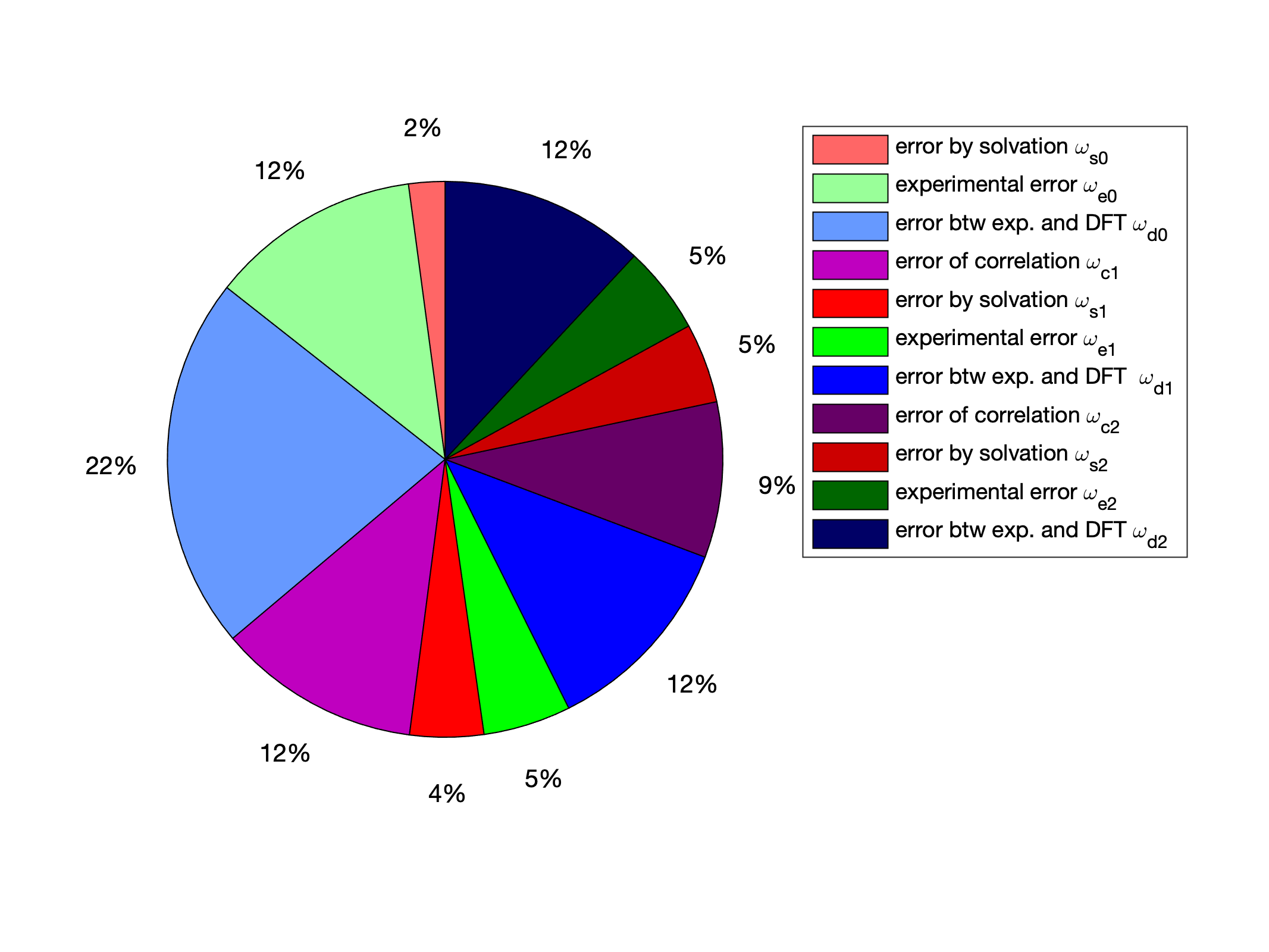}
\includegraphics[width=0.49\textwidth]{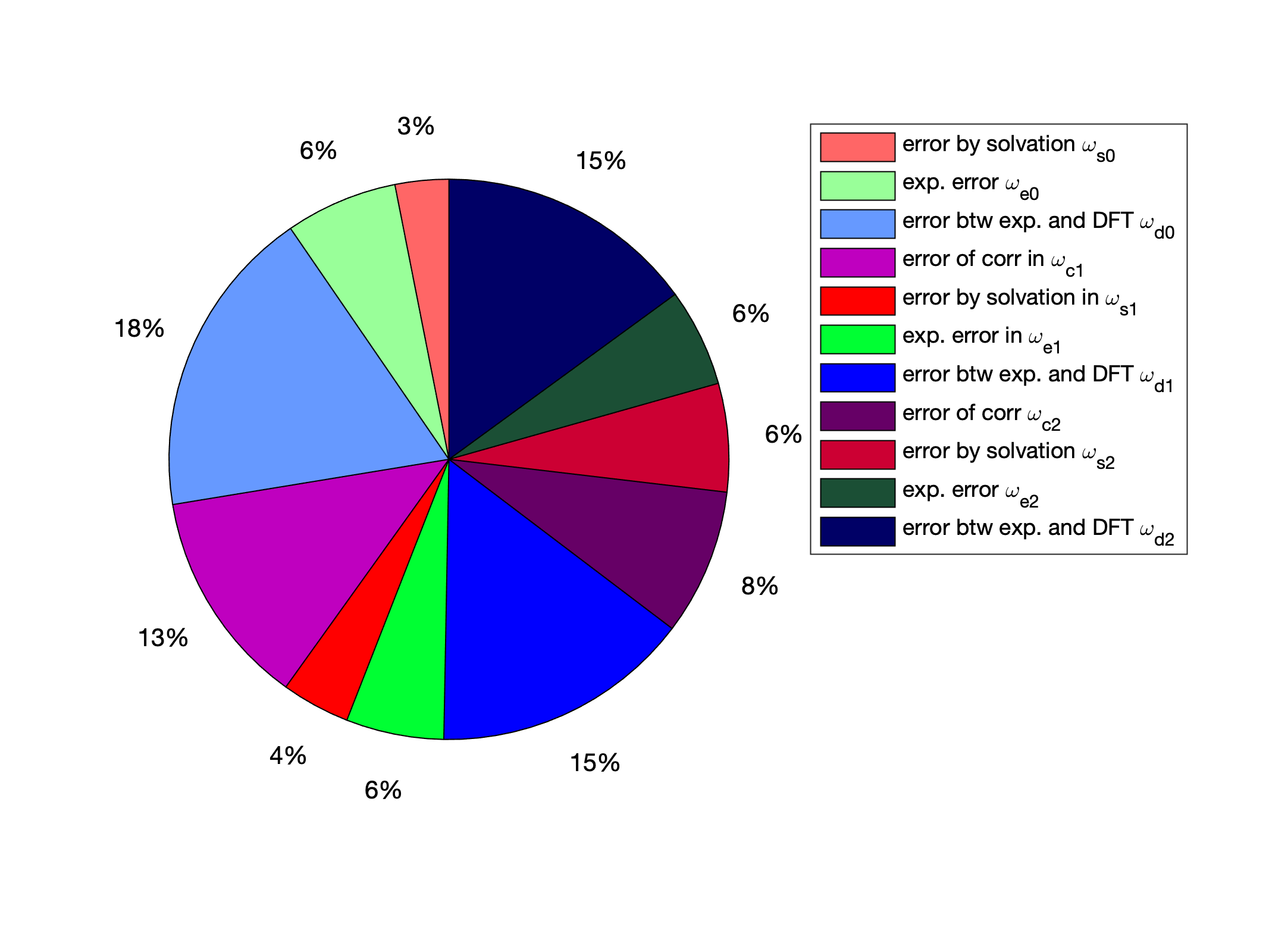}
\vspace{-1cm}
\caption{\small Relative model sensitivities \eqref{eq:ranking} for the QoI $x_{O^\ast}^P$ 
 in each ORR Bayesian network mechanism in Figure~\ref{fig: PGM} (b). {\bf  (Left)}  User-determined stress test (Step 1.B. in Section~\ref{sec: correctability});  $\eta_l$ has a fixed value for all $l$; the particular value does not matter since it is canceled out by the ratio in \eqref{eq:ranking}. {\bf (Right)} 
 Data-informed  stress  test (Step 1.A. in Section~\ref{sec: correctability});
   $\eta_l=R(data\|P_l)$ selected as a distance of each CPD from the available data.
   } 
\label{fig:Pie}
\end{figure}

\begin{remark}[Propagation/Non-Propagation of 
Uncertainties to the QoIs] The discrepancies  in the  propagation of model misspecification to the QoI  between  different Bayesian network components is depicted  in  Figure~\ref{fig:Pie}. In particular, in Figure~\ref{fig:Pie} (Left)  the same user-selected  model misspecification $\eta_l$ is applied on all ORR Bayesian network vertices. However not all propagate and affect the same the QoI. See also the example in 
Figure~\ref{fig:propagation}.
%
\end{remark}

\begin{remark}The construction of the ORR Bayesian network and  its model uncertainty  was carried out in \cite{Feng2020Science} for the optimal oxygen binding energy defined differently than \eqref{QoI1}, that is as  $\mathrm{argmax}_{x_0} \MEANNN{P}{y|x_0}$ with $y|x_0=\min\{y_1|x_0,y_2|x_0\}$. This is an alternative mathematical description of the same concept, however  \eqref{QoI1} allows to explicitly calculate the model sensitivity indices given by \eqref{eq:MFSI:ORR}-\eqref{eq:MFSI opt2} and provide clear insights in what model elements and uncertainties affect them the most.  On the other hand, in  \cite{Feng2020Science} the model sensitivity indices provided by Theorem~\ref{thm:MFSI general} can only be calculated computationally.
 \end{remark}
 \smallskip

\subsection{Correctability of the ORR Bayesian Network}\label{subsec:ORR:correctability}
Here we use the earlier model uncertainty/sensitivity analysis to first identify and  then correct the most impactful components in several  ways as discussed in Step 4 of Section~\ref{sec: correctability} and in the theoretical results on correctability in Section~\ref{Sec:GBN:Correctability}.

\smallskip
\noindent
{\it 1. Including targeted high quality data.} We seek  data that lead to the  reduction of  the variance $\sigma_{l^\ast}^2$ for some $l^\ast\in L$ (see Step 3  of Section~\ref{sec: correctability}), while the model misspecification $\eta_{l^\ast}$ does not increase or the increment is much smaller than the reduction of $\sigma_{l^\ast}^2$. Notice that in this case the model remains a Gaussian Bayesian network. For the ORR Bayesian network, it turns out that  we can add more data using DFT calculations for bimetallics to reduce the relative error for the correlation errors $\omega_{ci}$, $\sigma_{ci}^2$; see the bimetallics data set in \cite{Feng2020Science}.  Then the model sensitivity indices of $y_i$ on $\omega_{ci}$, $I^\pm(y_i,P;\mathcal{D}^{\eta_l}_{l,P}),\,l=\omega_{ci}$ given by \eqref{eq:MFSI:ORR} and  the model misspecification $\eta_{\omega_{ci}}$ are reduced. Consequently,  the model sensitivity indices of $x_{O^\ast}^P$ does so as well, see \eqref{eq:MFSI opt1}. The relative predictive uncertainty \eqref{eq:TOL:2} of such an updated model is demonstrated in Figure~\ref{fig:bimetallics} (Center), 
updated model 2.


\begin{figure}[ht]
\centering
\includegraphics[width=1\linewidth]{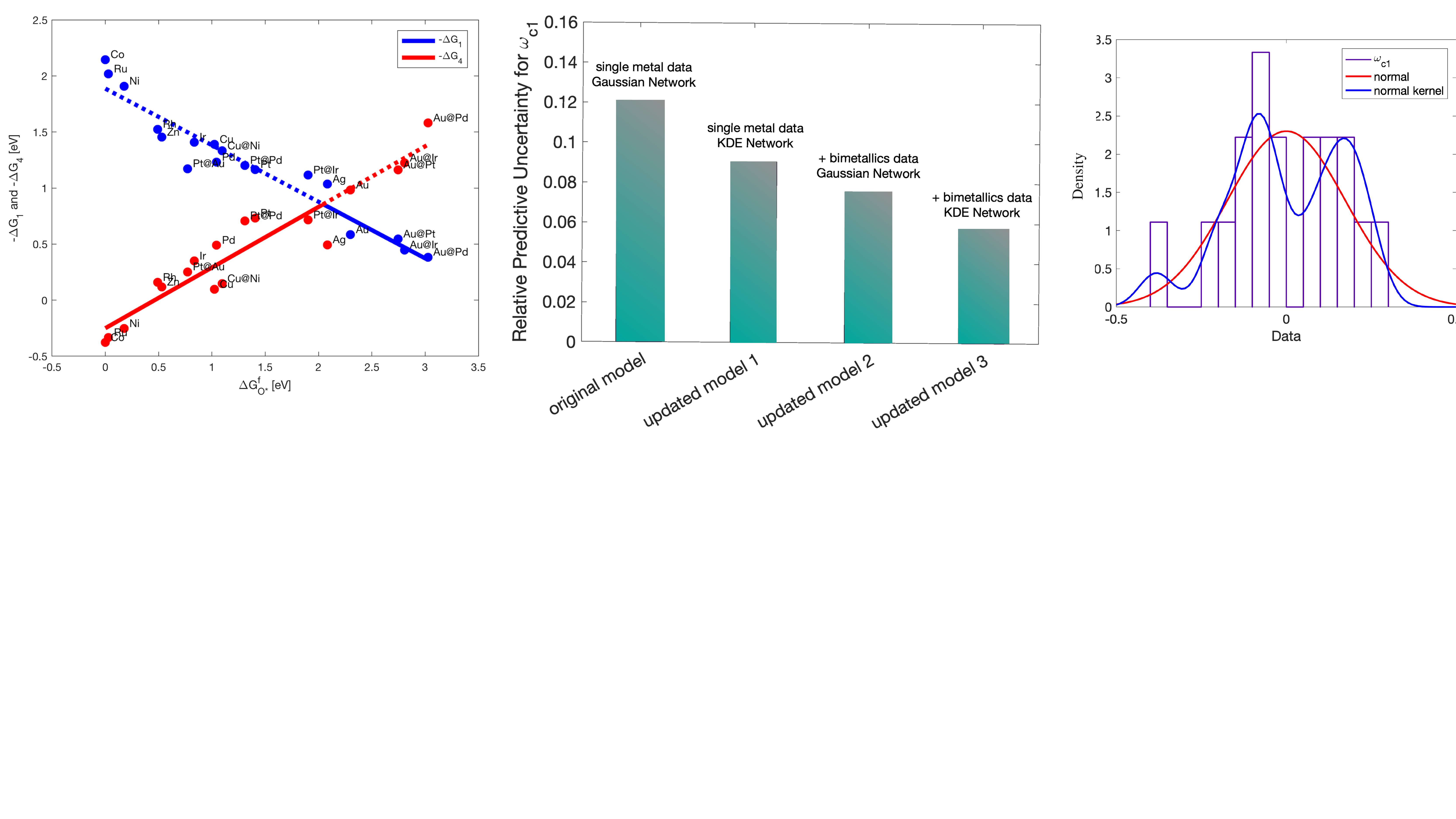}
\vspace{-4cm}
\caption{\small {\bf (Left)} DFT-computed data for reaction energies with respect to different metals/oxygen binding energies. Here   bimetallics data are also included in addition to the single metals in  Figure~\ref{fig: PGM} (c). {\bf (Center)}  Different relative model sensitivities \eqref{eq:TOL:2}
when we:  only perturb the model of $\omega_{c1}$ by $\eta_{c1}=R(data\|P_{c1})$ when $P_{c1}$ is Gaussian with the original single-metal data;  or using a KDE given by \eqref{eq:KDE} with the original data (updated model 1);  or using a Gaussian with the additional  bimetallics data (updated model 2); or using both KDE and  Bimetallics data (updated model 3). {\bf (Right)} Baseline model (Gaussian) of $\omega_{c1}$ (red curve) and the updated model (normal-kernel density estimation, blue curve) and additional  bimetallics data in this figure (Left).}
\label{fig:bimetallics}
\end{figure}

\noindent
 {\it 2. Increasing the complexity of CPDs.} We reduce the model misspecification $\eta_{l^\ast}$ by picking a better model $\tilde{P}_{l^\ast}$ than the baseline model $P_{l^\ast}$ for the $l^*$ component. The new model should represent the (fixed) available data more accurately by using a kernel-based method. In this case the new model is a mixture of  Gaussian and kernel-based networks \cite{koller2009probabilistic}.
For example, we replace the linear, Gaussian model for $\omega_{c1}$ demonstrated in Fig. \ref{fig: PGM} (g) with a linear, kernel-based model as shown in Figure~\ref{fig:bimetallics} (Right). Then  we can reduce the model sensitivity indices by decreasing the model misspecification $\eta_{l^*}$  without introducing new errors in the remaining components of the Bayesian Network as proved in Theorem~\ref{THM:GBNsCor}  (a). 

Moreover, 
we can  combine the approaches  above  to reduce the model sensitivity indices. For example, after adding more bimetallics data,  we first reduce the model sensitivity indices for the correlation errors $\omega_{ci}$. Then we  further reduce the indices of $\omega_{c1}$ by  replacing the corresponding component of  the baseline model for $\omega_{c1}$ (Gaussian model) by normal kernel density estimator without increasing the indices of the remaining nodes (see Theorem~\ref{THM:GBNsCor} (a)). The new model is the updated model 3 in  Figure~\ref{fig:bimetallics} (Center).
%
We can  compute the model sensitivity indices for the updated mixed model, where $P_l$ could be KDE or another distribution, using
Theorem~\ref{thm:MFSI} and in particular \eqref{ineq:MFSI variational}.


\smallskip
\noindent
{\it 3. Increasing the complexity of the graph.} Here, we discuss how model sensitivity indices can investigate  the change in graph structure. The available data for solvation energies  in  Figure~\ref{fig:orr_s} (Left), indicate  that there might be a linear dependence between $\omega_{s1}$ and $\omega_{s2}$. We represent such a connection as a directed edge $\omega_{s1}\to\omega_{s2}$, and thus the new graph has an extra edge illustrated in orange in Figure~\ref{fig:orr_s} (Right). The CPDs of the new Bayesian Network $Q$ are given by 
\begin{equation}
    q(\omega_{s1}) := p(\omega_{s1}) = \mathcal{N}(\beta_{s1,0},\sigma_{s1}^2)
\end{equation}
\begin{equation}
    q(\omega_{s2}|\omega_{s1}) := \mathcal{N}(\omega_{s1} + \beta_{s2,0},\sigma_{s2}^2)
\end{equation}
and all the remaining ones (i.e. $x, y_1, y_2$ and all $\omega_j$ with $j\neq s_2$) are the same and given by \eqref{eq: omega dist}-\eqref{eqyi'}. The correlation parameters $\beta_{s1,0},\beta_{s2,0}$ as well as $\sigma_{s1}^2, \sigma_{s2}^2$ are learned by using the global likelihood decomposition method mentioned earlier.
\begin{figure}[ht]
\centering
\includegraphics[width=0.9\textwidth]{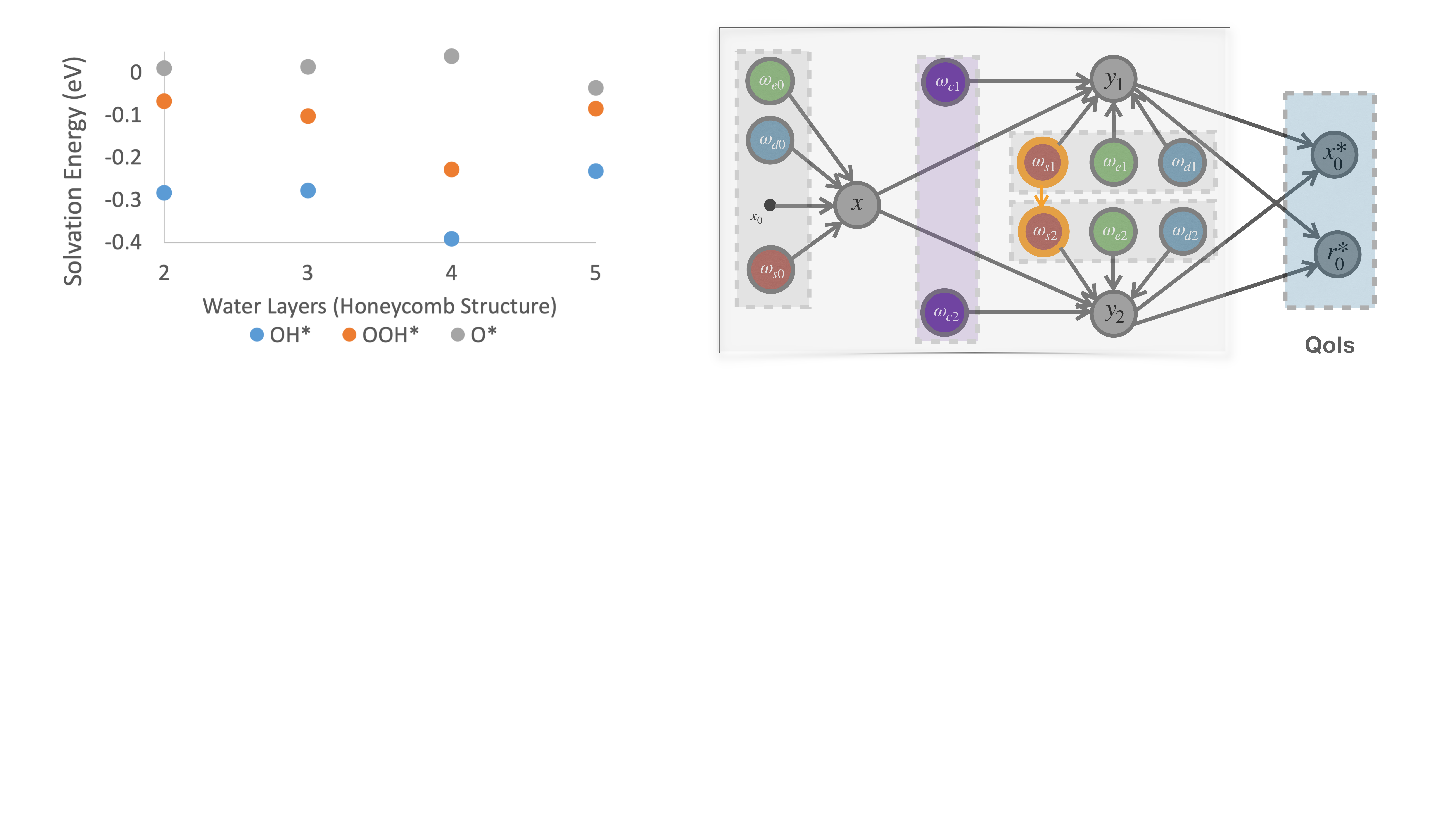}
\vspace{-3.7cm}
\caption{\small {\bf (Left)} DFT data for solvation energies $\omega_{s0}$ and $\omega_{si}$  of $x$ and $y_i$ respectively, with different water layers. {\bf (Right)} Based on the left figure, a potential correlation between $\omega_{si}$ is found. We incorporate such a correlation into the graph by adding a new edge between $\omega_{si}$ (orange edge) into the existing graph in Figure~\ref{fig: PGM} (b). The two energies  $y_1$ and $y_2$ are now not conditionally independent given $x$. However, by using  \eqref{eq:TOL:2}, the model sensitivity indices $I^{\pm}(x_{O^\ast}^P, P; \mathcal{D}^{\eta_{s2}}_{s2})$ are very small compared to the QoI $\mathbb{E}_P[x_{O^\ast}^P]$ 
(here the index in \eqref{eq:TOL:1} is normalized by the QoI) 
 implying that we can  ignore the proposed graph connection}
\label{fig:orr_s}
\end{figure}
The KL divergence between the Gaussian Bayesian networks $P$ and $Q$ of Figure~\ref{fig: PGM} (b) and Figure~\ref{fig:orr_s} respectively is given by
\begin{eqnarray}
R(Q\|P) 
&=& \int \log \frac{q(\omega_{s2}|\omega_{s1}) }{p(\omega_{s2})} q(\omega_{s2}|\omega_{s1}) q(\omega_{s1}) ds_2ds_1
\end{eqnarray}
and serves as a surrogate for the model misspecification $\eta_{s2}$. Using  gaussianity   $\eta_{s2} = 0.9173$, and by  Theorem \ref{thm:MFSI general},
$I^{\pm}(x_{O^\ast}^P, P; \mathcal{D}^{\eta_{s2}}_{s2}) = \pm 0.0928$. The latter value   is very small compared to the QoI  $\mathbb{E}_P[x_{O^\ast}^P] = 2.0434$. Thus, we may safely ignore the correlation between $\omega_{s1}$ and $\omega_{s2}$. Therefore, no further model improvement is necessary and we can retain the (simpler) baseline Bayesian network of Figure~\ref{fig: PGM} (b).

\section{Conclusions}\label{sec:conclusions}
In this paper, we developed information-theoretic, robust uncertainty quantification methods and non-parametric stress tests for Bayesian networks, which allowed us to assess the effect and the propagation through the graph of multi-sourced model uncertainties to the quantities of interest. These quantification methods also allowed us to rank these   sources  of  uncertainty  and correct the graphical  model by targeting the most influential components  with respect to  the quantities of interest.
%
However, one of the challenges  we did not discuss in depth here is the selection of the probabilistic metric or divergence $d$ in the formulation of robust uncertainty quantification,  e.g. in the definition of model uncertainty indices \eqref{eq:stresstest:intro}. In this paper we selected
the KL divergence to define  the ambiguity sets \eqref{eq:set:MFUQ:intro} 
since it allowed us to obtain easily computable and scalable model uncertainty indices.
However,  for Bayesian networks with  vastly different graphical structures e.g. an alternative model with more vertices than the baseline, the choice of KL is not suitable
due to the lack of absolute continuity between the baseline and the alternative model. In such cases, new divergences could be considered e.g. Wasserstein metrics already studied in the DRO literature \cite{MohajerinEsfahani2018,doi:10.1287/moor.2018.0936} or their  Integral Probability Metrics (IPM) generalization \cite{Muller1997};
alternatively we can consider various interpolations of divergences and IPMs studied recently in the machine learning literature such as 
\cite{Cuturi:short:2017, dupuis2019formulation, BDKPRB, KALE:Gretton:2021} and references therein.
For instance, the recently introduced $(f, \Gamma)$-divergences \cite{BDKPRB} are interpolations of $f$-divergences and IPMs  that combine  advantageous features  of both,
such as the capability to handle heavy-tailed data (property inherited from $f$-divergences) and to compare non-absolutely continuous distributions (inherited from IPMs).   
An additional issue that we touched upon here when we discussed model sensitivity indices is the need for  divergences to be able to isolate 
sources of uncertainty on localized parts of the graphical model in the spirit of ``divide and conquer". In that respect concepts of sub-additivity of divergences for PGMs can be essential as discussed in related recent literature
\cite{Daskalakis:PGM:2021GANsWC, dupuis2019formulation}.
\smallskip

\noindent{\bf Acknowledgments.}
The research of P.B. was supported by the Air Force Office of Scientific Research (AFOSR) under the grant FA-9550-18-1-0214. The research of J.F. was partially supported by the Defense Advanced Research Projects Agency (DARPA) EQUiPS program under the grant  W911NF1520122. 
The research of M. K. and L. R.-B. was partially supported by 
by the Air Force Office of Scientific Research (AFOSR) under the grant FA-9550-18-1-0214
and by 
the National Science Foundation (NSF) under
NSF TRIPODS  CISE-1934846 and  the grant DMS-2008970.

\newpage

\appendix

\section{Background on Model Uncertainty}\label{subsec:MFRMFU}

\subsection{Mathematical formulation of model uncertainty}

 We can formulate  mathematically   model uncertainty  by constructing (non-parametric) families  ${\mathcal Q}$ of  {\it alternative models} $Q$ to compare to a \textit{baseline model $P$} which is computationally tractable and inferred from data, and believed to be a good approximation for the physical model of $X$, while the ``true", intractable, partially unknown  model $Q^*$ should belong to $\mathcal{Q}$; for this reason we refer to ${\mathcal Q}$ as the \textit{ambiguity set}, typically defined as a neighborhood of models around the baseline $P$:
\begin{equation}\label{eq:ambiguity:0}
    \mathcal{Q}= \mathcal{D}^\eta=\big\{ Q: d(Q, P) \leq \eta\big\}\, ,
 \end{equation}
where  $\eta >0$ corresponds to  the size of the ambiguity set 
and  $d=d(Q, P)$ denotes a probability metric or divergence (see Figure~\ref{fig: MFUQ} (Left) for the schematic depiction where $d$ is the Kullback Leibler (KL) divergence (aka relative entropy) $R(Q\|P)$, \cite{cover2012elements}). The next  natural  mathematical goal is to assess the baseline model and understand the resulting biases for QoIs $f$ when we use $P$ for predictions instead of the true  model $Q^*\in \mathcal{Q}$.  As we see later,  the free energies $f = -\Delta G_i$ are considered as QoIs for the ORR PGM (see Section \ref{sec: ORR}).

\begin{figure}[ht]
\centering
\includegraphics[width=0.4\textwidth]{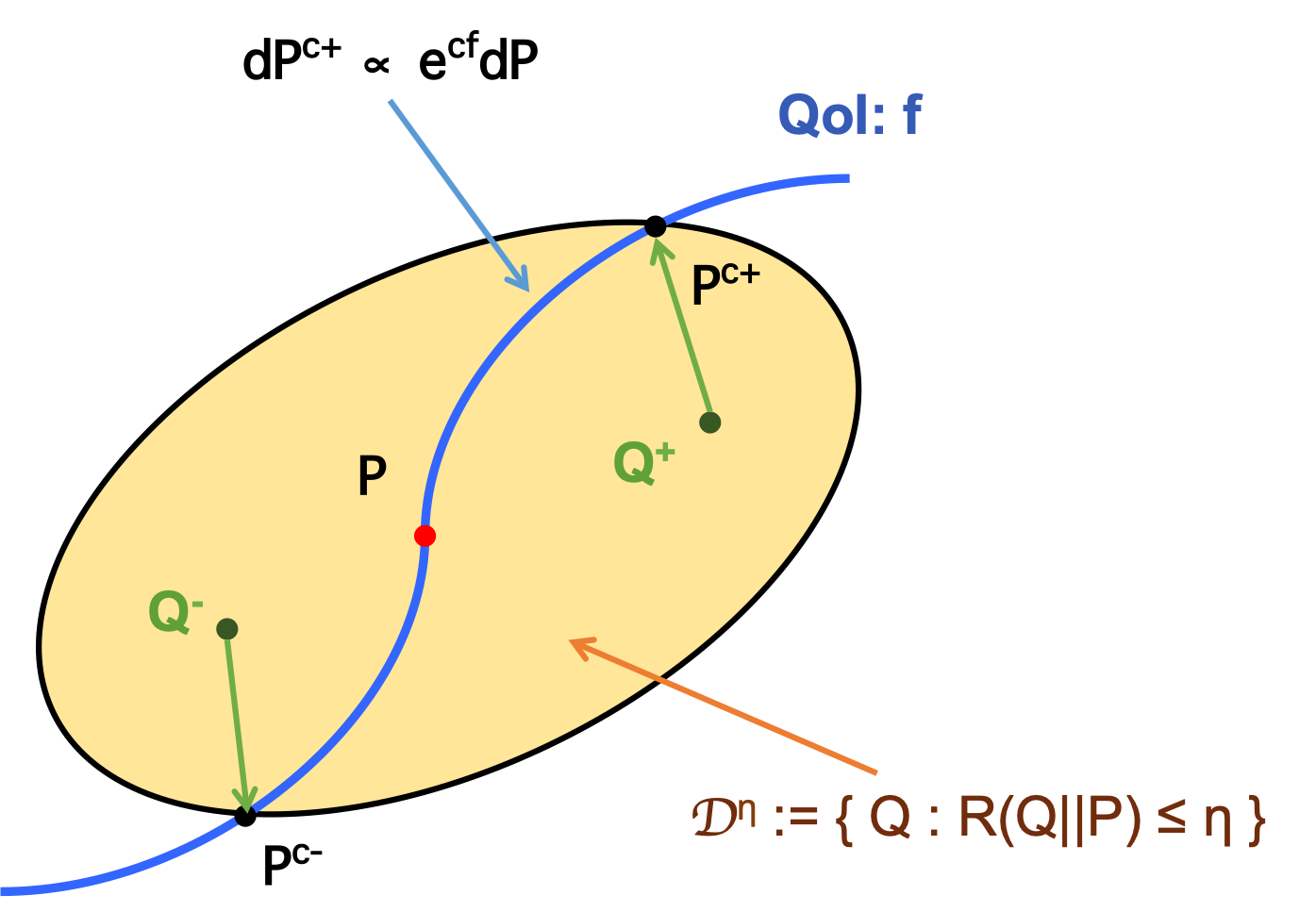}
\hspace{0.15in}
\includegraphics[width=0.4 \linewidth]{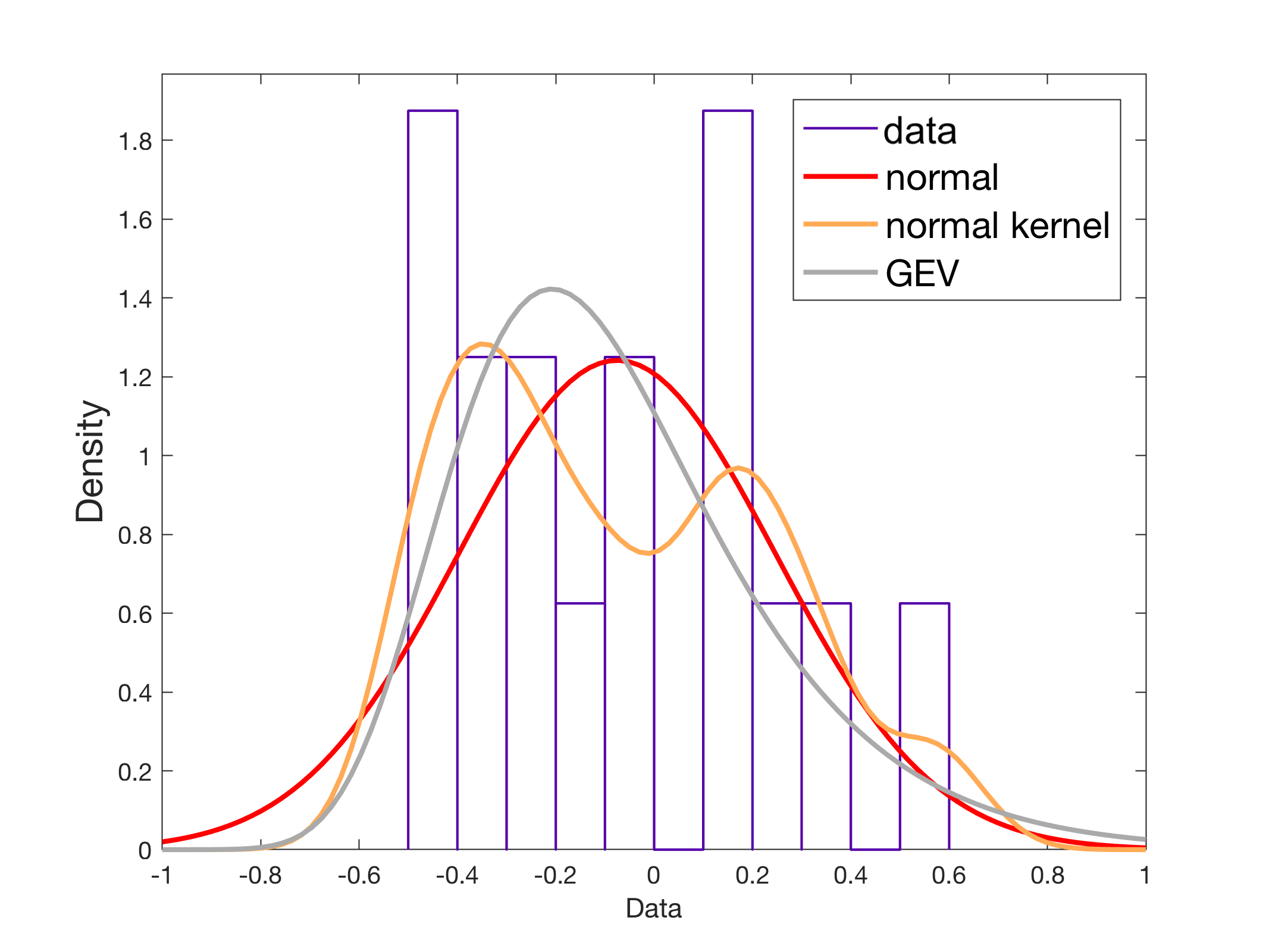}
\caption{\small {\bf (Left)} The schematic illustration of the ambiguity set (non-parametric family of models) given by \eqref{eq:ambiguity:0} with $d$ being the KL Divergence $R(Q\|P)$; the blue line represents a parametric family; $Q^\pm$ are the probability measures  that  the UQ indices/bounds $I^\pm$  with respect to QoI $f$ are attained and are provided by  \eqref{eq:optimizer} i.e. tightness of the bounds.  {\bf (Right)}  Three  probabilistic models with different CPDs for sparse data of a ORR PGM vertex $\omega_{d0}$: the red curve is used to build a baseline Gaussian model denoted by $P$, the gray curve is another parametric model (Generalized Extreme Value (GEV) distribution) which fits the data better, and the yellow curve is a non-parametric model (Kernel Density Estimation (KDE) with normal kernel). }
\label{fig: MFUQ}
\end{figure}

We define  the \textit{predictive uncertainty} (or bias)  for the QoI $f$ when using the baseline model $P$ instead of any alternative model $Q \in \mathcal{Q}$ as the  two worst case scenarios:
\begin{equation}\label{eq:Predictive:Uncertainty:0}
        I^\pm(f,P;\mathcal{Q}) := \underset{Q \in \mathcal{Q}}{\mathrm{sup/inf}}\  \{\MEANNN{Q}{f} -\MEANNN{P}{f} \}
    \end{equation}
    where $\MEANNN{Q}{f}$ denotes the expected value of the QoI $f$.  Therefore, \eqref{eq:Predictive:Uncertainty:0} provides {\em a robust performance guarantee} for the predictions of the baseline model $P$  for the QoI $f$ within
    the ambiguity set  $\mathcal{Q}$. This robust perspective for general probabilistic models $P$ is also known in Operations Research as \textit{Distributionally Robust Optimization} (DRO), e.g. 
     \cite{doi:10.1287/opre.1090.0741,doi:10.1287/opre.1090.0795,doi:10.1287/opre.2014.1314,Jiang2016,2016arXiv160402199G,2016arXiv160509349L,MohajerinEsfahani2018,doi:10.1137/16M1094725,doi:10.1287/moor.2018.0936}, where  optimal-transport (Wasserstein)  metrics were recently proposed for \eqref{eq:ambiguity:0}. Note that the predictive uncertainty represents the robustness of the model $P$ with respect to $\mathcal{Q}$, i.e. all the biases between the predictions of $f$ with $Q \in \mathcal{Q}$ and $P$ are bounded by the predictive uncertainty.
     
 \subsection{Existing results on model uncertainty}
While the definition \eqref{eq:Predictive:Uncertainty:0} 
is rather natural and  intuitive, at least  based on the model uncertainty  challenge depicted in Figure~\ref{fig: MFUQ} (Left), it is not obvious that it is  practically computable.  
However it becomes  tractable  if we use for metric $d$ in \eqref{eq:ambiguity:0}  the KL divergence
$R(Q\|P)$. Accordingly, $\eta$ is a  measure of the confidence  we put in the baseline model $P$ measured using KL divergence.
%
In recent work \cite{dupuis2011uq, dupuis2016path, gourgoulias2017biased,katsoulakis2017scalable}, it has been shown that $I^\pm(f,P;\mathcal{D}^\eta)$ (an infinite dimensional optimization problem) can be directly computable by a one dimensional optimization problem: 
\begin{equation}
\label{eq:UQ variational form}
\quad\quad    I^{\pm}(f, P; {\mathcal D}^\eta) = \pm \inf_{c>0}\Big[\frac{1}{c} \log \int e^{\pm c(f - \MEANNN{P}{f})}P(dx) + \frac{\eta}{c}\Big]=\MEANNN{Q^\pm}{f} -\MEANNN{P}{f} \, .
\end{equation}
which is derived by using the Gibbs variational principle \cite{dupuis2016path} for KL divergence. In the first equality of this formula we recognize two ingredients: $\eta$ is  model uncertainty from \eqref{eq:ambiguity:0} while the Moment Generating Function (MGF) $\int e^{\pm cf }P(dx)$ encodes the QoI $f$ at the baseline model $P$. In \cite{dupuis2016path, gourgoulias2017biased,katsoulakis2017scalable} techniques are developed to compute (exactly or approximately via asymptotics \cite{dupuis2016path}) as well as provide explicitly upper and lower bounds on $I^{\pm}(f, P; {\mathcal D}^\eta)$ in terms of concentration inequalities \cite{gourgoulias2017biased}. 
A  key point in \eqref{eq:UQ variational form} is that the parameter $\eta$ is not  necessarily small, allowing {\it global \& non-parametric} sensitivity analysis. 

\medskip

Moreover, in \cite{gourgoulias2017biased} the authors have proven that the second equality of \eqref{eq:UQ variational form} holds. In fact, this shows that $I^{\pm}(f, P; {\mathcal D}^\eta)$ is also tight, i.e when the sup and inf in \eqref{eq:Predictive:Uncertainty:0} are attained by appropriate measures $Q^\pm$. Formally, the authors have shown that there exist $0< \eta_{\pm} \le \infty$, such that for any $\eta \le \eta_{\pm}$, $Q^\pm(\cdot)=Q^\pm(\cdot\;;\pm c_{\pm})$ depend on $\eta$ and are given by
\begin{equation}
\label{eq:optimizer}
    dQ^{\pm} = \frac{e^{\pm c_\pm f}}{\MEANNN{P}{e^{\pm c_\pm f}}}\ dP
\end{equation}
where $c_\pm\equiv c_\pm(\eta)$  are the unique solutions of 
\begin{equation}
R(Q^{\pm}\|P)
=\eta.\label{eq:cond:re}%
\end{equation}

\subsection{Some fundamental  Lemmas}
In this subsection, we include Lemma~\ref{lemma:bounds} and \ref{lemma:tightness} for completeness  of the background presentation. These results were    proved  in \cite{dupuis2016path,dupuis2018sensitivity,gourgoulias2017biased}
and we present them here for the convenience of the reader.

\begin{lemma}
\label{lemma:bounds}
Let $P$ be a probability measure and let $f(X)$ be such that its MGF is finite in a neighborhood of the origin. Then for any $Q$ with $R(Q||P) < \infty$, we have
\begin{equation}\label{ineq:GO bounds}
-\inf_{c>0}\Big[\frac{1}{c} \log \MEANNN{P}{e^{- c\bar{f}(X)} } + \frac{\eta}{c}\Big]\leq \MEANNN{Q}{f(X)}-\MEANNN{P}{f(X)} \leq \inf_{c>0}\Big[\frac{1}{c} \log \MEANNN{P}{e^{c\bar{f}(X)} } + \frac{\eta}{c}\Big]
\end{equation}
\end{lemma}

\noindent
{\it Proof of Lemma \ref{lemma:bounds}.} For any general QoI $f(X)$ which has finite moment generating function (MGF), $\MEANNN{P}{e^{\pm c\bar{f}(X)} }:=\MEANNN{P}{e^{c(f(X) - \MEANNN{P}{f(X)})}}$, in a neighborhood of the origin, there is a known fact in statistics and large deviation theory \cite{dupuis2011weak,dupuis2016path} that
\begin{equation}
    \log \MEANNN{P}{e^{f(X)}} = \sup_{Q \ll P} \{ \MEANNN{Q}{f(X)} - R(Q||P)\}\ .
\end{equation}
Changing $f(X)$ to $c\bar{f}(X)=c(f(X) - \MEANNN{P}{f(X)})$, we get
\begin{equation}
    \MEANNN{P}{e^{\pm c\bar{f}(X)} } = \sup_{Q \ll P} \{ c(\MEANNN{Q}{f(X)}-\MEANNN{P}{f(X)}) - R(Q||P)\}
\end{equation}
which gives us the following upper and lower bounds with $c>0$,
\begin{equation}
   - \inf_{c>0}\Big[\frac{1}{c} \log \MEANNN{P}{e^{-c\bar{f}(X)} } + \frac{\eta}{c}\Big] \leq \MEANNN{Q}{f(X)}-\MEANNN{P}{f(X)} \leq \inf_{c>0}\Big[\frac{1}{c} \log \MEANNN{P}{e^{c\bar{f}(X)} } + \frac{\eta}{c}\Big]
\end{equation}
where $\eta = R(Q||P)$.

\begin{lemma}
\label{lemma:tightness}
Suppose $(d_-,d_+)$ is the largest open set such that cumulant generating function $\Lambda(c)=\log \MEANNN{P}{e^{c\bar{f}(X)}} <\infty$ for all $c \in (d_-,d_+)$.
\begin{enumerate}
\item For any $\eta \ge0$ the optimization problems
\begin{equation*}
\inf_{c>0} \frac{\Lambda(\pm c)+\eta}{c}
\end{equation*}
have unique minimizers $c_{\pm}\in[0, \pm d_\pm]$.   
Let $\eta_\pm$ be defined by 
\[
\eta_\pm = \lim_{c \nearrow \pm d_\pm}   \pm c \Lambda'( \pm c) - \Lambda(\pm c) \,.
\]
Then 
the minimizers $c_{\pm}=c_{\pm}(\eta)$
are finite for $\eta < \eta_{\pm}$ and $c_{\pm}(\eta) = \pm d_\pm$ if $\eta \geq \eta_{\pm}$.

\item If $c_{\pm}(\eta) < \pm d_{\pm}$ then
\begin{equation}
\frac{\Lambda(\pm c_{\pm})+\eta}{c_{\pm}}=
\inf_{c>0} \frac{\Lambda(\pm c)+\eta}{c}
\,=\,\pm \Lambda^{\prime}%
(\pm c_{\pm})=\pm \left( \mathbb{E}_{P_{\pm c_{\pm}}}[f] - \mathbb{E}_{P}[f]\right) \,,\label{eq:Bpmeqn}%
\end{equation}
where $c_{\pm}(\eta)$ is strictly increasing in $\eta$ and is determined by the
equation
\begin{equation}
R\left(  {P_{\pm c_{\pm}}}{\,||\,}{P}\right)
=\eta\,.%
\end{equation}

\item $\eta_{\pm}$ is finite in two distinct cases.
\begin{enumerate}
\item If $\pm d_\pm < \infty$ (in which case $g$ must be unbounded above/below) $\eta_\pm$ is finite if $\lim_{c \to \pm d_\pm} \Lambda(\pm c) := \Lambda( d_\pm) < \infty$ and  $\lim_{c \to \pm d_\pm} \pm \Lambda'(\pm c) := \pm \Lambda'(d_\pm) < \infty$, and  for $\eta \ge \eta_{\pm}$ we have 
\begin{equation}
\label{eq:dplus}
\inf_{c>0} \frac{\Lambda(\pm c)+\eta}{c} = \frac{\Lambda( d_{\pm})+\eta}{ \pm d_\pm}= \pm\left( \mathbb{E}_{P_{ d_{\pm}}}[f]-\mathbb{E}_{P}[f]\right) +  \frac{\eta - \eta_\pm}{\pm d_\pm} \,.
\end{equation}

\item If $\pm d_\pm = \infty$ and $\eta_\pm$ is finite then $f$ is $P$-a.s. bounded above/below and for $\eta \ge \eta_{\pm}$ we have 
\begin{equation}
\label{eq:gplus}
\inf_{c >0} \frac{\Lambda(\pm c)+\eta}{c}
=\mathrm{ess\,sup}%
_{x\in\mathcal{X}}\{\pm (f(x)- \MEANNN{P}{f(X)})\}\,.
\end{equation}
\end{enumerate}
\end{enumerate}
\end{lemma}

\noindent
{\it Proof of the Lemma \ref{lemma:tightness}.}
For notational ease, in the proof, let 
us set $\Lambda(c)= \log \MEANNN{P}{e^{c\bar{f}(X)}}$ so that the UQ indices is 
$$   I^\pm(f(X),P;\mathcal{D}^\eta) = \inf _{c>0}\left\{ \frac{ \Lambda(\pm c) +   \eta}{c} \,  \right\}$$
Note that $\Lambda(c)$ is convex function which we assume to 
be finite on an interval $(d_-, d_+)$ with $d_- < 0 <d+$. On that interval $\Lambda(c)$ is infinitely differentiable and strictly convex.   Since we centered the QoI we have $\Lambda(0)=\Lambda'(0)=0$ and $\Lambda''(0)=\VAR_P(f)$. 

\noindent
First note that it is enough to prove the result for $\Lambda(c)$ since the result for $\Lambda(-c)$ is obtained by replacing $f$ by $-f$. We also use the notation $\tilde{f}_+=\mathrm{ess\,sup}\{f(x)-\MEANNN{P}{f(X)}\}$. 

\noindent
We first claim that automatically%
\[
\Lambda(d_{+})=\lim_{c\nearrow d_{+}}\Lambda(c),
\]
where $\Lambda(d_{+})$ may be infinite. By monotone convergence
\[
\mathbb{E}_{P}[1_{\{ \tilde{f} \geq0\}}e^{c \tilde{f}}]\nearrow\mathbb{E}_{P}%
[1_{\{\tilde{f}\geq0\}}e^{d_{+} \tilde{f}}]
\]
as $c\nearrow d_{+}$. By dominated convergence%
\[
\mathbb{E}_{P}[1_{\{\tilde{f}<0\}}e^{c \tilde{f}}]\searrow\mathbb{E}_{P}[1_{\{\tilde{f}<0\}}%
e^{d_{+} \tilde{f}}]
\]%
as $c\nearrow d_{+}$, and the claim follows. A very similar argument shows that $\Lambda^{\prime
}(c)$ also has a limit as $c\nearrow d_{+}$.

\noindent
Let
\begin{equation}\label{eq:B:Appendix}
B(c;\eta)=\frac{\Lambda(c)+\eta}{c}.
\end{equation}
We divide into cases. 

\begin{enumerate}
\item $\tilde{f}_{+}<\infty$. In this case $\Lambda^{\prime}(c)\nearrow \tilde{f}_{+}<\infty$
as $ c \rightarrow\infty$ and $\Lambda^{\prime}(0)<\tilde{f}_{+}$. If $\eta=0$ then the
infimum is $\Lambda^{\prime}(0)$ and attained at $c_{+}=0$ since $\Lambda(c)/c$ is an increasing function. If $\eta>0$ then
\[
B^{\prime}(c;\eta)=\frac{c \Lambda^{\prime}(c)-\Lambda(c)-\eta}{c^{2}}%
\]
for $c \geq0$. The function $c  \Lambda^{\prime}(c )-\Lambda(c )$ strictly
increases from $0$ at $c =0$ to some limit $\eta_{+}>0$ at $c =\infty$,
and the minimizer is at the unique finite root of $c  \Lambda^{\prime}%
(c )-\Lambda(c )=\eta$ for $\eta<\eta_{+}$ and $c _{+}=\infty$ for $\eta\geq \eta_{+}$.

\item $\tilde{f}_{+}=\infty$. In this case there are two subcases.

\begin{enumerate}
\item $d_{+}=\infty$. In this case since $\tilde{f}_{+}=\infty$ we have $\Lambda^{\prime
}(c )\nearrow\infty$ as $c \rightarrow\infty$ and $c  \Lambda^{\prime
}(c )-\Lambda(c )\rightarrow\infty$ as $c \rightarrow\infty$. Since  
$0\Lambda^{\prime}(0)-\Lambda(0)=0$, in all cases of $\eta\geq0$ there is a unique root
to $c  \Lambda^{\prime}(c )-\Lambda(c )=\eta$ and hence a unique minimizer.

\item $d_{+}<\infty$. We know that $\Lambda^{\prime}(c )$ converges as
$c \nearrow d_{+}$ to a well defined left hand limit which we call
$\Lambda^{\prime}(d_{+})$ (note that this value could be $\infty$). Thus we have
that $c  \Lambda^{\prime}(c )-\Lambda(c )$ ranges from $0$ at $c =0$ to
$\eta_{+}=d_{+}\Lambda^{\prime}(d_{+})-\Lambda(d_{+})$. For $\eta\in\lbrack0,\eta_{+})$ there is a
unique minimizer in $[0,d_{+})$. For $\eta\geq \eta_{+}$ the unique minimizer is at
$c _{+}=d_{+}$.
\end{enumerate}
\end{enumerate}

\noindent
 To conclude the proof we note that if 
 $c _+ < d_+$ then an easy computation shows that 
 $$c _+ \Lambda'(c _+)- \Lambda(c _+) = R(P_{c _+}\,||\,P)=\eta
\,, $$
 and thus 
 \[
 B(c _+, \eta) = \Lambda'(c _+)=\mathbb{E}_{P_{c_+}}[f] - \MEANNN{P}{f(X)}
 \]
which proves \eqref{eq:Bpmeqn} and \eqref{eq:cond:re}.  Finally if $d_+=\infty$ and $f$ is $P$-a.s. bounded above then the infimum is equal to $\lim_{c  \to \infty}\frac{\Lambda(c )}{c }$ and this establishes \eqref{eq:gplus}.
If $d_+ < \infty$ and $\eta_+ < \infty$ then the bound takes the form \eqref{eq:dplus}.

\section{A simple example  for Bayesian networks}
\begin{expl}\label{ex:gBayesian network}
 In this example, we focus on the construction of the graph structure and CPDs of the optimal distributions provided by Theorem~\ref{thm:MFUQ PGM} (b) following the strategy of its proof. Note that in the next subsection by assuming
that each $X_i$ is linear Gaussian of its parents, we also compute the model uncertainty indices given by \eqref{eq:MFUQ PGM} in Theorem~\ref{thm:MFUQ PGM} (a). Let us consider a Bayesian network as shown in Figure~\ref{fig: MFUQ optimizer} (a), with density given by
\begin{equation}\label{eq:Ex:Thm2.1}
    p(x) = p(x_1)p(x_2)p(x_3|x_2,x_1)p(x_4)p(x_5|x_3)p(x_6|x_4,x_3)p(x_7|x_6,x_5)p(x_8|x_6)
\end{equation}
For a QoI $f(X_6)$, 
the optimizers in Theorem~\ref{thm:MFUQ PGM} (b) are obtained when the CPDs of $X_5$, $X_7$ and $X_8$ are the same with the corresponding CPDs of $P$ as these vertices are not ancestors of $X_6$ while
\begin{eqnarray}\label{ex161}
    q^\pm (x_6|x_{\pi_6^{Q^\pm}}) &=& \frac{e^{\pm c_\pm f(x_{6})}}{\MEANNN{P_{6|\{4,3\}}}{e^{\pm c_\pm f(X_6)}}} \cdot p(x_6|x_4,x_3)
\end{eqnarray}
where $\pi_6^{Q^\pm} \equiv \pi_6^P = \{4,3\}$,  then for $i \in \rho_6 = \{1,2,3,4\}$
\begin{eqnarray}
     q^\pm(x_{4}|x_{\pi_4^{Q^\pm}})  &=& \frac{\MEANNN{P_{6|\{4,3\}}}{e^{\pm c_\pm f(X_6)}}}{\MEANNN{P_{4}}{\MEANNN{P_{6|\{4,3\}}}{e^{\pm c_\pm f(X_6)}}}}  p(x_4)
\end{eqnarray}
since both normalization factors on the numerator and denominator depend on $X_{\pi_6}  = \{X_4, X_3\}$, so in general, we have $\pi_4^{Q^\pm} = \pi_4^P \cup \{3\} = \{3\}$, i.e., there is a new connection $X_3 \to X_4$ in $Q^\pm$, and 
\begin{eqnarray}\label{ex162}
     q^\pm(x_{3}|x_{\pi_3^{Q^\pm}})  &=& \frac{\MEANNN{P_{4}}{\MEANNN{P_{6|\{4,3\}}}{e^{\pm c_\pm f(X_6)}}}}{\MEANNN{P_{3|\{2,1\}}}{\MEANNN{P_{4}}{\MEANNN{P_{6|\{4,3\}}}{e^{\pm c_\pm f(X_6)}}}}}  p(x_{3}|x_2,x_1)
\end{eqnarray}
where $\pi_3^{Q^\pm} \equiv \pi_3^P = \{2,1\}$ since the normalization factors do not contain other variables. We can similarly do the same for $X_2$ and $X_1$ to get the entire structure of $Q^\pm$ which has another new connection $X_1 \to X_2$, and the results are shown in Figure~\ref{fig: MFUQ optimizer} (b).

For $A=\{3,6,7\}$, we consider a QoI  $f(X_A)=f(X_3,X_6, X_7)$ and by Theorem \ref{thm:MFUQ PGM} (a), the following holds:
\begin{eqnarray}
    I^{\pm}(f(X_3,X_6, X_7), P; \mathcal{D}^\eta) &=&   \pm \inf_{c>0}\Big[\frac{1}{c} \log \MEANNN{P_{A}}{ e^{ \pm c\bar{f}(X_A)} } + \frac{\eta}{c}\Big]\nonumber\\
    &=& \MEANNN{Q^\pm}{f(X_A)} -\MEANNN{P}{f(X_A)}
\end{eqnarray}
where $Q^\pm$ are the optimizers with CPDs given by \eqref{exnode8}-\eqref{exgen*} and 
\[
\MEANNN{P_{A}}{ e^{ \pm c\bar{f}(X_A)}}= \int e^{ \pm c\bar{f}(x_3,x_6,x_7)} \prod_{i=1}^{7}p(x_i|x_{\pi_{i}})dx_i
\]
We recall  \eqref{eq:MFUQ optimizer:1}  of Theorem \ref{thm:MFUQ PGM} (b), and we obtain the CPDs of $Q^{\pm}$ and the new parents of each vertex as follows:
\begin{equation}\label{exnode8}
    q^\pm (x_8|x_{\pi_8^{Q^\pm}}) \equiv p(x_8|x_{\pi_8^{Q^\pm}})\equiv p(x_8|x_6)
\end{equation}
\begin{eqnarray}
    q^\pm (x_7|x_{\pi_7^{Q^\pm}}) &=& \frac{e^{\pm c_\pm f(x_7,x_6,x_3)}}{\MEANNN{p_{7|\{6,5\}}}{e^{\pm c_\pm f(X_7,X_6,X_3)}}} \cdot p(x_7|x_6,x_5)
\end{eqnarray}
with $\pi_7^P\subset\pi_7^{Q^\pm}=\pi_7^P\cup\{3\}=\{6,5,3\}$. Let $\{l_1,\dots,l_6\}\equiv \rho_{3}^{P}\cup\rho_{6}^{P}\cup\rho_{7}^{P}\cup \{3,6\}=\{1,2,3,4,5,6\}$. We start with $X_6$ as it is indexed by the $\max \{l_j:j\in 1,\dots,6\}$
\begin{eqnarray}
    q^\pm (x_6|x_{\pi_6^{Q^\pm}}) &=&\frac{\mathbb{E}_{P_{7|\pi_7}}[e^{\pm c_\pm f(X_{A})}]}{\mathbb{E}_{P_{6|\pi_6},P_{7|\pi_7}}[e^{\pm c_\pm f(X_7,X_6,X_3)}]} \cdot p(x_6|x_4,x_3)
\end{eqnarray}
with $ \pi_6^P\subset\pi_6^{Q^\pm}\subset\pi_6^P\cup\{5\}=\{3,4,5\}$. Similarly, the CPD of $X_1,\cdots, X_5$ are given by 
\begin{align}
     q^\pm(x_{5}|x_{\pi_5^{Q^\pm}})  &= \frac{\mathbb{E}_{P_{6|\pi_6},P_{7|\pi_7}}[e^{\pm c_\pm f(X_{A})}]}{\mathbb{E}_{P_{5|\pi_{5}},P_{6|\pi_6},P_{7|\pi_7}}[e^{\pm c_\pm f(X_7,X_6,X_3)}]} \cdot p(x_5|x_3),\\ 
     q^\pm(x_{4}|x_{\pi_4^{Q^\pm}})  &= \frac{\mathbb{E}_{P_{5|\pi_{5}},P_{6|\pi_6},P_{7|\pi_7}}[e^{\pm c_\pm f(X_{A})}]}{\mathbb{E}_{P_{4|\pi_{4}},P_{5|\pi_{5}},P_{6|\pi_6},P_{7|\pi_7}}[e^{\pm c_\pm f(X_7,X_6,X_3)}]} \cdot p(x_4),\\
     q^\pm(x_{3}|x_{\pi_3^{q^\pm}})  &= \frac{\mathbb{E}_{P_{4|\pi_{4}},\cdots,P_{7|\pi_7}}[e^{\pm c_\pm f(X_{A})}]}{\mathbb{E}_{P_{3|\pi_{3}},\cdots,P_{7|\pi_7}}[e^{\pm c_\pm f(X_7,X_6,X_3)}]} \cdot p(x_3|x_2,x_1),  \\
     q^\pm(x_{2}|x_{\pi_2^{q^\pm}})  &= \frac{\mathbb{E}_{P_{3|\pi_{3}},\cdots,P_{7|\pi_7}}[e^{\pm c_\pm f(X_{A})}]}{\mathbb{E}_{P_{2|\pi_{2}},\cdots,P_{7|\pi_7}}[e^{\pm c_\pm f(X_7,X_6,X_3)}]} \cdot p(x_2|x_1), \\
     q^\pm(x_{1}|x_{\pi_1^{q^\pm}})  &= \frac{\mathbb{E}_{P_{2|\pi_{2}},\cdots,P_{7|\pi_7}}[e^{\pm c_\pm f(X_{A})}]}{\mathbb{E}_{P_{1|\pi_{1}},\cdots,P_{7|\pi_7}}[e^{\pm c_\pm f(X_7,X_6,X_3)}]} \cdot p(x_1),  \label{exgen*}
\end{align}
where the expectations involved in the above formulas are given by \eqref{eq:notation:conditionals:ancestors}. The corresponding structures are:
\begin{align}
&\qquad \pi_5^P\subset\pi_5^{Q^\pm}\subset\pi_5^P\cup\{4\}=\{3,4\}\\
&\qquad \pi_4^P\subset\pi_4^{Q^\pm}\subset\pi_4^P\cup\{3\}=\{3\}\\
&\qquad \pi_3^P\subset\pi_3^{Q^\pm}\subset\pi_3^P=\{1,2\}\\
&\qquad \pi_2^P\subset\pi_2^{Q^\pm}\subset\pi_2^P\cup\{1\}=\{1\}\\
&\qquad \pi_1^{Q^\pm}=\pi_1^P=\emptyset.
\end{align}
As a result, the structure of the associated graph to $Q^{\pm}$ may change and in particular, the vertices-with potentially extra parents-are $X_2,X_4, X_5, X_6$ and $X_7$ as illustrated in Figure \ref{fig: MFUQ optimizer} (c).
\end{expl}

\section{A simple example for  Gaussian Bayesian networks}\label{sec:GBNsup}

\begin{expl} [Continuation of  Example~\ref{ex:gBayesian network}]\label{ex:cont} We assume that CPDs of Example~\ref{ex:gBayesian network} with graph structure as in Figure~\ref{fig: MFUQ optimizer}, (a) are given by:
\begin{equation*}
\begin{aligned}[c]
p(x_8|x_6)&=\mathcal{N}(\beta_{80}+\beta_{86}x_6, \sigma_8^2)\\
p(x_7|x_6,x_5)&=\mathcal{N}(\beta_{70}+\beta_{76}x_6+\beta_{75}x_5, \sigma_7^2)\\
p(x_6|x_4,x_3)&=\mathcal{N}(\beta_{60}+\beta_{64}x_4+\beta_{63}x_3, \sigma_6^2)\\
p(x_5|x_3)&=\mathcal{N}(\beta_{50}+\beta_{53}x_3, \sigma_5^2)
\end{aligned}
\quad 
\begin{aligned}[c]
p(x_4)&=\mathcal{N}(\beta_{40}, \sigma_4^2)\\
p(x_3|x_2,x_1)&=\mathcal{N}(\beta_{30}+\beta_{32}x_2+\beta_{31}x_1, \sigma_3^2)\\
p(x_2)&=\mathcal{N}(\beta_{20}, \sigma_2^2)\\
p(x_1)&=\mathcal{N}(\beta_{10}, \sigma_1^2)
\end{aligned}
 \end{equation*}  
\noindent
Then, $p(x) = \mathcal{N}(\mu,\mathcal{C})$, \cite[Theorem 7.3]{koller2009probabilistic}. As before, the QoI depends on $X_6$ and for simplicity, we consider $f(X_6)=X_6$. For  $c>0$, we compute the MGF of $f$ with respect to $P$: 
\begin{eqnarray*}
\MEANNN{P}{e^{\pm c\bar{f}(X_6)}}=\exp\left\{ \frac{c^2}{2}\left(\sigma_6^2+\beta_{64}^2\sigma_4^2+\beta_{63}^2\sigma_{3}^2+\beta_{63}^2\beta_{32}^2\sigma_{2}^2+\beta_{63}^2\beta_{31}^2\sigma_{1}^2\right)\right\}\equiv e^{ \frac{c^2}{2}\mathcal{C}_{66}}
\end{eqnarray*}
where $\mathcal{C}_{66}=\sigma_6^2+\beta_{64}^2\sigma_4^2+\beta_{63}^2\sigma_{3}^2+\beta_{63}^2\beta_{32}^2\sigma_{2}^2+\beta_{63}^2\beta_{31}^2\sigma_{1}^2$ since $ \MEANNN{P}{X_6}=\beta_{60}+\beta_{64}\beta_{40}+\beta_{63}\beta_{30}+\beta_{63}\beta_{32}\beta_{20}+\beta_{63}\beta_{31}\beta_{10}$. We minimize  \eqref{eq:MFUQ PGM} in Theorem \ref{thm:MFUQ PGM} with respect to $c$:
\begin{eqnarray*}
I^{\pm}(f(X_6), P; \mathcal{D}^\eta)&=&\pm \inf_{c>0}\Big[\frac{1}{c} \log \MEANNN{P}{e^{\pm c\bar{f}(X_6)}}+ \frac{\eta}{c}\Big]=\pm \inf_{c>0}\Big[c\,\mathcal{C}_{66}+ \frac{\eta}{c}\Big]
\end{eqnarray*}
which in turn gives us the optimizer $c=\sqrt{\frac{\eta}{\mathcal{C}_{66}}}$, and thus
\begin{equation}
\label{eq:GBayesian network MFUQ ex}
 I^{\pm}(f(X_6), P; \mathcal{D}^\eta) = \pm \sqrt{2\mathcal{C}_{66}\eta} = \pm \sqrt{2(\sigma_6^2+\beta_{64}^2\sigma_4^2+\beta_{63}^2\sigma_{3}^2+\beta_{63}^2\beta_{32}^2\sigma_{2}^2++\beta_{63}^2\beta_{31}^2\sigma_{1}^2)\eta}
\end{equation}
By \eqref{eq:MFUQ optimizer:1}, the optimizers in Theorem \ref{thm:MFUQ PGM} are obtained when 
\begin{equation}
    q^{\pm}(x_i|x_{\pi^{Q^{\pm}}_i}) = p(x_i|x_{\pi_i})=\mathcal{N}(\beta_{i0}+\beta_i^T x_{\pi_i}, \sigma_i^2)
\end{equation}
for $i=5,7,8$, since they are not ancestors of $X_6$, and by recalling  \eqref{ex161}-\eqref{ex162}  we further compute the CPDs of  the remaining vertices as follows: Since $f(X_6)=X_6$ is linear and all random variables are linearly depended on their parents, we appropriately pair the factor $e^{\pm c_\pm x_{6}}$ with the exponential of the Gaussian CPD $p(x_6|x_{\pi_6^P})$ and we get a new quadratic term in the exponential as well as a term which linearly depends on the parents of $X_6$. The latter term is canceled out with the corresponding one in the normalizing factor $\MEANNN{P_{6|\pi_6^P}}{e^{\pm c_\pm X_6}}$ as the parents of $X_6$ are given. Precisely,
\begin{eqnarray*}
q^\pm (x_6|x_{\pi_6^{Q^\pm}}) &=& \frac{e^{\pm c_\pm x_{6}}}{\MEANNN{P_{6|\pi_6^P}}{e^{\pm c_\pm X_6}}} \cdot p(x_6|x_{\pi_6^P})\nonumber\\
&=&\frac{\exp\left\{-\frac{(x_6-\beta_{60}-\beta_{64}x_4-\beta_{63}x_2 \mp c_{\pm}\sigma_6^2 )^2}{2\sigma_6^2}\pm c_\pm (\beta_{64}x_4+\beta_{63}x_3)\right\}}{\int_{\mathcal{X}_6} \exp\left\{-\frac{(x_6-\beta_{60}-\beta_{64}x_4-\beta_{63}x_2 \mp c_{\pm}\sigma_6^2 )^2}{2\sigma_6^2}\pm c_\pm (\beta_{64}x_4+\beta_{63}x_3)\right\}dx_6}
\end{eqnarray*}
Thus,
\begin{eqnarray}
  \;\;\;  \;\;\;\;\;q^\pm (x_6|x_{\pi_6^{Q^\pm}}) = \mathcal{N}\left(\beta_{60}+\beta_{64}x_4+\beta_{63}x_3 \pm c_{\pm}\sigma_6^2, \sigma_6^2 \right),\quad \pi_6^{Q^\pm} \equiv \pi_6^P = \{4,3\}.
\end{eqnarray}
Similarly, for $4 \in \rho_6 $ $=\{4,3,2,1\}$
\begin{equation*}
     q^\pm(x_{4}|x_{\pi_4^{Q^\pm}})  = \frac{\MEANNN{P_{6|\{4,3\}}}{e^{\pm c_\pm X_6}}}{\MEANNN{P_{4}}{\MEANNN{P_{6|\{4,3\}}}{e^{\pm c_\pm X_6}}}}  p(x_4)=
     \frac{e^{-\frac{(x_4-\beta_{40} \mp c_{\pm} \beta_{43}\sigma_4^2)^2}{2\sigma_4^2}}e^{\pm c_\pm (\beta_{63}x_3)}}{ \int_{\mathcal{X}_4} e^{-\frac{(x_4-\beta_{40} \mp c_{\pm} \beta_{43}\sigma_4^2)^2}{2\sigma_4^2}}\ dx_4\ e^{\pm c_\pm (\beta_{63}x_3)}}
\end{equation*}
Using same the same argument as before, we get 
\begin{eqnarray}
    q^\pm(x_{4}|x_{\pi_4^{Q^\pm}}) = \mathcal{N}\left(\beta_{40} \pm c_{\pm}\sigma_4^2, \sigma_4^2 \right),\qquad \pi_4^{Q^\pm} \equiv \pi_4^P = \emptyset.
\end{eqnarray}
Furthermore, 
\begin{align}
    q^\pm (x_3|x_{\pi_3^{Q^\pm}})&=\mathcal{N}\left(\beta_{30}+\beta_{32}x_2+\beta_{31}x_1\pm c_{\pm}\sigma_3^2, \sigma_3^2 \right),\qquad &\pi_3^{Q^\pm} \equiv \pi_3^P =\{1,2\}\\
   q^\pm(x_{2}|x_{\pi_2^{Q^\pm}}) &= \mathcal{N}\left(\beta_{20} \pm c_{\pm}\sigma_2^2, \sigma_2^2 \right),\qquad &\pi_2^{Q^\pm} \equiv \pi_2^P = \emptyset\\
   q^\pm(x_{1}|x_{\pi_1^{Q^\pm}}) &= \mathcal{N}\left(\beta_{10} \pm c_{\pm}\sigma_1^2, \sigma_1^2 \right),\qquad &\pi_1^{Q^\pm} \equiv \pi_1^P = \emptyset.
\end{align}
By using the equation $\pm c_\pm \MEANNN{Q^\pm}{X_6} - \log \MEANNN{P}{e^{\pm c_\pm X_6}} =\eta$, the parameters $c_{\pm}$ are given  by
\begin{equation}
    c_\pm = \pm \sqrt{\frac{2\eta}{\mathcal{C}_{66}}} = \pm \sqrt{\frac{2\eta}{\sigma_6^2+\beta_{64}^2\sigma_4^2+\beta_{63}^2\sigma_{3}^2+\beta_{63}^2\beta_{32}^2\sigma_{2}^2+\beta_{63}^2\beta_{31}^2\sigma_{1}^2}} 
\end{equation}
\end{expl}
\begin{expl}[Computation of $F$ for Example~\ref{ex:cont}] \label{ex:contF}  For Example~\ref{ex:cont}, we compute $F(x_3,\rho_{3})$ with $f(X_6)=X_6$ and $l=3$ (and thus $\rho_6=\{1,2,3,4\}$ and $\rho_{3}=\{1,2\}$) as

\begin{eqnarray}\label{eq:F_3}
F(x_3,x^P_{\rho_{3}})\equiv F(x_3,x_2,x_1)&=&\int_{\mathcal{X}_{\{4,6\}}}x_6  p(x_6|x_4,x_3)p(x_4)dx_6dx_4\\
&=& \beta_{60}+\beta_{64}\beta_{40}+\beta_{63}x_3= F(x_3)\nonumber
\end{eqnarray}

\end{expl}

\begin{expl}[Computation of $\beta_{kl}$ and  $\tilde{\beta}_{kl}$ for Example~\ref{ex:cont}]\label{ex:contBeta}
 Let us now revisit Example~\ref{ex:cont} and compute $\beta_{kl}$ and $\tilde{\beta}_{kl}$ of Corollary~\ref{cor:Gaussian Bayesian network MFSI} when $l \in \pi_k^P$, e.g $l=3$  and  $l\in\rho_{k}\setminus \pi_k$, e.g. $l=2$ respectively. In the first case, $P_{3|\pi_{3}}$ is  perturbed under the constraint $R(Q_{3|\pi_3^Q}\|P_{3|\pi_{3}}) \leq \eta_3$ or $R(Q_{3|\pi_{3}}\|P_{3|\pi_{3}}) \leq \eta_3$, i.e. consider $Q \in \mathcal{D}^{\eta_3}_3$ or $\mathcal{D}^{\eta_3}_{3,P}$ and $f(X_6)=X_6$. $F(x_3,x_{\rho_3})$ is given by \eqref{eq:F_3} and by Theorem~\ref{thm:MFSI general}, \ref{thm:MFSI} and  \eqref{eq:MFSI general}, we can conclude that
\begin{equation}
    I^{\pm}(f(X_6), P; \mathcal{D}_3^{\eta_3}) = I^{\pm}(f(X_6), P; \mathcal{D}_{3,P}^{\eta_3}) = \pm |\beta_{63}|\sqrt{2\sigma_3^2\eta_3} 
\end{equation}
In the second case, $P_{2|\pi_{2}}$ is  perturbed under the constraint $R(Q_{2|\pi_2^Q}\|P_{2|\pi_{2}}) \leq \eta_2$ or $R(Q_{2|\pi_{2}}\|P_{3|\pi_{2}}) \leq \eta_2$. We compute $F(x_2,x_{\rho_2})=\beta_{60}+\beta_{64}\beta_{40}+\beta_{63}\beta_{30}+\beta_{63}\beta_{32}x_2+\beta_{63}\beta_{31}\beta_{10}=F(x_2)$ and $\tilde{\beta}_{62}=\beta_{63}\beta_{32}$.

\end{expl}
\section{Model uncertainty for inhomogeneous Markov chains}
We consider the Markov chain models shown in Figure \ref{fig: Markov chain}, and the QoI $f(X_k)$. Then we only perturb $P_{l|l-1}$ with $l \leq k$, under the constraint $R(Q_{l|\pi_l^Q}\|P_{l|l-1}) \leq \eta_l$. The function $F(x_l,x_{\rho_l^P})$ defined in \eqref{eq:F} depends only on $x_l$ and by Theorem \ref{thm:MFSI general},  we have
    \begin{equation}
    I^{\pm}(f(X_k), P; \mathcal{D}_l^{\eta_l})  = \pm \MEANNN{P_{\{l-1\}}}{
\inf_{c>0}\Big[ \frac{1}{c} \log \MEANNN{P_{l|l-1}}{e^{\pm c\bar{F}(X_l,X_{\rho_l})}}+\frac{\eta_l}{c} \Big]}. 
\end{equation}
\noindent Since $F(x_l,x_{\rho_l^P})=F(x_l)$ the condition on Theorem \ref{thm:MFSI} is satisfied, and therefore we have $I^{\pm}(f(X_k), P; \mathcal{D}^{\eta_l}_{l,P}) = I^{\pm}(f(X_k), P; \mathcal{D}_l^{\eta_l})$. To obtain the optimizers in both Theorem \ref{thm:MFSI general} and \ref{thm:MFSI}, we use  \eqref{eq:opt:MFSI:c}-\eqref{eq:opt:MFSI:1} and thus 
\begin{equation}
    q^\pm (x_i|x_{i-1}) \equiv p(x_i|x_{i-1}) \qquad \textrm{ for all $i \neq l$}
\end{equation}
and
\begin{equation}
    q^\pm (x_l|x_{l-1}) = \frac{e^{\pm c_\pm(x_{l-1}) F(x_l)}}{\MEANNN{P}{e^{\pm c_\pm(x_{l-1}) F(X_l)}|x_{l-1}}} p(x_l|x_{l-1}) 
\end{equation}
where $c_{\pm}(x_{l-1})$ are the unique solutions of $R(P^{ \pm c_{\pm}}_{l|l-1}\|P_{l|l-1}) =\eta_l$
for all $x_{l-1}$. Moreover, by perturbing $P_{l|l-1}$, $l>k$, with the constraint $$R(Q_{l|\pi_l^Q}\|P_{l|l-1}) \leq \eta_l \textrm{ or } R(Q_{l|l-1}\|P_{l|l-1}) \leq \eta_l,$$and by Theorem \ref{thm:MFSI general} and \ref{thm:MFSI}, we have $I^{\pm}(f(X_k), P; \mathcal{D}_l^{\eta_l}) = I^{\pm}(f(X_k), P; \mathcal{D}^{\eta_l}_{l,P}) = 0$. Note that when the ambiguity set is given by \eqref{eq:set:MFUQ}, it includes also $Q$'s that are \textit{non-Markovian}. However, the optimizers are inhomogeneous Markov chains and are provided by Theorem~\ref{thm:MFUQ PGM}.
\begin{figure}[ht]
\centering
\includegraphics[width=1\textwidth]{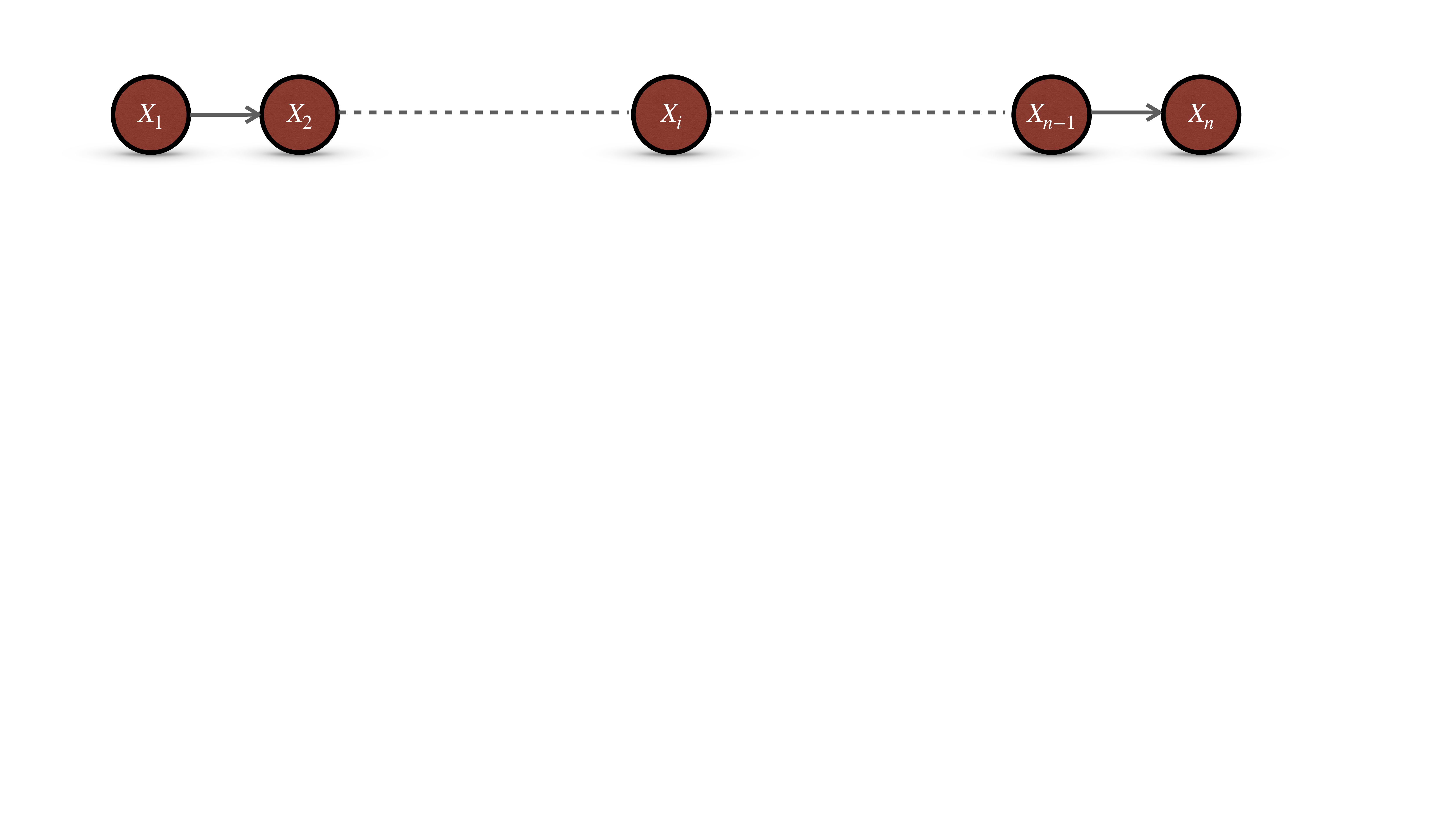}
\vspace{-6cm}
\caption{An inhomogeneous Markov chain consists of $X=\{X_1, X_2,\dots,X_n\}$ with $p(x)=\prod_{i=1}^n p(x_i|x_{i-1})$.}
\label{fig: Markov chain}
\end{figure}

\section{Proof of Lemma~\ref{lem:UpperF}}\label{app:lem:UpperF}
Since for any $Q \in \mathcal{D}_l^{\eta_l}$, we have $\pi_j^Q \equiv \pi_j^P =\pi_j$ and $Q_{j|\pi_j} \equiv P_{j|\pi_j}$ for all $j \neq l$, therefore, we can rewrite the bias $\MEANNN{Q}{f(X_k)} - \MEANNN{P}{f(X_k)}$ as
\begin{eqnarray}\label{compsecondmainthm}
    \;\;\;\;\;\;\;\;\;\;\;\;\;&=&    \int_{\mathcal{X}} f(x_k) \prod_{i=1}^n Q(dx_i|x_{\pi_i^Q}) -  \int_{\mathcal{X}} f(x_k) \prod_{i=1}^n P(dx_i|x_{\pi_i^P}) \\
     &=&   \int_{\mathcal{X}_k} \int_{\mathcal{X}_{\rho_k^Q}} f(x_k) \prod_{i \in  \rho_k^Q\cup\{k\}} Q(dx_i|x_{\pi_i^Q}) -  \int_{\mathcal{X}_k} \int_{\mathcal{X}_{\rho_k^P}} f(x_k) \prod_{i \in \rho_k^P\cup\{k\}} P(dx_i|x_{\pi_i^P})\notag\\
     &=&\MEANNN{Q_{\{k\}}}{f(X_k)} - \MEANNN{P_{\{k\}}}{f(X_k)}\nonumber
\end{eqnarray}
If $l \notin \bar{\rho}_k^P$, we have $\pi_i^Q \equiv \pi_i^P =: \pi_i$ and $Q(dx_i|x_{\pi_i}) \equiv P(dx_i|x_{\pi_i})$ for all $i \in\bar{\rho}_{k}$, therefore $Q_{\{k\}}\equiv P_{\{k\}}$, and thus $\MEANNN{Q}{f(X_k)} - \MEANNN{P}{f(X_k)} =0$. Based on this calculation for $Q \in \mathcal{D}_l^{\eta_l}$, we stress that our indices capture the graph structure correctly, e.g. perturbations on disconnected vertices do not affect the QoI $f=f(X_k)$. Since $Q(dx_j|x_{\pi_j}) \equiv P(dx_j|x_{\pi_j})$ for all $j \neq l$, \eqref{compsecondmainthm} equals to
\begin{eqnarray}
&& \;\;\;\;\;\;\;\;\;\;\;\MEANNN{Q}{f(X_k)} - \MEANNN{P}{f(X_k)} =\MEANNN{Q_{\{k\}}}{f(X_k)} - \MEANNN{P_{\{k\}}}{f(X_k)}\\
    &&\qquad\;\;= \int_{\mathcal{X}_{\rho_k^Q\cup\{k\}}} f(x_k) \prod_{i \in \rho_k^Q \cup \{k\} \setminus \rho_l^Q\cup\{l\}} Q(d x_i|x_{\pi_i^Q}) \cdot Q(dx_l|x_{\pi_l^Q}) \cdot \prod_{i \in \rho_l^Q } Q(d x_i|x_{\pi_i^Q})\notag\\
    & &\qquad\;\; \;\;- \int_{\mathcal{X}_{\rho_k^P\cup\{k\}}} f(x_k) \prod_{i \in \bar{\rho}_k^P \setminus \rho_l^P\cup\{l\}} P(d x_i|x_{\pi_i}^P) \cdot P(dx_l|x_{\pi_l}^P) \cdot \prod_{i \in \rho_l^P } P(d x_i|x_{\pi_i}^P)\notag\\
     && \qquad\;\;=  \int_{\mathcal{X}_{\rho_k^P\cup\{k\}}} f(x_k) \prod_{i \in \bar{\rho}_k^P \setminus \rho_l^P\cup\{l\}} P(d x_i|x_{\pi_i^P}) \cdot Q(dx_l|x_{\pi_l^Q}) \cdot \prod_{i \in \rho_l^P } P(d x_i|x_{\pi_i^P})\notag\\
     & & \qquad\;\;\;\;- \int_{\mathcal{X}_{\rho_k^P\cup\{k\}}} f(x_k) \prod_{i \in \bar{\rho}_k^P \setminus \rho_l^P\cup\{l\}} P(d x_i|x_{\pi_i^P}) \cdot P(dx_l|x_{\pi_l^P}) \cdot \prod_{i \in \rho_l^P} P(d x_i|x_{\pi_i^P})\notag\\
     &&\qquad\;\;=  \int_{\mathcal{X}_{\rho_l^P}}  \left[\int_{\mathcal{X}_{l}} F(x_l,x_{\rho_l^P}) Q(dx_l|x_{\pi_l^Q}) - \int_{\mathcal{X}_{l}} F(x_l,x_{\rho_l^P}) P(dx_l|x_{\pi_l^P})\right] \prod_{i \in \rho_l } P(d x_i|x_{\pi_i^P})\notag\\
     &&\qquad\;\;= \MEANNN{P_{\rho_l^P}}{ \MEANNN{Q_{l|\pi_l^Q}}{F(X_l,X_{\rho_l^P})} - {\MEANNN{P_{l|\pi_l^P}}{F(X_l,X_{\rho_l^P})}}}\notag
     \end{eqnarray}

\section{KL-divergence Chain Rule for Bayesian networks}
\label{subsec:eta in PGM} 
In this subsection, we discuss the KL chain rule \cite{cover2012elements} in the context of  Bayesian networks as it paves the way for considering suitable ambiguity sets (different than \eqref{eq:set:MFUQ})  and applying  model sensitivity analysis to each component on a baseline Bayesian network. We remind that  $P_{i|\pi_i^P}$ is the conditional distribution of $X_i$ with given parents $X_{\pi_i^P} = x_{\pi_i}$, i.e. $P_{i|\pi_i^P}(dx_i)=P(dx_i|x_{\pi_i})$ and for clarity purposes and stressing the given values,  we  write $P_{i|X_{\pi_i^P}=x_{\pi_i}}(dx_i)$ 
instead. 

\begin{definition}\label{def:CondKL}
Let $P$ and $Q$ be two PGMs with densities $p$ and $q$ respectively defined as \eqref{eq:PGM:def}.  For each $i\in\{1,\dots,n\}$, we define the conditional KL divergence between $Q_{i|X_{\pi_i^Q}}$ and $P_{i|X_{\pi_i^P}}$ with given $X_{\pi_i^Q}=x_{\pi_i^Q}$ and $X_{ \pi_i^P}=x_{\pi_i^P}$ as
\begin{equation}
    R(Q_{i|X_{\pi_i^Q}=x_{\pi_i^Q}}\|P_{i|X_{\pi_i^P}=x_{\pi_i^P}}) = \int_{\mathcal{X}_i} \log \frac{Q(dx_i|x_{\pi_i^Q})}{P(dx_i|x_{\pi_i^P})}Q(dx_i|x_{\pi_i^Q}) 
\end{equation}
\end{definition}
\begin{lemma}[Chain Rule of Relative Entropy for PGMs] \label{lem:CKLPGMs}For any two PGMs $P$ and $Q$ with densities $p(x)=\prod_{i=1}^n p(x_i|x_{\pi_i^P})$ and $q(x)=\prod_{i=1}^n q(x_i|x_{\pi_i^Q})$, the KL divergence can be expressed as:
\begin{equation}
\label{eq:chain rule for KL}
   R(Q\|P) = \sum_{i=1}^n \MEANNN{Q_{\pi_i^Q \cup \pi_i^P}}{R(Q_{i|X_{\pi_i^Q}}\|P_{i|X_{\pi_i^P}})}
    \end{equation}
 where $R(Q_{i|X_{\pi_i^Q}}\|P_{i|X_{\pi_i^P}})$ is the conditional KL divergence given in Definition~\ref{def:CondKL}  and  $\mathbb{E}_{Q_{\pi_i^Q \cup \pi_i^P}}$ is the expectation with respect to $Q_{A}$ defined in Section~\ref{sec:MFUQ} with $A=\pi_i^Q \cup \pi_i^P$.
\end{lemma}
\begin{proof}
\begin{eqnarray}
\;\;\;\;\;\;\;\;\;\;\;R(Q\|P) &=&\int_{\mathcal{X}} \sum_{i=1}^n \log \frac{Q(dx_i|x_{\pi_i^Q})}{P(dx_i|x_{\pi_i^P})}\prod_{j=1}^n Q(dx_j|x_{\pi_j^Q}) \\
&=& \sum_{i=1}^n \int_{\mathcal{X}_{ \rho_i^P\cup\rho_i^Q}}\int_{\mathcal{X}_i} \log \frac{Q(dx_i|x_{\pi_i^Q})}{P(dx_i|x_{\pi_i^P})}Q(dx_i|x_{\pi_i^Q})\prod_{j \in \{\rho_i^Q \cup \rho_i^P\}} Q(dx_j|x_{\pi_j^Q}) \notag\\
&=& \sum_{i=1}^n \MEANNN{Q_{\pi_i^Q \cup \pi_i^P}}{R(Q_{i|X_{\pi_i^Q}}\|P_{i|X_{\pi_i^P}})} \nonumber
\end{eqnarray}
\end{proof}

    
    \section{Schematic for Model Sensitivity Indices}    
 This schematic refers to the main theorems of Section~\ref{sec:QUSI}.
     \begin{figure}[ht]
\centering
\includegraphics[width=1.1\textwidth]{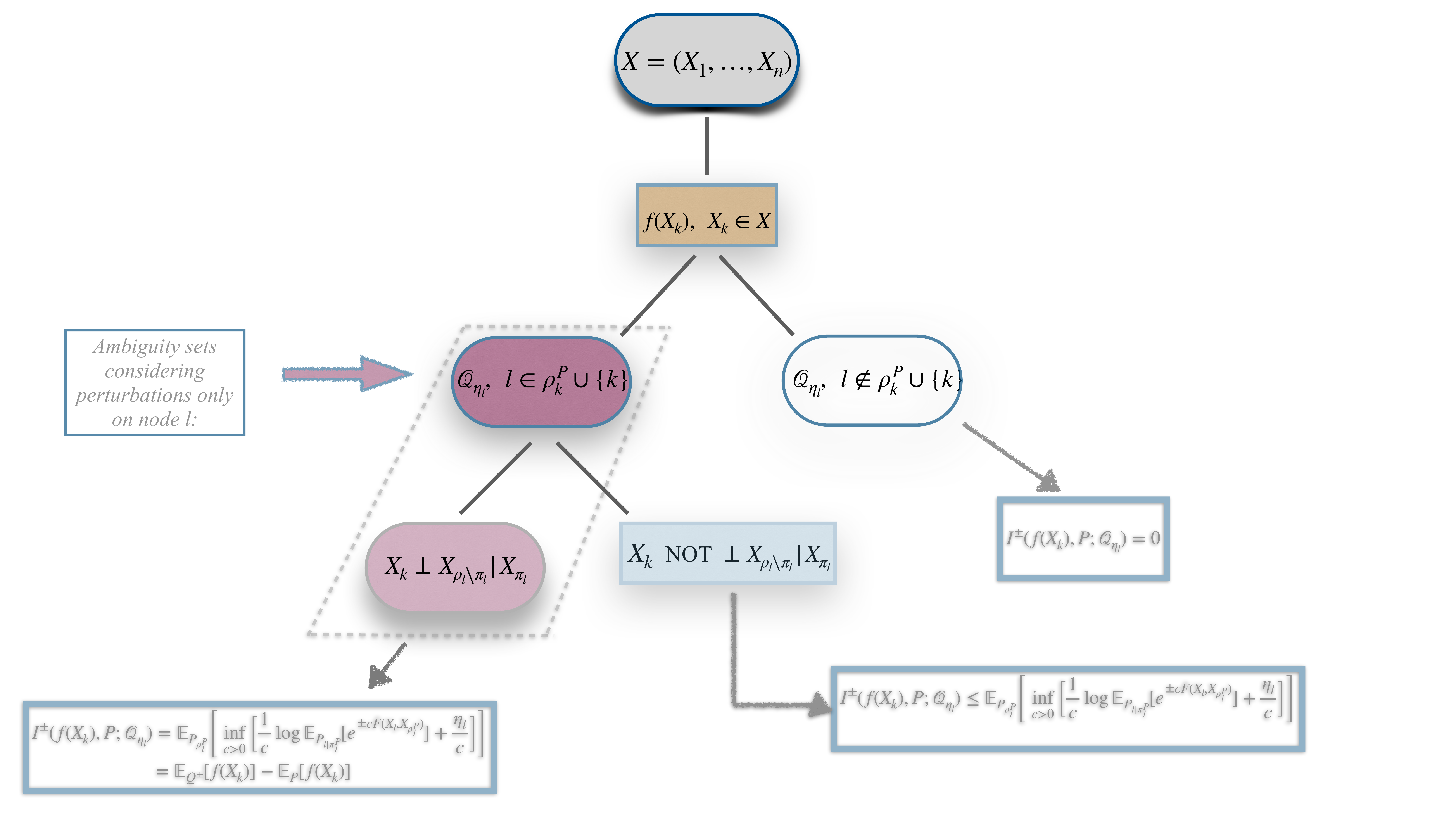}
\caption{\small{A schematic representation of how the set of vertices of a graph can be decomposed according to the relative position of vertex $k$ that corresponds to the QoI $f(X_k)$, and a vertex $l$ such that perturbations of $P_{l|\pi_l}$ are considered. By Lemma~\ref{lem:UpperF}, Theorem~\ref{thm:MFSI general} and ~\ref{thm:MFSI}, the predictive uncertainty given by \eqref{eq:PU general MFSI} with $l$ being in different parts of the decomposition varies: The set of vertices is first split as $X=(\bar{\rho}_k^P)\cup(\bar{\rho}_k^P)^c$.  The predictive uncertainty $I^{\pm}(f(X_k), P; \mathcal{D}^{\eta_l}_l)$ over models with only perturbed $P_{l|\pi_l}$ when $l\in\bar{\rho}_k^P$ is given by \eqref{eq:MFSI general} and is tight in $\mathcal{D}^{\eta_l}_l$,while perturbation on $P_{l|\pi_l}$ when $l\notin\bar{\rho}_k^P$ do not affect the QoI and thus $I^{\pm}(f(X_k), P; \mathcal{Q}_{\eta_l})=0$. We then decompose the set of vertices $\bar{\rho}_k^P$ into $\{l: X_k \perp X_{\rho_l \setminus \pi_l} | X_{\pi_l}\}$ and its complement. The predictive uncertainty over $\mathcal{D}^{\eta_l}_{l,P}$ with $l$ in the former set  is same as the one for $l\in\bar{\rho}_k^P$ with the difference that is tight on $\mathcal{D}^{\eta_l}_{l,P}$ contrary to the one in the latter set where the bound is not attained as well as is not tight.}}\label{fig:readgraph}
\end{figure}

\section{Data-informed stress tests for Gaussian Bayesian networks} 
In this section, we explain with detail Data-informed stress tests analysis when the baseline model $P$ is a Gaussian Bayesian network. Let $P$ be a Gaussian Bayesian network with conditional probability densities $p(x_i|x_{\pi_i})$ satisfying $p(x_i|x_{\pi_i}) = \mathcal{N}(\beta_{i0}+\beta_i^T x_{\pi_i}, \sigma_i^2)$ for some $\beta_{i0}$, $\beta_i$, and $\sigma_i^2$, i.e. $P_{i|\pi_i}$ is the conditional distribution of $X_i = \beta_{i0}+\beta_i^T X_{\pi_i} + \epsilon_i$.  The random variable $\epsilon_i$ has density  $p_{\epsilon_i}(x) = \mathcal{N}(0, \sigma_i^2)$ and comes from fitting data with Maximum-Likelihood-Estimation. Let us  consider alternative models to $P$ such as
\begin{equation}\label{GBayesian networkstress1}
Q_{i|\pi_i}:\quad    X_i = \beta_{i0}+\beta_i^T X_{\pi_i} + \tilde{\epsilon}_i
\end{equation}
where $\tilde{\epsilon}_i$ follows another approximate distribution of the data with density $q_{\tilde{\epsilon}_i}(x)$. For instance, we can consider  $Q_{\tilde{\epsilon}_i}$ with density $q_{\tilde{\epsilon}_i}$ as the histogram, that is
\begin{equation}
\label{eq:hist}
q_{\tilde{\epsilon}_i}^{hist}(x)=\sum_{k=1}^{m} \frac{\nu_k}{n h} I(x \in B_k)\, ,
\end{equation}
where ${B_1, \dots, B_m}$ are the histogram bins, $h$ is the bin width, $n$ is the number of observations and $\nu_k$ is the number of observations in $B_k$. Alternatively, we can consider the model $Q_{\tilde{\epsilon}_i}$ given by a KDE viewed here as a high resolution but smooth approximation of the histogram, namely
\begin{equation}
\label{eq:KDE}
q_{\tilde{\epsilon}_i}^{KDE}(x)=\sum_{k=1}^{n} \frac{1}{n h} K(\frac{x-x_{i}}{h})\, ,
\end{equation}
where $K(\cdot)$ is the normal kernel smoothing function with bin width $h$, $(x_{1},\dots,x_{n})$ are the samples of $\epsilon_i$. Other KDE kernels can be considered here (see \cite{w06}) or any other probabilistic representations of the data in the histogram. Then, based on the following computation:Therefore, for given $x_{\pi_i}$, we have
\begin{eqnarray}
\label{eq:eta_i_Pa}
\eta_i^{\pi_i}  &=& \int \log \frac{q(x_i|x_{\pi_i})}{p(x_i|x_{\pi_i})}q(x_i|x_{\pi_i}) dx_i\notag\\
&=&  \int \log \frac{q(x_i-\beta_{i0}-\beta_i^T x_{\pi_i}|x_{\pi_i})}{p(x_i-\beta_{i0}-\beta_i^T x_{\pi_i}|x_{\pi_i})}q(x_i-\beta_{i0}-\beta_i^T x_{\pi_i}|x_{\pi_i}) dx_i\notag\\
&=&\int \log \frac{q_{\tilde{\epsilon}_i}(x)}{p_{\epsilon_i}(x)}q_{\tilde{\epsilon}_i}(x) dx\, ,
\end{eqnarray}
Based on the above computation, $\eta_i^{\pi_i}$ is independent of  $\pi_i$ and hence $\eta_i^{\pi_i} \equiv \eta_i$.

\section{Model sensitivity indices for the ORR Bayesian network} 
\begin{table}[ht]\centering\setlength{\tabcolsep}{16pt}\renewcommand{\arraystretch}{1.5}
\caption{Outcomes of MLE for the parameters involved in  \eqref{eq: omega dist}-\eqref{eqyi'} for the ORR Bayesian network in Section~\ref{sec: ORR}.}
\begin{tabular}{|c|c|}
\hline
\hline
$\beta_{y_1,0}$ = 0.0595 & $\beta_{e0,0}$, $\beta_{ei,0}$ = 0 \\
\hline
$\sigma_{e0}^2$ = 0.0329 & $\sigma_{ei}^2$ = 0.0065\\
\hline
$\beta_{y_2,0}$ = 1.8231 & $\beta_{d0,0}$ = -0.0754 \\
\hline
$\beta_{di,0}$ = -0.0222 & $\sigma_{di}^2$ = 0.0354\\
\hline
$\beta_{y_1,x_0}$ = 0.5111 & $\sigma_{d0}^2$ = 0.1032 \\
\hline
$\beta_{s1,0}$ = -0.2967 & $\sigma_{s1}^2$ = 0.0046\\
\hline
$\beta_{y_2,x_0}$  = -0.5564 & $\beta_{s0,0}$ = 0.0067 \\
\hline
$\beta_{s2,0}$ = -0.1209 & $\sigma_{s2}^2$ = 0.0054\\
\hline
$\beta_{ci,0}$  = 0 & $\sigma_{s0}^2$ = 0.0010\\
\hline
$\sigma_{c1}^2$ = 0.0347 & $\sigma_{c2}^2$ = 0.0204 \\
\hline
\hline
\end{tabular}
\label{tab:MLE beta}
\end{table}


\subsection{Calculation of  model sensitivity indices}\label{sebsec:important calculation for qoi}
We remind that the optimal oxygen binding energy is defined as
\[
x_{O^*}^P=\mathrm{argmax}_{x_0} \left[\min\{\MEANNN{P}{y_1|x_0},\MEANNN{P}{y_2|x_0}\}\right]
\]
We compute $\MEANNN{P}{y_i|x_0}$ for $i=1,2$ by using \eqref{eqyi} and \eqref{eqyi'} as follows
\begin{eqnarray*}
\MEANNN{P}{y_i|x_0}&=&\beta_{y_i,0}+\beta_{y_i,x}(x_0+\beta_{s0,0}+\beta_{e0,0}+\beta_{d0,0})\\
&&\qquad+\beta_{ci,0}+\beta_{si,0}+\beta_{ei,0}+\beta_{di,0}\nonumber
\end{eqnarray*}
Then 
\begin{eqnarray}\label{CQOI}
x_{O^*}^P=\frac{\beta_{y_2,0}+\bar\beta_2-\beta_{y_1,0}-\bar\beta_1}{\beta_{y_1,x}-\beta_{y_2,x}}
\end{eqnarray}
where 
\begin{align}
\bar\beta_i=\beta_{y_i,x}(\beta_{s0,0}+\beta_{e0,0}+\beta_{d0,0})+\beta_{ci,0}+\beta_{si,0}+\beta_{ei,0}+\beta_{di,0}
\end{align}
It is a straightforward calculation that for $i=1,2$ and  $l=e0,d0,s0,e1,d1,s1,c1,e2,d2,s2$ and $c2$.
\[
\beta_{y_i,0}+\bar\beta_i=\MEANNN{P_{\omega_l}}{F_{l,i}},\qquad \textrm{ for any $i$ and $l$}
\]
where $F_{l,i}=\MEANNN{P_{y_i|\omega_l}}{y_i|x_0}$ and $p(y_i|\omega_{l},x_0) = \mathcal{N}(\tilde{\beta}_{y_i,0}+\tilde{\beta}_{y_i,\omega_{l}}\omega_{l},\tilde{\sigma}_{y_i}^2)$. Hence \eqref{CQOI} equals to
\[
x_{O^*}^P=\frac{\MEANNN{P_{\omega_l}}{F_{l,2}}-\MEANNN{P_{\omega_l}}{F_{l,1}}}{\beta_{y_1,x}-\beta_{y_2,x}}
\]
Note that since we compute the model sensitivity indices over $\mathcal{D}_{l,P}^{\eta_l}$ for any $l\in\{e0,d0,s0,e1,d1,s1,c1,e2,d2,s2\}$, the alternative Bayesian networks $Q$ have the same structure given by \eqref{eqyi}, same CPDs as the Bayesian network $P$ except from the CPD of $\omega_l$. Let us denote  its conditional distribution by $Q_{\omega_l}$(since $\rho_{\omega_l}=\pi_{\omega_l}=\emptyset$) and  its CPD by $q_{\omega_l}$.  Then, 
\begin{equation}
x_{O^{*}}^Q - x_{O^{*}}^P=
\frac{\MEANNN{Q_{\omega_l}}{F_{l,2}}-\MEANNN{P_{\omega_l}}{F_{l,2}}-(\MEANNN{Q_{\omega_l}}{F_{l,2}}-\MEANNN{P_{\omega_l}}{F_{l,1}})}{\beta_{y_1,x}-\beta_{y_2,x}}
\end{equation} 
which further gives us
\[
 \frac{I^\mp(y_2,P; \mathcal{D}^{\eta_l}_{l,P})-I^\pm(y_1,P; \mathcal{D}^{\eta_l}_{l,P})}{\beta_{y_1,x}-\beta_{y_2,x}}\leq x_{O^{*}}^Q - x_{O^{*}}^P\leq \frac{I^\pm(y_2,P; \mathcal{D}^{\eta_l}_{l,P})-I^\mp(y_1,P; \mathcal{D}^{\eta_l}_{l,P})}{\beta_{y_1,x}-\beta_{y_2,x}}
\]
In the above inequality, by combining Corollary~\ref{cor:Gaussian Bayesian network MFSI} and Table~\ref{tab:tildebeta} we get \eqref{eq:MFSI opt1} and \eqref{eq:MFSI opt2}.
\begin{table}[ht]\centering
\caption{The values of $\tilde{\beta}_{y_i,\omega_{l}}$ involved in $p(y_i|\omega_{l},x_0) = \mathcal{N}(\tilde{\beta}_{y_i,0}+\tilde{\beta}_{y_i,\omega_{l}}\omega_{l},\tilde{\sigma}_{y_i}^2)$. They are used to evaluate the model sensitivity indices $I^\pm(y_i,P; \mathcal{D}^{\eta_l}_{l,P})$, $i=1,2$ provided by Corollary~\ref{cor:Gaussian Bayesian network MFSI}.}
\begin{tabular}{|c|c|c|c|}
\hline
\hline
  & $\omega_l = \omega_{e0},\omega_{d0},\omega_{s0}$ & $\omega_l = \omega_{e1},\omega_{d1},\omega_{s1},\omega_{c1}$ & $\omega_l = \omega_{e2},\omega_{d2},\omega_{s2},\omega_{c2}$\\
\hline
$f = y_1$ & $\tilde{\beta}_{y_1,\omega_{l}} = \beta_{y_1,x}$ & $\tilde{\beta}_{y_1,\omega_{l}} = 1$ & $\tilde{\beta}_{y_1,\omega_{l}} = 0$\\
\hline
$f = y_2$ & $\tilde{\beta}_{y_2,\omega_{l}} = \beta_{y_2,x}$ & $\tilde{\beta}_{y_2,\omega_{l}} = 0$ & $\tilde{\beta}_{y_2,\omega_{l}} = 1$\\
\hline
\hline
\end{tabular}
\label{tab:tildebeta}
\end{table}

\subsection{Propagation of 
model uncertainties to the QoIs} We note the discrepancies  in the  propagation of model misspecification to the QoI  between  different Bayesian network components, as  demonstrated in  Figure~\ref{fig:Pie}. In particular, in Figure~\ref{fig:Pie} (Left)  the same uncertainty (described by model misspecification $\eta_l$) is applied on all ORR Bayesian network vertices, however not all propagate and affect the same the QoI:  see Figure \ref{fig:propagation} for  examples of propagation ($22\%$) and non-propagation ($5\%$ and $0\%$) of model misspecification to the QoI.

\begin{figure}[ht]
\centering
\includegraphics[width=0.7\textwidth]{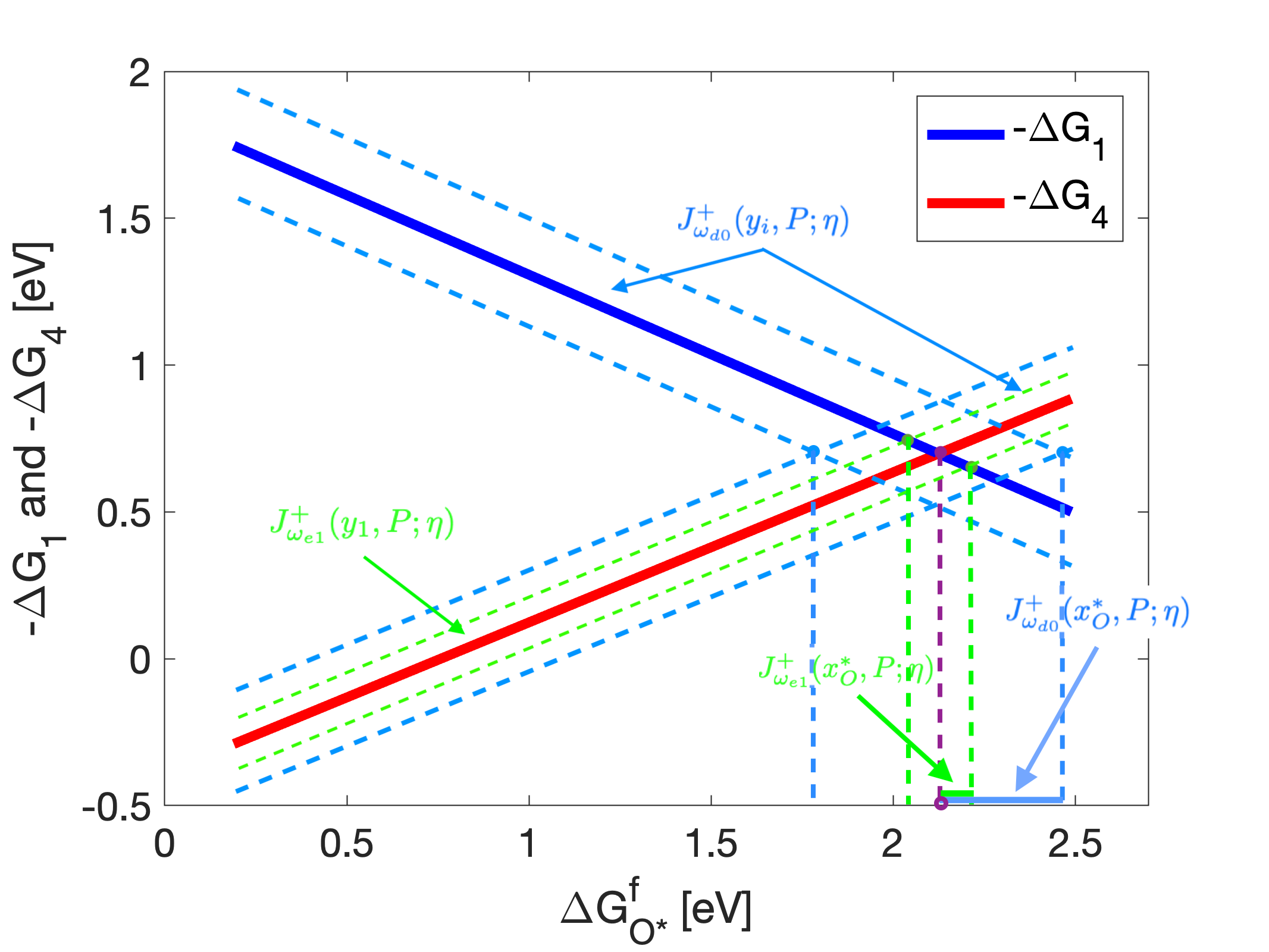}
\caption{\small Propagation 
vs. Non-propagation
of model misspecification of the Bayesian network vertices $\omega_{d0}$ and $\omega_{e1}$ respectively,  to the predictions of the QoI $x_{O^\ast}^P$; model misspecification is set to $\eta=1$ for both Bayesian network vertices. First, note that $I^+(y_2,P;\mathcal{D}^\eta_{\omega_{e1}})=0$ i.e., the model misspecification of $\omega_{e1}$ only affects the prediction of $y_1$, but not $y_2$, see Figure \ref{fig:bounds of QoI}; therefore  the uncertainty of $\omega_{e1}$ only propagates to $x_{O^\ast}^P$ through $y_1$, while  $I^+(y_1,P;\mathcal{D}^\eta_{\omega_{e1}})$ is small since $\omega_{e1}$ has a lower variance which is associated with  more informative available data. Thus, it results in a small corresponding  uncertainty in $x_{O^\ast}^P$. Meanwhile, the uncertainty of $\omega_{d0}$ propagates to $x_{O^\ast}^P$ through both $y_1$ and  $y_2$, (i.e., the model misspecification of $\omega_{d0}$ affects both the predictions of $y_1$ and $y_2$), and $I^+(y_i,P;\mathcal{D}^\eta_{\omega_{d0}})$ is larger since $\omega_{d0}$ has a higher variance (due to insufficient informative data available). Therefore we have  a   larger corresponding uncertainty in $x_{O^\ast}^P$ predictions, as shown in the figure.}
\label{fig:propagation}
\end{figure}

%

 \newpage

\bibliographystyle{plain} 
\bibliography{SIAM_BN_BFKR}
\end{document}